\documentclass[11pt, oneside]{article}
\usepackage{geometry}
\usepackage{graphicx}	
\usepackage{amssymb}
\usepackage{amsfonts,latexsym,amsthm,amssymb,amsmath,amscd,euscript}
\usepackage{bbm}
\usepackage{framed}
\usepackage{fullpage}
\usepackage{mathrsfs}
\usepackage{enumitem}
\usepackage{hyperref}
\hypersetup{
  colorlinks=true,
  linkcolor=blue,
  filecolor=blue,
  citecolor = black,      
  urlcolor=cyan,
}
\usepackage{accents}
\usepackage{setspace}
\hypersetup{colorlinks=true,citecolor=blue,urlcolor=black,linkbordercolor={1 0 0}}
\usepackage{tikz-cd}

\usepackage[boxruled,vlined,nofillcomment,linesnumbered]{algorithm2e}

\makeatletter
\newtheorem*{rep@theorem}{\rep@title}
\newcommand{\newreptheorem}[2]{%
\newenvironment{rep#1}[1]{%
 \def\rep@title{#2 \ref{##1}}%
 \begin{rep@theorem}}%
 {\end{rep@theorem}}}
\makeatother

\theoremstyle{plain}
\newtheorem{theorem}{Theorem}
\newreptheorem{theorem}{Theorem}
\newtheorem{lemma}[theorem]{Lemma}

\newtheorem{proposition}[theorem]{Proposition}

\newtheorem{claim}[theorem]{Claim}

\theoremstyle{definition}
\newtheorem{definition}{Definition}
\newtheorem{defn}[definition]{Definition}

\newreptheorem{lemma}{Lemma}
\newreptheorem{claim}{Claim}
\newreptheorem{proposition}{Proposition}

\numberwithin{theorem}{section}
\numberwithin{definition}{section}

\newcommand{\nc}{\newcommand}
\nc{\DMO}{\DeclareMathOperator}
\newcount\Comments  %
\DeclareMathOperator*{\argmax}{arg\,max}

\DMO{\prox}{prox}
\DMO{\Span}{span}
\DMO{\UCB}{UCB}
\DMO{\LCB}{LCB}

\nc{\til}{\widetilde}
\nc{\td}{\tilde}
\nc{\wh}{\widehat}
\nc{\todo}[1]{\ifnum\Comments=1 {\color{red}  [TODO: #1]}\fi}
\nc{\old}[1]{\ifnum\Comments=1 {\color{brown}  [OLD: #1]}\fi}
\nc{\noah}[1]{\ifnum\Comments=1 {\color{purple} [ng: #1]}\fi}

\nc{\MH}{\mathcal{H}}
\nc{\fools}[3]{\MF_{#3}({#1}, {#2})}
\nc{\fool}[2]{\MF({#1},{#2})}
\nc{\foolsimple}[1]{\MF^{\rm simp}({#1})}
\nc{\foolingset}{\MF}
\nc{\clip}[2]{{\rm clip}\left[ \left. {#1} \right| {#2} \right]}
\nc{\imax}{\omega}
\DMO{\conv}{conv}
\nc{\st}{\star}
\nc{\lng}{\langle}
\nc{\rng}{\rangle}
\DMO{\OOPT}{opt}
\nc{\dopt}[2]{\ell_{\OOPT}({#1},{#2})}
\nc{\grad}{\nabla}
\nc{\MG}{\mathcal{G}}
\nc{\MP}{\mathcal{P}}
\nc{\PP}{\mathbb{P}}
\nc{\TT}{\mathbb{T}}
\nc{\TTmax}{\TT_{\max}}
\DMO{\Reg}{Reg}
\DMO{\Ham}{Ham}
\DMO{\Gap}{Gap}
\DMO{\GD}{GD}
\DMO{\GDA}{GDA}
\DMO{\EG}{EG}
\DMO{\OGDA}{OGDA}
\nc{\CR}{\mathscr{R}}
\nc{\CC}{\mathscr{C}}
\DMO{\Unif}{Unif}
\DMO{\Tr}{Tr}
\nc{\ul}{\underline}
\nc{\ol}{\overline}
\nc{\Qu}{\ul{Q}}
\nc{\Qo}{\ol{Q}}
\nc{\qu}{\ul{q}}
\nc{\qo}{\ol{q}}
\nc{\Ro}{\ol{R}}
\nc{\Vu}{\ul{V}}
\nc{\Vo}{\ol{V}}
\nc{\RanQ}{\Delta Q}
\nc{\RanV}{\Delta V}
\nc{\clipQ}{\Delta \breve{Q}}
\nc{\frzQ}{\Delta \mathring{Q}}
\nc{\tfrzQ}{\widetilde{\Delta} \mathring{Q}}
\nc{\clipV}{\Delta \breve{V}}
\nc{\clipdelta}{\breve{\delta}}
\nc{\cliptheta}{\breve{\theta}}
\nc{\delmin}{\Delta_{{\rm min}}}
\nc{\tildm}{\widetilde{\Delta}_{{\rm min}}}
\nc{\delmins}[1]{\Delta_{{\rm min},{#1}}}
\nc{\gapfinal}[1]{\max \left\{ \frac{\frzQ_{{#1}}^{k^\st}(x,a)}{2H},\frac{\delmin}{4H^2} \right\}}
\nc{\tgapfinal}[1]{\max \left\{ \frac{\tfrzQ_{{#1}}^{k^\st}(x,a)}{2H},\frac{\tildm}{4H^2} \right\}}
\nc{\gapfinals}[2]{\max \left\{ \frac{\frzQ_{{#1}}^{#2}(x,a)}{2H}, \frac{\delmin}{4H^2} \right\}}
\nc{\Regret}[1]{{\rm Regret}_{#1}}
\nc{\FF}{{\rm F}}
\nc{\TTT}{{\rm T}}

\nc{\algnst}[1]{\begin{align*}#1\end{align*}}
\nc{\algn}[1]{\begin{align}#1\end{align}}
\nc{\matx}[1]{\left(\begin{matrix}#1\end{matrix}\right)}

\nc{\nuu}{\nu}

\nc{\bv}{\mathbf{v}}
\nc{\bone}{\mathbf{1}}
\nc{\bX}{\mathbf{X}}
\nc{\bY}{\mathbf{Y}}
\nc{\bG}{\mathbf{G}}
\nc{\bz}{\mathbf{z}}
\nc{\bw}{\mathbf{w}}
\nc{\bB}{\mathbf{B}}
\nc{\bA}{\mathbf{A}}
\nc{\bJ}{\mathbf{J}}
\nc{\bK}{\mathbf{K}}
\nc{\bb}{\mathbf{b}}
\nc{\ba}{\mathbf{a}}
\nc{\bc}{\mathbf{c}}
\nc{\bC}{\mathbf{C}}
\nc{\BR}{\mathbb R}
\nc{\BA}{\mathbb{A}}
\nc{\BP}{\mathbb{P}}
\nc{\BC}{\mathbb C}
\nc{\bx}{\mathbf{x}}
\nc{\bS}{\mathbf{S}}
\nc{\bM}{\mathbf{M}}
\nc{\bR}{\mathbf{R}}
\nc{\bN}{\mathbf{N}}
\nc{\by}{\mathbf{y}}
\nc{\sy}{y}
\nc{\sx}{x}

\nc{\MO}{\mathcal O}
\nc{\MU}{\mathcal{U}}
\nc{\ME}{\mathcal{E}}
\nc{\MN}{\mathcal{N}}
\nc{\MK}{\mathcal{K}}
\nc{\MM}{\mathcal{M}}
\nc{\MS}{\mathcal{S}}
\nc{\MT}{\mathcal{T}}
\nc{\BF}{\mathbb F}
\nc{\BQ}{\mathbb Q}
\nc{\MX}{\mathcal{X}}
\nc{\MA}{\mathcal{A}}
\nc{\MD}{\mathcal{D}}
\nc{\MB}{\mathcal{B}}
\nc{\MZ}{\mathcal{Z}}
\nc{\MJ}{\mathcal{J}}
\nc{\MW}{\mathcal{W}}
\nc{\MR}{\mathcal{R}}
\nc{\MY}{\mathcal{Y}}
\nc{\BZ}{\mathbb Z}
\nc{\BN}{\mathbb N}
\nc{\ep}{\epsilon}
\nc{\gapfn}[1]{\varepsilon_{#1}}
\nc{\ggapfn}[2]{\varphi_{#1}({#2})}
\nc{\epsahk}{\gapfn{0}}
\nc{\BH}{\mathbb H}
\nc{\BG}{\mathbb{G}}
\nc{\D}{\Delta}
\nc{\MF}{\mathcal{F}}
\nc{\One}{\mathbbm{1}}
\nc{\bOne}{\mathbf{1}}
\nc{\Aopt}{\mathcal{A}^{\rm opt}}
\nc{\Amul}{\mathcal{A}^{\rm mul}}
\nc{\MEwc}{\ME^{\rm wc}}
\nc{\MEpred}{\ME^{\rm pred}}

\nc{\SP}{\mathsf P}
\nc{\SQ}{\mathsf Q}

\nc{\DO}{\accentset{\circ}{\D}}
\nc{\mf}{\mathfrak}
\nc{\mfp}{\mathfrak{p}}
\nc{\mfq}{\mf{q}}
\nc{\Sp}{\mbox{Spec}}
\nc{\Spm}{\mbox{Specm}}
\nc{\hookuparrow}{\mathrel{\rotatebox[origin=c]{90}{$\hookrightarrow$}}}
\nc{\hookdownarrow}{\mathrel{\rotatebox[origin=c]{-90}{$\hookrightarrow$}}}
\nc{\hra}{\hookrightarrow}
\nc{\tra}{\twoheadrightarrow}
\nc{\sgn}{{\rm sgn}}
\nc{\aut}{{\rm Aut}}
\nc{\Hom}{{\rm Hom}}
\nc{\img}{{\rm Im}}
\DMO{\id}{Id}
\DMO{\supp}{supp}
\DMO{\KL}{KL}
\DMO{\BSS}{BSS}
\DMO{\BES}{BES}
\DMO{\BGS}{BGS}
\DMO{\poly}{poly}
\nc{\indep}{\perp}
\DMO{\sink}{sink}
\nc{\fp}[1]{\MP_1({#1})}
\nc{\algname}{\texttt{QLearningPreds}\xspace}
\nc{\banditalg}{\texttt{BanditPreds}\xspace}
\nc{\polcomplex}{\texttt{PolComplex}\xspace}
\nc{\polsimple}{\texttt{PolicySelection}\xspace}
\nc{\deltaconst}{\texttt{DeltaConst}\xspace}
\nc{\deltaincr}{\texttt{DeltaIncr}\xspace}
\nc{\MWsig}{\MW^\sigma}
\nc{\MWtau}{\MW^\tau}
\nc{\MWrho}{\MW^\rho}

\DMO{\PR}{Pr}
\renewcommand{\Pr}{\PR}
\nc{\E}{\mathbb{E}}
\nc{\ra}{\rightarrow}

\title{Can Q-Learning be Improved with Advice?}

\author{Noah Golowich\thanks{MIT CSAIL. Email:  \texttt{nzg@mit.edu}. Supported by a Fannie \& John Hertz Foundation Fellowship and an NSF Graduate Fellowship.} \qquad \qquad  Ankur Moitra\thanks{MIT Math, SDSC and CSAIL. Email: \texttt{moitra@mit.edu}. This work was
supported in part by a Microsoft Trustworthy AI Grant, NSF CAREER Award CCF-1453261, NSF Large CCF1565235, a David and Lucile Packard Fellowship and an ONR Young Investigator
Award.}}
\date{October 25, 2021}

\begin{document}
\maketitle

\begin{abstract}
Despite rapid progress in theoretical reinforcement learning (RL) over the last few years, most of the known guarantees are worst-case in nature, failing to take advantage of structure that may be known a priori about a given RL problem at hand. In this paper we address the question of whether worst-case lower bounds for regret in online learning of Markov decision processes (MDPs) can be circumvented when information about the MDP, in the form of predictions about its optimal $Q$-value function, is given to the algorithm. We show that when the predictions about the optimal $Q$-value function satisfy a reasonably weak condition we call \emph{distillation}, then we can improve regret bounds by replacing the set of state-action pairs with the set of state-action pairs on which the predictions are grossly inaccurate. This improvement holds for both uniform regret bounds and gap-based ones. Further, we are able to achieve this property with an algorithm that achieves sublinear regret when given arbitrary predictions (i.e., even those which are not a distillation). Our work extends a recent line of work on \emph{algorithms with predictions}, which has typically focused on simple online problems such as caching and scheduling, to the more complex and general problem of reinforcement learning. 
\end{abstract}

\newpage
\tableofcontents
\newpage
\section{Introduction}
The study of worst-case algorithm design has traditionally been a mainstay of much of computer science, leading to provable and efficient algorithms for various tractable problems. However, many problems encountered in practice are often intractable, in the sense that efficient algorithms for them would violate widely held complexity theoretic hypotheses, or there are strong {unconditional} lower bounds on the amount of data required to achieve desired error bounds. As such, the study of beyond-worst case algorithm design \cite{roughgarden_beyond_2021}, which aims to use additional information about the structure of problem instances to improve algorithms' guarantees, has attracted significant attention in recent years.

An exciting approach to beyond-worst case algorithm design is to assume that the algorithm has access to certain \emph{predictions} regarding the nature of the problem instance at hand. For example, while  millions of samples may be required in order to teach a humanoid robot to walk starting from scratch, physical approximations of the robot's dynamics can be used to furnish an \emph{approximately} optimal policy, viewed as a prediction about the optimal policy. Then, starting from this predicted policy, we can  algorithmically fine-tune it using relatively few samples. This general approach, known as \emph{algorithm design with predictions} (or \emph{advice}), aims to improve an algorithm's guarantees, such as by reducing sample complexity, when it is given access to accurate predictions. It has been studied from a theoretical perspective in several recent works for online problems such as the ski rental problem, scheduling, caching, and many others \cite{mitzenmacher_algorithms_2020}. %
In this paper we address the design of algorithms with predictions for the much broader problem of reinforcement learning, in particular the setting of no-regret online learning in tabular Markov decision processes (MDPs). In turn, MDPs can be used to model learning problems in a plethora of settings including, for instance, personalized medicine, optimal control, and market design \cite{sutton_reinforcement_2018}.

\subsection{Model overview}
We consider the setting of a tabular finite-horizon episodic MDP with a finite state space $\MS$ consisting of $S$ states, a finite action space $\MA$ consisting of $A$ actions, and a horizon of length $H$ \cite{agarwal_reinforcement_2021}. In this setting, the rewards and transitions of the MDP are unknown, but the learning algorithm has the ability to simulate trajectories in the MDP corresponding to policies of its choice. In total the learner simulates $K \in \BN$ trajectories, each of which consists of $H$ steps; the total number of samples is then $T := KH$. The learner aims to minimize the regret, namely the difference between the reward it would have received had it always followed the optimal policy and its actual aggregate reward; we refer the reader to Section \ref{sec:prelim} for additional preliminaries. In Definition \ref{def:predicted} below, we introduce the setting of \emph{RL with $Q$-value predictions}, namely where the algorithm is given access to predictions $\til Q$ of the optimal $Q$-value function $Q_h^\st(x,a)$. Recall that for a state $x \in \MS$, action $a \in \MA$, and step $h \in [H]$, $Q_h^\st(x,a)$ denotes the cumulative expected reward when action $a$ is taken at state $x$ and step $h$, and thereafter the optimal policy is followed.
\begin{defn}[RL with $Q$-value predictions]
  \label{def:predicted}
  We assume the learning algorithm is given, at the onset of its interaction with the MDP, access to a
  collection of \emph{predictions} $\til Q = (\til Q_1, \ldots, \til Q_H)$, where each $\til Q_h : \MS \times \MA \ra \BR$ represents a prediction for the optimal $Q$-value function $Q_h^\st$. %
\end{defn}

For many important applications of reinforcement learning (RL), predictions $\til Q$ as in Definition \ref{def:predicted} may be readily available to a learning algorithm; for instance: %
\begin{itemize}
\item In robotics, powerful physics-based simulation engines allow easy collection of large amounts of data from simulated environments, but the learned policies from these simulated environments often do not transfer directly to real-world environments due to factors such as measurement error and mis-specification of the simulation parameters. To bridge this gap, information gleaned from learning in the simulated environment can be used as a prediction to be fine-tuned when interacting with the real-world environment. This general approach (sometimes called \emph{sim-to-real}) has attracted much attention recently, such as in  \cite{rusu_sim--real_2018,chebotar_closing_2019}. %
\item In applications of RL to healthcare, $Q$-learning based methods (such as deep $Q$-learning) are commonly employed \cite{yu_reinforcement_2020}. It is therefore reasonable to expect that predictions $\til Q$ for a given task can be computed based on data collected for similar tasks, such as from a previous iteration of a clinical trial to treat a particular disease. Such ideas were used in \cite{liao_personalized_2020}, where various parameters to an algorithm in a mobile health trial, \texttt{HeartSteps V2}, were set based on data collected in an earlier iteration, \texttt{HeartSteps V1}.\footnote{The algorithm used was an ad hoc approach tuned to the particular task, rather than $Q$-learning.}
\item More broadly, our framework provides an approach for the related problems of \emph{multi-task learning} and \emph{transfer learning} in RL \cite{taylor_transfer_2009,zhu_transfer_2021}. These problems consider an RL agent which wishes to perform well on multiple related tasks (e.g., a robot moving in an environment with gradually changing obstacles) throughout its lifetime. An estimate of $Q^\st$ for earlier tasks (e.g., as obtained by $Q$-learning) may be used as the input predictions $\til Q$ for later tasks. As discussed in Section \ref{sec:related-work}, there has been a large body of empirical work devoted to improving transfer learning in RL. Some of this work considers the reuse of certain representations of an MDP such as the $Q$-value function; our results can thus be interpreted as a theoretical justification for such techniques. %
\end{itemize}

In this paper we address the following question(s): \emph{Is it possible to leverage prior knowledge, in the form of access to predictions $\til Q$ as in Definition \ref{def:predicted}, to show a regret bound that beats the worst-case? Moreover, can we achieve such a result with an algorithm that still obtains sublinear regret when the predictions $\til Q$ are arbitrary?}

\subsection{Overview of results}
We begin by reviewing known results regarding worst-case regret in online learning for tabular MDPs. In this paper we prove both uniform regret bounds, which depend only on the parameters $S,A,H$ of the MDP as well as the number of samples $T$, as well as instance-dependent gap-based bounds, which we proceed to explain. 
The \emph{gap} for action $a$ at state $x$ and step $h$ is defined to be $\Delta_h(x,a) = V_h^\st(x) - Q_h^\st(x,a)$, where $V_h^\st$ and $Q_h^\st$ are the optimal value function and $Q$-value function, respectively.\footnote{$Q_h^\st$ was defined previously, and $V_h^\st(x) = \max_{a \in \MA} \{ Q_h^\st(x,a)\}$; see Section \ref{sec:prelim} for further details.} $\Delta_h(x,a)$ denotes the marginal loss incurred, relative to the optimal policy, when action $a$ is taken at state $x$ and step $h$. In this section (including in Theorems \ref{thm:deltaincr-informal} and \ref{thm:robust-informal}) we make the simplifying assumption that for each $(x,h)$ there is a unique optimal action $a$ (this assumption is relaxed in the full statements of our results in Section \ref{sec:results}). 
In this case \cite{xu_fine-grained_2021} exhibit an algorithm, \texttt{AMB}, which satisfies the following regret guarantee:
\begin{align}
\Regret{T} \leq & \widetilde O \left(  \min \left\{ \sqrt{H^5 SAT} ,\  \sum_{\substack{(x,a,h) \in \MS \times \MA \times [H]:\\ \Delta_h(x,a) > 0}} \frac{H^5}{\Delta_h(x,a)}\right\} \right)\label{eq:amb-regret-informal},
\end{align}

The regret bound (\ref{eq:amb-regret-informal}) is optimal up to lower-order terms in the following sense: the uniform regret bound of $\widetilde O(\sqrt{H^5 SAT})$ matches the minimax regret in tabular MDPs, $\widetilde O(\sqrt{H^2 SAT})$ \cite{jin_is_2018,zhang_almost_2020}, up to a factor of $\sqrt{H^3}$.\footnote{In this paper we generally disregard factors polynomial in $H$ and $\log(SAT)$, and do not attempt to optimize the dependence of our own bounds on $H$ and $\log(SAT)$.} Moreover, \cite[Proposition 2.2]{simchowitz_non-asymptotic_2019} shows that a term of the form $\sum_{\substack{(x,a,h) : \MS \times \MA \times [H] :\\ \Delta_h(x,a) > 0}} \frac{\log K}{\Delta_h(x,a)}$ must appear in the gap-based regret bound in (\ref{eq:amb-regret-informal}).

Our first main result addresses the following question: \emph{Suppose the learning algorithm has access to predictions $\til Q$ which are accurate on an unknown set of many state-action pairs. Then can we improve upon the worst-case regret bound (\ref{eq:amb-regret-informal}), as if we had fewer state-action pairs to begin with?}
We show an affirmative answer to this question, replacing the set of all state-action pairs with the set of  those for which $\til Q$ is inaccurate.  
In order for our improved bounds to ``kick in'', though, it is necessary that the predictions $\til Q$ satisfy an additional property, which we formalize as being a \emph{(approximate) distillation} of $Q^\st$ (Definition \ref{def:distillation}). We focus on the case of exact distillation here, which corresponds to $\ep = 0$ in Definition \ref{def:distillation}: we say that $\til Q$ is a distillation of $Q^\st$ if for each state $x \in \MS$ and step $h \in [H]$, letting $\pi_h^\st(x)$ denote the optimal action at $(x,h)$, it holds that $\til Q_h(x,\pi_h^\st(x)) \geq Q_h^\st(x,\pi_h^\st(x))$. 
Intuitively, $\til Q$ ``picks out'' the good action $\pi_h^\st(x)$ at $(x,h)$.\footnote{We give an example in Section \ref{sec:pred-properties} showing that this assumption is necessary to beat minimax lower bounds.} Of course, there could be additional actions $a'$ for which $\til Q_h(x,a')$ is equally large, but for which $a'$ is very much suboptimal at $(x,h)$.
Theorem \ref{thm:deltaincr-informal} shows that when the predictions are a distillation of $Q^\st$, the regret bound (\ref{eq:amb-regret-informal}) can be improved by replacing the set $\MS \times \MA \times [H]$ with a much smaller set consisting of tuples $(x,a,h)$ for which $\til Q_h(x,a)$ is grossly inaccurate:
\begin{theorem}[Simplified/informal version of Theorem \ref{thm:main-deltaincr}, item \ref{it:distillation-deltaincr}]
  \label{thm:deltaincr-informal}
  Suppose that $\til Q$ is guaranteed to be a distillation of $Q^\st$. Then there is an algorithm (\algname, Algorithm \ref{alg:main}\footnote{The parameter $\lambda$ is set to $0$ for the purposes of Theorem \ref{thm:deltaincr-informal}.}) which achieves regret
  \begin{align}
\widetilde O \left( \min \left\{ \sqrt{H^5 T \cdot |\foolingset|}, \sum_{(x,a,h) \in \foolingset} \frac{H^4}{\Delta_h(x,a)} \right\} \right)\label{eq:deltaincr-informal-regret},
  \end{align}
  where
  \begin{align}
    \foolingset := \left\{ (x,a,h) \in \MS \times \MA \times [H] : \til Q_h(x,a) > V_h^\st(x) \mbox{ or } \left( a \neq \pi_h^\st(x) \mbox{ and } \til Q_h(x,a) \geq V_h^\st(x)\right) \right\}.\label{eq:foolsimple}
  \end{align}
\end{theorem}
The regret bound (\ref{eq:deltaincr-informal-regret}) of Theorem \ref{thm:deltaincr-informal} depends on the set $\foolingset$ of action-state pairs $(x,a,h)$ for which $\til Q_h(x,a)$ %
is larger than the optimal value at $(x,h)$, namely $V_h^\st(x) = \max_{a \in \MA} Q_h^\st(x,a)$.
Recall that since the distillation property requires that $\til Q_h(x,\pi_h^\st(x)) \geq Q_h^\st(x,\pi_h^\st(x))$ for all $(x,h)$, if $\til Q_h(x,\pi_h^\st(x)) \neq Q_h^\st(x,\pi_h^\st(x))$, then $(x,\pi_h^\st(x),h) \in \MF$. %
In the full version of Theorem \ref{thm:main-deltaincr} we relax the condition of exact equality by allowing for an approximate version of distillation (Definition \ref{def:distillation}) and an approximate version of the set $\foolingset$ which set call the \emph{fooling set}  (Definition \ref{def:fooling}). %

To complement Theorem \ref{thm:deltaincr-informal}, it is desirable to have a single algorithm which obtains nontrivial regret bounds (i.e., sublinear regret) for arbitrary predictions $\til Q$, i.e., even those which are not an (approximate) distillation of $Q^\st$, which \emph{also} obtains improved regret bounds (such as (\ref{eq:deltaincr-informal-regret})) when $\til Q$ {is} an approximate distillation. The former guarantee is often known as \emph{robustness} in the literature on algorithms with predictions \cite{lykouris_competitive_2020,mitzenmacher_algorithms_2020}. Robustness in this context is well-motivated since the predictions are often generated by an ad hoc procedure with few provable guarantees (such as the use of deep RL techniques on a simulated environment to estimate $\til Q$ for use in the real-world environment), making them liable to be grossly inaccurate. Theorem \ref{thm:robust-informal} below gives such a guarantee for the case of uniform regret bounds; Theorems \ref{thm:main-deltaincr} and \ref{thm:main-deltaconst} provide more general robustness bounds that cover the gap-based case as well.
\begin{theorem}[Simplified/informal version of Theorem \ref{thm:main-deltaincr}, uniform version]
  \label{thm:robust-informal}
  There is an algorithm (\algname, Algorithm \ref{alg:main}) which satisfies the following two guarantees, when given as input a parameter $\lambda \in \left( \frac{SAH^4}{T}, 1 \right)$ and predictions $\til Q$:
  \begin{enumerate}
  \item \label{it:worstcase-informal} For an arbitrary choice of $\til Q$, the regret is $\widetilde O \left( \sqrt{\frac{TSAH^{10}}{\lambda}}\right)$.
  \item \label{it:distillation-informal} If the predictions $\til Q$ are a distillation of $Q^\st$, then the regret is
    \begin{align}
\widetilde O \left( \sqrt{\lambda \cdot SATH^{10}} + \sqrt{|\foolingset| \cdot T H^5 }\right)\label{eq:regret-distillation-informal},
    \end{align}
    where $\foolingset$ was defined in (\ref{eq:foolsimple}). %
  \end{enumerate}
\end{theorem}
As was the case for Theorem \ref{thm:deltaincr-informal}, the notion of distillation and the set $\foolingset$ are relaxed to their approximate analogues in Theorems \ref{thm:main-deltaincr} and \ref{thm:main-deltaconst}. Notice that there is a tradeoff (mediated by the parameter $\lambda$) in Theorem \ref{thm:robust-informal} between the regret for arbitrary $\til Q$ (i.e., the robustness) and the improved regret bound for $\til Q$ that is an $\ep$-approximate distillation of $Q^\st$. This is a common occurrence in the study of algorithms with predictions, occurring, for instance, in the ski rental problem \cite{purohit_improving_2018,wei_optimal_2020} and the problem of non-clairvoyent scheduling \cite{purohit_improving_2018,wei_optimal_2020}.

\subsection{Warm up: Stochastic multi-armed bandits}
In this section we give a brief overview of the techniques used to prove our regret upper bounds; a detailed description of the algorithm is given in Section \ref{sec:alg-overview}, and a more in-depth overview of the proof is given in Section \ref{sec:proofs-overview}. As a warm-up, we begin by addressing the easier case of multi-armed bandits: in this case there is a single state, and for each action (also known as an \emph{arm}) $a \in \MA$, we denote the expected reward of taking $a$ as $Q^\st(a)$.\footnote{In the bandit setting, the reward received upon taking action $a$ is a random variable; this is in contrast to the full RL setting where we assume the immediate rewards $r(x,a)$ are deterministic. This discrepancy does not lead to any significant differences in the algorithm or analysis, though.} Again we assume that there is a unique optimal action, denoted by $a^\st$. The algorithm receives at onset a function $\til Q : \MA \ra [0,1]$ denoting predictions for the mean reward of each arm. Moreover, $\til Q$ is a distillation of $Q^\st$ if $\til Q(a^\st) \geq Q^\st(a^\st)$. The below proposition specializes Theorem \ref{thm:robust-informal} to the multi-armed bandit setting.
\begin{proposition}[Bandit case]
  \label{prop:bandits}
  There is an algorithm (\banditalg, Algorithm \ref{alg:bandit}) which satisfies the following two guarantees, when given as input a parameter $\lambda \in \left( \frac{A}{T}, 1 \right)$ and predictions $\til Q$:
  \begin{enumerate}
  \item If the predictions $\til Q$ are a distillation of $Q^\st$, then the regret is $\widetilde O(\sqrt{|\MG| \cdot T} + \sqrt{\lambda \cdot AT})$, where $\MG := \left\{ a \in \MA \backslash \{ a^\st \} : \til Q(a) \geq  Q^\st(a^\st)\right\}$. \label{it:bandit-intro-distillation}
  \item For an arbitrary choice of $\til Q$, the regret is $\widetilde O \left(\sqrt{ \frac{TA}{\lambda} }\right)$.\label{it:bandit-intro-robustness}
  \end{enumerate}
\end{proposition}
For simplicity we have only stated uniform regret bounds in Proposition \ref{prop:bandits}, but gap-based regret bounds specializing Theorems \ref{thm:main-deltaincr} and \ref{thm:main-deltaconst} to the bandit case may readily be derived using similar techniques. Algorithm \ref{alg:bandit}, which establishes Proposition \ref{prop:bandits}, generally speaking aims to choose the action $a$ which maximizes $\til Q(a)$. Of course, when $\til Q$ is not accurate (e.g., because it is not a distillation or $\MG$ is nonempty) Algorithm \ref{alg:bandit} must make the following modifications: %
\begin{itemize}
\item To handle non-optimal actions $a$ which are in the set $\MG$ (i.e., actions for which $\til Q$ predicts them as having higher reward than $a^\st$), we maintain both upper and lower confidence bounds for the mean reward of each action $a$. For each $t$ we then project $\til Q(a)$ onto the interval $[\Qu^t(a), \Qo^t(a)]$ and use the resulting projected value instead  of $\til Q(a)$. 
\item Even with the use of upper and lower confidence bounds, we could run into the following difficulty: if $\til Q(a) = Q^\st(a)$ for all $a \neq a^\st$, but $\til Q(a^\st) \ll Q^\st(a^\st)$ (which can happen when $\til Q$ is not a distillation), then choosing the action $a$ to maximize $\til Q(a)$ would simply choose the second-best action at all time steps, thus incurring linear regret. To deal with this situation, we insert an initial \emph{exploration} phase consisting of $\lambda T$ time steps, in which we choose an action with the highest upper confidence bound (as per the UCB algorithm \cite[Chapter 7]{lattimore_bandit_2020}). If $\til Q(a^\st)$ is significantly sub-optimal, this initial exploration phase will discover that and subsequently learn to ignore the prediction $\til Q(a^\st)$ (i.e., round it up to $\Qu^t(a^\st)$, which, over time, will approach $Q^\st(a^\st)$). 
\end{itemize}
The case of full RL (i.e., learning in MDPs) requires significant innovation beyond the above techniques for the multi-armed bandit case. At a high level, this occurs because errors in $\til Q$ can compound over multiple steps $h$ in the standard $Q$-learning updates. To handle such challenges, we have to use a more sophistocated rule to modify $\til Q(a)$ over time than simply projecting it onto $[\Qu^t(a), \Qo^t(a)]$. Additionally, the initial exploration phase described above must be made \emph{state-specific}, meaning that different states may leave the exploration phase at different times, according to the current value estimates at each respective state. We refer the reader to Sections \ref{sec:alg-overview} and \ref{sec:proofs-overview} for further details. 

\subsection{Related work}
\label{sec:related-work}
\paragraph{Algorithms with predictions} Many recent works studying algorithms with predictions have primarily focused on relatively specific online problems including the  ski rental problem \cite{purohit_improving_2018,wei_optimal_2020}, scheduling \cite{purohit_improving_2018,wei_optimal_2020,mitzenmacher_scheduling_2019,lattanzi_online_2019}, caching \cite{lykouris_competitive_2020,rohatgi_near-optimal_2019}, design of bloom filters \cite{kraska_case_2018,mitzenmacher_model_2019}, and revenue optimization \cite{medina_revenue_2017}; see also \cite{mitzenmacher_algorithms_2020} for an overview of the above papers.
In many of these problems, the performance parameter optimized by the algorithm (and improved with access to predictions) is the \emph{competitive ratio}, namely the ratio between a cost measure specific to the problem and the optimal cost in hindsight, rather than the regret. There has also been a fruitful line of work showing that by using predictions it is possible, using variants of \emph{optimistic mirror descent}, to significantly decrease the \emph{regret} in settings including online linear optimization \cite{hazan_extracting_2010,rakhlin_online_2012,rakhlin_optimization_2013,steinhardt_adaptivity_2014,mohri_accelerating_2015,bhaskara_online_2020} and contextual bandits \cite{wei_taking_2020}. As some of these works include the case of bandit feedback, they might seem to generalize Proposition \ref{prop:bandits}; however, this is not the case, since they face the significant limitation that a prediction is required by the algorithm at each time step, and the regret bound depends on the aggregate distance between the predictions and the realized values of the cost vectors or rewards over all time steps. In the setting of stochastic multi-armed bandits, this would require the predictions to track the noise of the realized reward over all $T$ time steps; in contrast, Proposition \ref{prop:bandits} only requires a single set of predictions which must be close to the mean reward vector.

\paragraph{Transfer learning in RL} More closely related to our results is a collection of work which proposes to solve the problem of \emph{transfer learning in RL} \cite{taylor_transfer_2009,zhu_transfer_2021} by reusing information (such as $Q$-values) from certain RL tasks in order to solve related RL tasks. In the particular case of $Q$-value reuse, the $Q$-values for each successive task may be initialized as some function (e.g., the mean) of the $Q$-values from the previous tasks; these $Q$-values are then updated over the course of the learning procedure for the current task \cite{singh_transfer_1992,asada_vision-based_1994,tanaka_multitask_2003,torrey_using_2005,taylor_transfer_2009-1}. Many of these papers show that doing so outperforms an initialization of $Q$ which is agnostic to previous tasks. These works, however, are purely empirical in nature, with no supporting theory. %

Very recently there has been some effort to perform theoretical analyses for such transfer learning techniques; \cite{tkachuk_effect_2021} shows that if the algorithm is given at onset predictions $\til Q$ which are known to be equal to $Q^\st$ at all state-action pairs except a \emph{single known} state-action pair at step $h=1$, then it is possible to achieve regret $\widetilde O(\sqrt{H^2 T})$ using $Q$-learning (thus eliminating the dependence on $SA$). Additionally, the recent work \cite{zhang_provably_2021} shows that if $M$ agents are interacting with separate MDPs whose optimal $Q$-value functions are known to be $\ep$-close in $\ell_\infty$ distance, then by sharing information about their respective MDPs, they can decrease their aggregate regret by a factor of $\sqrt{M}$ (in the uniform case) or $M$ (in the gap-based case). Unlike our work, \cite{zhang_provably_2021} does not show that minimax regret bounds can be beaten for a single MDP; moreover, both \cite{tkachuk_effect_2021,zhang_provably_2021} do not consider any notion of robustness nor do they allow relaxations to the $\ell_\infty$-closeness of the $Q$-value functions.

\paragraph{Theoretical RL background}
The minimax optimal regret for online learning in tabular MDPs is (up to polylogarithmic factors) $\Theta(\sqrt{SATH^2})$; the lower bound is shown in \cite{jin_is_2018}, and the upper bound is known for both model-based algorithms such as \texttt{UCBVI} \cite{azar_minimax_2017}, as well as the model-free algorithm \texttt{UCB-Advantage} \cite{zhang_almost_2020} (which is a variant of the $Q$-learning algorithm). Non-asymptotic gap-based upper bounds for tabular MDPs were shown in \cite{simchowitz_non-asymptotic_2019} using the model-based \texttt{StrongEuler} algorithm, a variant of \texttt{EULER} \cite{zanette_tighter_2019}; additional algorithms achieving gap-based bounds were shown in \cite{lykouris_corruption_2020,yang_q-learning_2021,xu_fine-grained_2021}. Our gap-based bounds are based on the techniques in \cite{xu_fine-grained_2021}. Very recently some works \cite{dann_beyond_2021,tirinzoni_fully_2021,wagenmaker_beyond_2021} have derived new instance-dependent bounds in RL, such as by making alternative definitions of gaps; using insights from these works to improve our gap-based bounds is an interesting direction left for future work. The books \cite{agarwal_reinforcement_2021,lattimore_bandit_2020} contain a more comprehensive overview of the flurry of recent work in theoretical RL.

\section{Preliminaries}
\label{sec:prelim}
We consider the setting of a tabular finite-horizon episodic Markov decision process (MDP) $M = (\MS, \MA, H, \BP, r)$, where $\MS$ denotes the (finite) state space, $\MA$ denotes the (finite) action space, $H \in \BN$ denotes the horizon, $\BP = (\BP_1, \ldots, \BP_H)$ denotes the transitions, and $r = (r_1, \ldots, r_H)$ denotes the reward functions. In particular, for each $h \in [H]$, $\BP_h(x' | x,a)$ (for $x,x' \in \MS$, $a \in \MA$) denotes the probability of transitioning to $x'$ from $x$ at step $h$ when action $a$ is taken; and $r_h(x,a)$ denotes the reward received when at state $x$ and step $h$ when action $a$ is taken. We assume each reward lies in $[0,1]$, i.e., $r_h : \MS \times \MA \ra [0,1]$.  We write $S := |\MS|$ and $A := |\MA|$. A \emph{policy} $\pi$ is a collection of mappings $\pi_h : \MS \ra \MA$, for each $h \in [H]$.\footnote{Note that our setting of finite-horizon MDPs is equivalent to the setting of \emph{layered} MDPs in the literature (e.g., \cite{xu_fine-grained_2021}), where a different copy of the state space $\MS$ is created for each step $h$, and transitions from states in layer $h$ always to go states in layer $h+1$.}

In each episode, a state $x_1$ is picked by an adversary. For each $h \in [H]$, the agent observes the state $x_h$, picks an action $a_h \in \MA$ (usually given according to some policy $\pi$, i.e., $a_h = \pi_h(x_h)$), receives reward $r_h(x_h, a_h)$, and transitions to a new state $x_{h+1}$, drawn according to $\BP_h(\cdot | x_h, a_h)$. Upon receiving the reward $r_H(x_H, a_H)$ at the final step $H$, the episode ends. %
For a policy $\pi$, we let $V_h^\pi : \MS \ra \BR$ denote the \emph{$V$-value function} at step $h$; in particular, $V_h^\pi(x)$ gives the expected total reward received by the agent when it starts in state $x$ at step $h$ and thereafter follows policy $\pi$. In a similar manner, we let $Q_h^\pi : \MS \times \MA \ra \BR$ denote the \emph{$Q$-value function} at step $h$; $Q_h^\pi(x,a)$ gives the expected total reward received by the agent when it starts in state $x$ at step $h$, takes action $a$, and thereafter follows policy $\pi$. Formally, $V_h^\pi$ and $Q_h^\pi$ are defined as follows:
\begin{align}
V_h^\pi(x) := \E_\pi \left[ \sum_{h'=h}^H r_h(x_{h'}, a_{h'}) | x_h = x \right], \qquad Q_h^\pi(x,a) = \E_\pi\left[ \sum_{h'=h}^H r_{h'}(x_{h'}, a_{h'}) | x_h = x, a_h = a \right],\nonumber
\end{align}
where $\E_\pi[\cdot]$ denotes that $\pi$ is used to choose the action at each state.

We let $\pi^\st$ denote the \emph{optimal policy}, namely the policy which maximizes $V_h^{\pi^\st}(x)$ for all $(x,h) \in \MS \times [H]$. We write $V_h^\st(x) := V_h^{\pi^\st}(x)$ for all $x,h$. With slight abuse of notation, we let $\BP_h$ denote the Markov operator $\BP_h : \BR^{S} \ra \BR^{S \times A}$, defined by, for any value function $V_{h+1} : \MS \ra \BR$, $(\BP_h V_{h+1})(x,a) := \E_{x' \sim \BP_h(\cdot | x,a)}[V_{h+1}(x')]$. The following relations (\emph{Bellman equation} and \emph{Bellman optimality equation}) are standard and follow easily from the definitions: for all $(x,a,h) \in \MS \times \MA \times [H]$,
\begin{align}
  \begin{cases}
    V_h^\pi(x) =& Q_h^\pi(x, \pi_h(x)) \\
    Q_h^\pi(x,a) =& (r_h + \BP_h V_{h+1}^\pi)(x,a) \\
    V_{H+1}^\pi(x) =& 0 %
  \end{cases}\nonumber
\quad \mbox{ and } \quad
\begin{cases}
  V_h^\st(x) =& \max_{a \in \MA} Q_h^\st(x,a) \\
  Q_h^\st(x,a) =& (r_h + \BP_h V_{h+1}^\st(x,a) \\
  V_{H+1}^\st(x) =& 0 .
\end{cases}
\end{align}
For some $K \in \BN$, over a series of $K$ \emph{episodes}, the RL agent interacts with the MDP $M$ as follows: for each $k \in K$, the agent chooses a policy $\pi^k$, and applies the policy $\pi^k$ in the MDP to obtain a \emph{trajectory} $(x_1^k, a_1^k, r_1^k), \ldots, (x_H^k, a_H^k, r_H^k)$, as explained above; here $r_h^k := r_h(x_h^k, a_h^k) \in [0,1]$ denotes the reward received at step $h$. We measure the agent's performance with the \emph{regret}:
\begin{align}
\Regret{K} := \sum_{k=1}^K \E\left[ V_1^\st(x_1^k) - V_1^{\pi^k}(x_1^k)\right]\nonumber,
\end{align}
where the expectation is taken over the randomness of the environment (in particular, the policies $\pi^k$ are random variables since they depend on trajectories from previous episodes).

\paragraph{Notation for gap-based bounds} In this paper we will derive gap-dependent regret bounds; for $(x,a,h) \in \MS \times \MA \times [H]$, the \emph{gap at $(x,a,h)$} is defined as:
\begin{align}
\Delta_h(x,a) := V_h^\st(x) - Q_h^\st(x,a)\nonumber.
\end{align}
The gap denotes the marginal sub-optimality in reward the agent suffers as a result of taking action $a$ at state $x$ and step $h$. For $\ep > 0$, we write, for $(x,h) \in \MS \times [H]$,
\begin{align}
\Aopt_{h,\ep}(x) := \{ a \in \MA : \Delta_h(x) \leq \ep \}\nonumber
\end{align}
to denote the set of actions with gap at most $\ep$ at $(x,h)$. For $x \in \MS$ and $h \in [H]$, define $\delmins{h}(x) := \min_{a \not \in \Aopt_{h,0}(x)} \{ \Delta_h(x,a) \}$ to be the minimum positive gap at $(x,h)$. Also define the minimum positive gap in the entire MDP to be $\delmin := \min_{x,a,h : \Delta_h(x,a) > 0} \{ \Delta_h(x,a) \}$. Following \cite{xu_fine-grained_2021}, our gap-based bounds will have a term depending on the number of state-action pairs which are optimal and for which there is \emph{not} a unique optimal action at that state, i.e., the size of the set:
\begin{align}
\Amul := \{ (x,a,h) \in \MS \times \MA \times [H] :\ \Delta_h(x,a) = 0 \quad \mbox{ and } \quad |\Aopt_{h,0}(x)| > 1 \}.\nonumber
\end{align}

\paragraph{Prior worst-case regret bound}
In the special case that each state has a unique optimal action we discussed the worst-case regret bound (\ref{eq:amb-regret-informal}) and prior work showing its optimality. In the general case, \cite[Theorem B.1 \& Corollary B.10]{xu_fine-grained_2021} showed the following regret bound:
  \begin{align}
\Regret{T} \leq & O \left( H^2 SA + \log(SAT) \cdot \min \left\{\frac{H^5 |\Amul|}{\delmin} +  \sum_{\substack{(x,a,h) \in \MS \times \MA \times [H]:\\ \Delta_h(x,a) > 0}} \frac{H^5}{\Delta_h(x,a)} ,\ \sqrt{H^5 SAT} \right\} \right)\label{eq:amb-regret}.
\end{align}
\cite[Theorem 5.1]{xu_fine-grained_2021} shows that a term of the form $\log K \cdot \frac{|\Amul|}{\delmin}$ is necessary, even in the presence of the term $\sum_{\substack{(x,a,h) \in \MS \times \MA \times [H]:\\ \Delta_h(x,a) > 0}} \frac{1}{\Delta_h(x,a)}$ of the regret bound. Thus, in general, the bound (\ref{eq:amb-regret}) cannot be improved by more than $\poly(H, \log(SAT))$ factors. %

\paragraph{Additional notation}
Given a real number $x$, let $[x]_+$ denote $x$ if $x > 0$, and 0 otherwise. We will write $T = HK$ to denote the total number of samples over $K$ episodes; note that $\Regret{K} \leq T$ always holds. We also set $\iota := \log(SAT)$. For $(x,a,h,k) \in \MS \times \MA \times [H] \times [K]$, $N_h^k(x,a)$ denotes the number of episodes before episode $k$ in which $(x,a,h)$ is visited, i.e., action $a$ was taken at state $x$ and step $h$ ($N_h^k(x,a)$ is also defined in step \ref{it:nhk} of Algorithm \ref{alg:main}). For integers $i \geq 1$ and $(x,a,h) \in \MS \times \MA \times [H]$, we let $k_h^i(x,a)$ denote the episode $k$ which is the $i$th episode that $(x,a,h)$ was visited. If no such episode exists, we set $k_h^i(x,a) = K+1$ as a matter of convention. 

\section{Learning in MDPs with predictions: main results}
\label{sec:results}
\subsection{Properties of the predictions $\til Q$}
\label{sec:pred-properties}
Our main result shows that in the presence of arbitrary predictions $\til Q$, we are able to obtain a sublinear regret bound for our algorithm \algname, and moreover, if $\til Q$ satisfies an additional property, then we can obtain an improved regret bound that can \emph{beat} the minimax regret bounds for learning in MDPs (i.e., (\ref{eq:amb-regret})), replacing the space $\MS \times \MA \times [H]$ with a smaller space representing the set of state-action pairs where $\til Q$ is inaccurate (consistency). Definition \ref{def:distillation} below captures the additional property (referred to as being an \emph{approximate distillation} of $Q^\st$) that $\til Q$ needs to satisfy in order to obtain improved regret bounds.

To motivate the definition, consider the setting where there is a single state $x_0$ and $H=1$ (which is equivalent to the stochastic multi-armed bandit problem). Moreover suppose there is a unique optimal action $a^\st$ with reward $1$ and all other $A-1$ actions have reward $1-\Delta$ for some positive $\Delta < 1/A$. If we are given the predictions $\til Q_1^{\FF}$, where $\til Q_1^{\FF}(x_0,a) := 1-\Delta$ for all $a$, then $\til Q_1^{\FF}$ is only incorrect at a single action (namely, $a^\st$), but it provides no information about what $a^\st$ is, and it is straightforward to show that, even given $\til Q_1^{\FF}$, the regret of any algorithm must be $\Omega(A/\Delta)$, giving no improvement over the setting without predictions \cite[Chapter 16]{lattimore_bandit_2020}. On the other hand, consider  the predictions $\til Q_1^{\TTT}$ defined as equal to $Q_1^\st$ except at a single (unknown) \emph{non-optimal action} $a'$.\footnote{$\til Q_1^{\TTT}(x,a')$ can be set to any real number.} Though both $\til Q_1^{\FF}, \til Q_1^{\TTT}$ both differ from $\til Q_1^\st$ at a single action, it will follow from Theorem \ref{thm:main-deltaincr} (with $\lambda = A/T$) that given $\til Q_1^{\TTT}$, \algname obtains the much smaller regret bound of $\widetilde O(1/\Delta)$. As this example shows, a set of accurate predictions $\til Q$ cannot entirely mitigate the exploration problem: even if the predictions are accurate at nearly all states and actions, if they do not provide any information as to the identity of the optimal action at a given state (e.g., as for $\til Q_1^{\FF}$), then we cannot hope to beat existing regret bounds. The notion of \emph{approximate distillation}, defined below, formalizes the notion that $\til Q$ must provide information about the optimal action at each state:
\begin{defn}[Approximate distillation]
  \label{def:distillation}
  Consider a predicted $Q$-value function $\til Q \in \BR^{[H] \times \MS \times \MA}$. For $\ep > 0$, we say that $\til Q$ is an \emph{$\ep$-approximate distillation} of the optimal value function $Q^\st$ if the following holds: for each $(x,h) \in \MS \times [H]$, there is some $a \in \MA$ so that
  \begin{align}
\Delta_h(x,a) + [ Q_h^\st(x,a) - \til Q_h(x,a)]_+ \leq \ep \nonumber.
  \end{align}
\end{defn}
In words, Definition \ref{def:distillation} requires that for each $(x,h)$, there is some action $a$ which is nearly optimal and for which $\til Q_h(x,a)$ does not greatly underestimate the value of $Q_h^\st(x,a)$. 

Next we define the \emph{fooling set} for any $\til Q$, which is the set of tuples $(x,a,h)$ for which $\til Q_h(x,a)$ differs from $Q_h^\st(x,a)$ in a particular way:
\begin{defn}[Fooling set]
  \label{def:fooling}
    Given a set of predictions $\til Q$ for any $\ep_1, \ep_2 > 0$, we define the set of \emph{$(\ep_1, \ep_2)$-fooling tuples $(x,a,h)$}, denoted $\fool{\ep_1}{\ep_2} \subset \MS \times \MA \times [H]$, to be the set of tuples $(x,a,h)$ so that 
  $$
 \til Q_h(x,a) - Q_h^\st(x,a) \geq \Delta_h(x,a) - \ep_1 \geq \ep_2 - \ep_1 \quad \mbox{ or } \quad \til Q_h(x,a) > V_h^\st(x) + \ep_2.
$$
In this context, we will always have $\ep_2 > \ep_1 > 0$. %
\end{defn}
Notice that we could alternatively define the fooling set as those $(x,a,h)$ for which $|\til Q_h(x,a) - Q_h^\st(x,a)| > \ep_2 - \ep_1$; this set, however, is in general larger than $\fool{\ep_1}{\ep_2}$, and so using $\fool{\ep_1}{\ep_2}$ allows us to obtain tighter regret bounds.

One of our results will also make use of the following assumption on $\til Q$: \begin{defn}[Optimal fooling actions]
  \label{def:ofa}
For $\ep' > 0$, we say that predictions $\til Q$ \emph{lack $\ep'$-fooling optimal actions} if there is no $(x,h)$ with multiple optimal actions (i.e., for which $|\Aopt_{h,0}(x)| > 1$) so that for some $a \in \Aopt_{h,0}(x)$, $\til Q_h(x,a) > V_h^\st(x) + \ep'$. 
\end{defn}
Note that in the context of Definition \ref{def:ofa}, $\til Q_h(x,a) > V_h^\st(x) + \ep'$ implies that $(x,a,h) \in \fool{\ep}{\ep'}$ for any $\ep$, explaining the terminology of the definition.

\subsection{Main theorems}
A common thread in the literature on algorithms with predictions is an inherent tradeoff between an algorithm's robustness and its accuracy when it receives correct predictions (sometimes called \emph{consistency})  \cite{purohit_improving_2018,wei_optimal_2020}. Such a tradeoff occurs in our setting too. To describe this tradeoff, we introduce a parameter $\lambda \in (0,1)$: %
as $\lambda$ decreases to 0, the regret in the presence of predictions which are an (approximate) distillation improves but the robustness (i.e., regret in the presence of arbitrary predictions) worsens.

\paragraph{$\lambda$-Cost}
We will be able to obtain both gap-based regret bounds and (instance-independent) uniform ones for both robustness and consistency in the presence of predictions. To simplify the dependence of these bounds on the parameter $\lambda$ introduced above, we define the \emph{$\lambda$-cost} for an MDP $M$ as follows: given an MDP $M$, a value $T \in \BN$ and a value $\lambda \in (0,1)$, the $\lambda$-cost of $M$, denoted $\CC_{M, T,\lambda}$, is the following quantity:
\begin{align}
\CC_{M,T,\lambda}:= \min \left\{ \sqrt{\lambda \cdot TSAH^8 \iota},\ H^7 \iota \cdot \left( \sum_{(x,a,h) \in \MS \times \MA \times [H] : a \not \in \Aopt_{h,0}(x)} \frac{1}{\Delta_h(x,a)} + \frac{|\Amul|}{\delmin} \right) \right\}\label{eq:lambda-complexity}
\end{align}
Recall that $\iota = \log(SAT)$. Note that, ignoring ${\rm poly}(H, \iota)$ factors, $\CC_{M,T,\lambda}$ is in general no greater than the worst-case regret bound (\ref{eq:amb-regret}): moreover, if the first term in the minimum in (\ref{eq:lambda-complexity}) (i.e., $\sqrt{\lambda \cdot TSAH^8\iota}$) is much smaller than the second, then  due to the factor of $\sqrt{\lambda}$ in this term, $\CC_{M,T,\lambda}$ will be much smaller than the right-hand side of  (\ref{eq:amb-regret}) (again, ignoring ${\rm poly}(H,\iota)$ factors). %

\paragraph{Explicit-$\lambda$ guarantee}
Our first main result is stated below; for simplicity, we present here the result under the additional assumption that each $(x,h)$ has a unique optimal action (i.e., $|\Aopt_{h,0}(x)| = 1$); this assumption has been made previously in \cite{xu_fine-grained_2021}. As we show in an extended version of the theorem (see the version in Section \ref{sec:main-thms-proof}), this assumption may be removed if we assume that $\delmin$ is known to the algorithm; further, our second main result (Theorem \ref{thm:main-deltaconst}) avoids making either assumption altogether. Theorem \ref{thm:main-deltaincr} states that the regret of \algname (Algorithm \ref{alg:main}) under arbitrary predictions $\til Q$ is $\widetilde O(\frac{H}{\lambda} \cdot \CC_{M,T,\lambda})$, whereas the regret under accurate predictions (i.e., predictions which are an approximate distillation) is the sum of $\widetilde O(H \cdot \CC_{M,T,\lambda})$ plus a quantity that grows as the degree of accuracy of the predictions degrades.
\begin{theorem}
  \label{thm:main-deltaincr}
  Suppose that for each $(x,h)$ there is a unique optimal action (i.e., $|\Aopt_{h,0}(x)| =1$). 
  The algorithm \algname (Algorithm \ref{alg:main}) with the \deltaincr subroutine (Algorithm \ref{alg:delta-incr}) with parameter $\tildm = 0$ satisfies the following two guarantees, when given as input a parameter $\lambda \in [0,1]$ and predictions $\til Q$:
\begin{enumerate}
\item Suppose $\lambda \geq \frac{SAH^4}{T}$. Then for an arbitrary choice of input predictions $\til Q$, the regret of \algname is $O ( \frac{H\iota}{\lambda} \cdot \CC_{M,T,\lambda})$.\label{it:worstcase-deltaincr}
\item \label{it:distillation-deltaincr} Fix any $\ep > 0$, and set $\ep' = 4\ep \cdot (H+1)$. When the input predictions $\til Q$ are an $\ep$-approximate distillation of $Q^\st$ (Definition \ref{def:distillation}), the regret of \algname is
  \begin{align}
    \hspace{-1cm}
O \left( H^2 \iota \cdot \CC_{M,T,\lambda} + \ep' TH+  \min \left\{ \sqrt{H^5 T \iota \cdot |\fool{\ep'/2}{\ep'}|}, \sum_{(x,a,h) \in \fool{\ep'/2}{\ep'}} \frac{H^4 \iota}{[\Delta_h(x,a) - \ep'/2]_+}\right\} \right)\label{eq:fooling-regret-deltaincr}.
  \end{align}
\end{enumerate}
\end{theorem}
To exhibit the full power afforded by bounds in terms of $\CC_{M,T,\lambda}$, we consider the following example. 
  Suppose the algorithm is promised that at least one of the following holds: either $\til Q = Q^\st$ (in particular, $\til Q$ is a 0-approximate distillation),\footnote{More generally, we could assume that $\til Q$ is an $\ep$-approximate distillation and that $\fool{\ep'/2}{\ep'}$ is small.} or it holds that all non-zero gaps are at least a constant, which implies that $\CC_{M,T,\lambda} \leq \poly(H,\iota) \cdot O(SA)$; however, which of these possibilities holds is unknown. Then by choosing $\lambda = \sqrt{\frac{SA}{T}}$ in Theorem \ref{thm:main-deltaincr} (and using $\ep = \ep' = 0$), we obtain a regret bound of $\poly(H,\iota) \cdot O(SA)$ (which is independent of $K$) in the case that $\til Q = Q^\st$, and a regret bound of $\poly(H,\iota) \cdot O(\sqrt{SAT})$ in the other case. Thus we always manage to achieve regret at least as small as the minimax bound of $\widetilde O(\sqrt{H^2SAT})$ (up to $\poly(H,\iota)$ factors), and in the former case can get a much-improved regret bound that is polylogarithmic in $K$.

\paragraph{Implicit-$\lambda$ guarantee}
Unlike in much of the literature on algorithms with predictions \cite{mitzenmacher_algorithms_2020}, the quantity $\CC_{M,T,\lambda}$ which appears in our regret bounds is in general \emph{unknown} to the algorithm, as the quantities $\Delta_h(x,a), |\Amul|, \delmin$ are all unknown. Therefore, the standard paradigm in which a user chooses a parameter $\lambda$ and then runs an algorithm depending on $\lambda$ is somewhat less well-motivated because the user does not have an explicit formula for how the choice of $\lambda$ influences the regret bounds in the case when either the predictions are accurate or inaccurate. Therefore, in our next main result, Theorem \ref{thm:main-deltaconst}, we adopt the alternative procedure in which the user instead inputs a parameter $\CR < T$. Given $\CR$, the algorithm's robustness (i.e., performance under arbitrary predictions) is guaranteed to be $O(\CR)$, while the performance under accurate predictions grows with $\CC_{M,T,\wh \lambda}$ for $\wh \lambda$ implicitly chosen optimally so as to still guarantee regret $O(\CR)$ in the worst case.
\begin{theorem}
  \label{thm:main-deltaconst}
  The algorithm \algname with the \deltaconst subroutine satisfies the following two guarantees, when given as input a parameter $\CR \in [SAH^3, \frac{T}{SA}]$ %
  and predictions $\til Q$:
  \begin{enumerate}
  \item If $\CR \geq \CC_{M,T,1}$, then for an \emph{arbitrary} choice of input predictions $\til Q$, the regret of \algname is $O(\CR)$.\label{it:worstcase-deltaconst}
  \item Fix any $\ep > 0$, and set $\ep' = 4\ep \cdot (H+1)$. When the input predictions $\til Q$ are an $\ep$-approximate distillation of $Q^\st$ (Definition \ref{def:distillation}) and lack $\ep'$-fooling optimal actions (Definition \ref{def:ofa}), the regret of \algname is
    \begin{align}
      O\left( H \cdot \CC_{M,T,\wh \lambda} + \ep' TH+\min \left\{ \sqrt{H^5 T \iota \cdot |\fool{\ep'/2}{\ep'}|}, \sum_{(x,a,h) \in \fool{\ep'/2}{\ep'}} \frac{H^4 \iota}{[\Delta_h(x,a) - \ep'/2]_+}\right\} \right),\label{eq:fooling-regret}
    \end{align}
    where $\wh \lambda \in (0,1)$ is chosen so that $\frac{1}{\wh \lambda} \cdot \CC_{M,T,\wh \lambda} = \CR$. 
    \label{it:distillation-deltaconst}
  \end{enumerate}
\end{theorem}
We remark that in item \ref{it:distillation-deltaconst} of the above theorem,  the additional assumption that $\til Q$ lacks $\ep'$-fooling optimal actions is without loss of generality if we are purely interested in obtaining gap-based bounds (namely, in getting the bound of the second term of the minimum in (\ref{eq:fooling-regret})). In particular, if for some $x,a,h$, $\Delta_h(x,a) = 0$ and $\til Q_h(x,a) > V_h^\st(x)+ \ep$, then $(x,a,h) \in \fool{\ep'}{\ep'}$, meaning that one of the terms in the sum in (\ref{eq:fooling-regret}) will be $\frac{H^4 \iota}{[0 - \ep']_+}$, which is infinite.

\section{Algorithm overview}
\label{sec:alg-overview}
\subsection{Algorithm description}
\label{sec:alg-description}
Our algorithm, \algname (Algorithm \ref{alg:main}) used in Theorems \ref{thm:main-deltaincr} and \ref{thm:main-deltaconst}, is based loosely off of the $Q$-learning algorithm \cite{jin_is_2018}, and incorporates numerous additional aspects (including several ideas from \cite{xu_fine-grained_2021}) to effectively use the predictions $\til Q_h(x,a)$. In this section we describe the main ideas of the algorithm. At each episode $k$, the algorithm maintains upper and lower bounds on the $Q$-value and $V$-value functions, denoted $\Qo_h^k(x,a), \Vo_h^k(x,a)$ and $\Qu_h^k(x,a), \Vu_h^k(x,a)$, respectively. Unlike previous versions of $Q$-learning, our algorithm makes use of additional functions, denoted $\Ro_h^k(x,a), \til Q_h^k(x,a), \til V_h^k(x)$, which may be interpreted as follows:
\begin{itemize}
\item $\til Q_h^k(x,a)$ is a refinement of the predictions $\til Q_h(x,a)$ given to the algorithm as input; $\til Q_h^1$ is set to equal $\til Q_h$ (step \ref{it:init-params}),  and $\til Q_h^k$ is refined over time as the algorithm collects trajectories.
\item The values $\Ro_h^k(x,a)$ are used in the process of refining $\til Q_h^k(x,a)$; $\Ro_h^k(x,a)$ represents an approximate upper bound on $Q_h^\st(x,a)$, assuming that the prediction $\til Q$ is an $\ep$-approximate distillation (Definition \ref{def:distillation}).
\item $\til V_h^k(x)$ is an upper estimate for the $V$-value function at a state $x$ that makes use of the refined predictions $\til Q_h^k(x,a)$.
\end{itemize}
\algname additionally employs the technique of \emph{action elimination}, maintaining sets $A_h^k(x)$ (defined in step \ref{it:def-active-set}) which for each $x,h,k$ contain the actions $a$ which could plausibly be optimal at the beginning of episode $k$ ($A_h^1(x)$ is initialized to all of $\MA$ in step \ref{it:init-params}). Action elimination has previously been used in bandit learning and reinforcement learning when one must be robust to adversarial corruptions \cite{even-dar_action_2006,lykouris_stochastic_2018,lykouris_corruption_2020}, as well as to obtain gap-based regret bounds \cite{xu_fine-grained_2021}. In our algorithm, the sets $A_h^k(x)$ are used for \emph{both} of these purposes %
(where the robustness is with respect to the possible inaccuracy of the predictions $\til Q_h$). For convenience, we set $\MG_h^k$ to denote the set of states $x$ for which $|A_h^k(x)| = 1$ (meaning all but one action at $x$ has been eliminated at the beginning of episode $k$; see step \ref{it:def-ghk}).

After being initialized in step \ref{it:init-params} of \algname, the values $\Qo_h^k, \Qu_h^k, \Vo_h^k, \Vu_h^k, \til Q_h^k, \til V_h^k, \Ro_h^k$ are updated in \algname in steps \ref{it:update-wc} and \ref{it:update-preds} according to established updating procedures, namely using exploration bonuses of $b_n = C_0 \cdot \sqrt{H^3 \iota/n}$ (for some constant $C_0$) and a learning rate of $\alpha_n = \frac{H+1}{H+n}$, for $n \in \BN$  \cite{jin_is_2018,xu_fine-grained_2021}. In particular, $\Vo_h^k, \Vu_h^k, \Qo_h^k, \Qu_h^k$ are updated in step \ref{it:update-wc} according to the adaptive multi-step bootstrap technique of \cite{xu_fine-grained_2021}, which uses sequences of multiple rewards (namely, at contiguous sequences of states in which the optimal action has been determined) to perform the Bellman update. Our updates differ slightly from those in previous works in that we also maintain supplementary estimates $\qo_h^k, \qu_h^k$ (steps \ref{it:qo-update-gap} and \ref{it:qu-update-gap}) to ensure that $\Qo_h^k, \Vo_h^k$ are non-increasing with respect to $k$, and $\Qu_h^k, \Vu_h^k$ are non-decreasing with respect to $k$ (Lemma \ref{lem:qv-monotonicity}).

The purpose of maintaining $\Vo_h^k, \Vu_h^k, \Qo_h^k, \Qu_h^k$ is primarily to obtain the robustness regret bounds (i.e., of $\frac{1}{\lambda} \cdot \CC_{M,T,\lambda}$) in Theorems \ref{thm:main-deltaincr} and \ref{thm:main-deltaconst}. On the other hand, the values $\til Q_h^k, \til V_h^k, \Ro_h^k$, which are updated in step \ref{it:update-preds} of \algname, are used to obtain improved regret bounds in the presence of accurate predictions. The updates here only use a single step to perform the Bellman update, as in the standard $Q$-learning algorithm \cite{jin_is_2018}. 

For future reference we define the following learning rate parameters used in the algorithm's analysis: for $n \geq i \geq 1$, set
    \begin{align}
      \alpha_0^0 := 1, \qquad \alpha_n^0 := 0, \qquad \alpha_n^i := \alpha_i \prod_{j=i+1}^n (1-\alpha_j)\label{eq:alphas-all}.
    \end{align}
    Intuitively, $\alpha_n^i$ denotes the impact of an update made the $i$th time a state-action pair $(x,a,h)$ is visited on the value of any value function (e.g., $\Qo_h^k, \Qu_h^k$, etc.) when $(x,a,h)$ is visited for the $n$th time. 
    In the remainder of the section, we describe how \algname chooses its policies (step \ref{it:call-polsel}); the challenge of doing so leads to most of the algorithmic novelties in \algname. 
\paragraph{State-specific exploration \& exploitation phases} At each episode $k$, \algname chooses a policy $\pi^k$ by using the functions $\Qo_h^k, \Qu_h^k, \til Q_h^k$ in the \polsimple subroutine (Algorithm \ref{alg:pol-simple}).  
 A key challenge addressed in this step is that of obtaining a ``best of both worlds'' guarantee which improves upon the minimax regret guarantee of $\widetilde O(\sqrt{SATH^{O(1)}})$ (or, in the gap-based case, $\poly(H) \cdot \widetilde O \left(\sum_{(x,a,h)} \frac{1}{\Delta_h(x,a)} + \frac{|\Amul|}{\delmin} \right)$) in the case that the predictions $\til Q$ are accurate, but still manages to obtain sublinear regret when $\til Q$ is arbitrarily inaccurate. \algname overcomes this challenge by dividing the set of episodes in which we visit each state $x$ at each step $h$ into two phases: %
\begin{itemize}
\item In the first phase, we employ \emph{exploration}: whenever $(x,h)$ is visited during an episode $k$ in this phase, the policy $\pi^k$ takes an action $a \in A_h^k(x)$ which maximizes the gap between $\Qo_h^k(x,a)$ and $\Qu_h^k(x,a)$ (this approach is slightly different from the more standard UCB approach which chooses $a$ to maximize $\Qo_h^k(x,a)$ \cite{jin_is_2018}, but was used in \cite{xu_fine-grained_2021} to obtain gap-based bounds; it is used in \algname for the same reason).
\item After a certain number of episodes, \algname will decide it has sufficiently explored at the state $(x,h)$, and thus, when visiting $(x,h)$, it will choose an action $\hat a \in A_h^k(x)$ which maximizes the refined predictions $\til Q_h^k(x,\hat a)$\footnote{For technical reasons, $\hat a$ is actually chosen to maximize $\max \{ \til Q_h^k(x,\hat a), \Qu_h^k(x,\hat a) \}$}. This second phase may be seen as a \emph{constrained exploitation} phase: it attempts to exploit the predictions $\til Q_h$, but the action $\hat a$ is constrained to lie in the action set $A_h^k(x)$. As explained below, any action $a'$ at $x$ which is very suboptimal will be removed from $A_h^k(x)$ after a bounded number of episodes, which limits the impact of inaccurate predictions. %
\end{itemize}
We emphasize that the partition into the two phases is \emph{state-specific}; namely, at any given episode, some states may be in their exploration phase whereas others may be in their exploitation phase. Notice that there is a tradeoff between the lengths of the two phases: if the first phase, which does not make use of the predictions $\til Q_h$ and thus cannot outperform the minimax bounds, is too long, then if the predictions $\til Q$ are accurate we will not improve sufficiently upon the minimax regret guarantee. On the other hand, if the first phase is too short (or nonexistent), the following may occur: suppose that the predictions $\til Q$ are inaccurate in that for some state $x$, step $h$, and sub-optimal action $a$, $\til Q_h(x,a)$ is large, but $\til Q_h(x,a^\st)$ is small, where $a^\st \neq a$ is the unique optimal action at $(x,h)$ and satisfies $V_h^\st(x) = Q_h^\st(x, a^\st) \gg Q_h^\st(x,a)$. Suppose for simplicity that $\MA = \{a, a^\st \}$ and that $\til Q_h \equiv \til Q_h^k$ (which can approximately hold). Ideally the first phase should be long enough to eliminate $a$ from $A_h^k(x)$; this will happen when $\Qu_h^k(x,a^\st)$ grows sufficiently to be greater than $\Qo_h^k(x, a)$. However, if the first phase ends before this happens, then at the beginning of the second phase, $A_h^k(x) = \MA$, and so $\pi_h^k(x)$ will be set to $a$ in step \ref{it:choose-policy} of \polsimple. Thus \algname would suffer linear regret. %

\paragraph{An adaptive exploration-exploitation cutoff} 
\algname trades off the lengths of the exploration and exploitation phases described above according to the input parameter $\lambda$ (or, in the case of Theorem \ref{thm:main-deltaconst}, $\wh \lambda$ as determined by $\CR$). %
To describe how \algname makes this tradeoff, we begin by defining the \emph{$Q$- and $V$-range functions} (following the presentation of \cite{xu_fine-grained_2021}). First, we make a few additional definitions: for $(k,h) \in [K] \times [H]$, for which $x_h^k \not \in \MG_h^k$, set $h'(k,h) \in [H+1]$ to be the first step $h'$ after step $h$ for which $x_{h'}^k \not \in \MG_h^k$ (if such $h' \leq H$ does not exist, then set $h'(k,h) = H+1$). Next, for $n \in \BN$, define the following parameters $\beta_n$, which may be viewed as aggregated versions of the exploration bonuses $b_i = C_0 \sqrt{H^3 \iota / i}$ (recall the definition of $\alpha_n^i$ in (\ref{eq:alphas-all})):
\begin{align}
\beta_0 := 0 ,\qquad \beta_n = 2 \sum_{i=1}^n \alpha_n^i \cdot b_i\label{eq:define-beta}.
\end{align}
\begin{defn}[Range function]
  \label{def:range-function}
  For $(x,a,h,k) \in \MS \times \MA \times [H] \times [K]$ for which $x \not \in \MG_h^k$ and $a \in A_h^k(x)$, define the \emph{range $Q$-function} as follows: set $\RanQ_h^0(x,a) = H$, and 
  \begin{align}
    \RanQ_h^k(x,a) :=&\min \left\{ \RanQ_h^{k-1}(x,a),\   \alpha_n^0 H + \beta_n + \sum_{i=1}^n \alpha_n^i \cdot \RanV_{h'(k_h^i,h)}^{k_h^i}(x_{h'(k_h^i,h)}^{k_h^i})  \right\}\nonumber\\
    & \quad \mbox{where} \quad n = N_h^k(x,a), \ k_h^i = k_h^i(x,a) \ \forall i \in [n]\nonumber.
  \end{align}
  Moreover, for $(x,h,k) \in \MS \times [H] \times [K]$ for which $x \not \in \MG_h^k$, define the \emph{range $V$-function} as follows: set $\RanV_h^0(x) = H$, and 
  \begin{align}
\RanV_h^k(x) := \min \{ \RanV_h^{k-1}(x),\ \RanQ_h^k(x, a^\st)\} \quad \mbox{ for } \quad a^\st = \argmax_{a' \in A_h^k(x)} \Qo_h^k(x, a') - \Qu_h^k(x,a') \nonumber.
  \end{align}
  Finally, define $\RanQ_{H+1}^k(x,a) = \RanV_{H+1}^k(x) = 0$ for all $x,a,k$. 
\end{defn}
The functions $\RanQ_h^k, \RanV_h^k$ should be interpreted as upper bounds on the gap between the upper and lower $Q, V$ values; note that they satisfy a similar recursion to $\Qo_h^k - \Qu_h^k$ and $\Vo_h^k - \Vu_h^k$ (see (\ref{eq:qo-qu-gap})). Indeed, in Lemma \ref{lem:range-bound} below we show that $\RanQ_h^k, \RanV_h^k$ are upper bounds on $\Qo_h^k - \Qu_h^k, \Vo_h^k - \Vu_h^k$, respectively. 

Now that we have defined the range functions, the choice of policy at each $(x,h)$ (equivalently, the choice of ``exploration'' and ``constrained exploitation'' phases described above) is simple to state: \algname maintains a parameter $\wh \Delta^k$ at each episode $k$, which represents a ``target error bound'' that \algname hopes to obtain. The parameter $\wh \Delta^k$ adapts to the input parameter $\CR$ (or $\lambda$) as well as the gap-based complexity of the given MDP. Given $\wh \Delta^k$ at episode $k$, the policy $\pi_h^k$ at each step $h$ is specified in (\ref{eq:explore-exploit}) in the algorithm \polsimple: following our terminology above, a state $(x,h)$ is declared to be in the ``exploration'' phase if $\clipV_h^k(x) > \ggapfn{h}{\wh \Delta^k}$\footnote{Recall that for $a > 0$ and some constant $C_1$, we have defined $\ggapfn{h}{a} = C_1 \cdot \left(1 + \frac 1H \right)^{4(H+1-h)} \cdot a$ in \algname. Thus $\ggapfn{h}{\wh \Delta^k} = \Theta(\wh \Delta^k)$; the function  $\ggapfn{h}{\cdot}$ is introduced for technical considerations in the proof.}  (meaning there is still much uncertainy about the optimal value at $(x,h)$ relative to $\wh \Delta^k$), and is defined to be in the ``constrained exploitation'' phase otherwise (i.e., $\clipV_h^k(x) \leq \ggapfn{h}{\wh \Delta^k}$). We will show (in Lemmas \ref{lem:clip-monotonicity} and \ref{lem:misc-monotonicity}) that $\clipV_h^k(x)$ is nonincreasing with respect to $k$ and $\ggapfn{h}{\wh \Delta^k}$ is nondecreasing with respect to $k$; thus, each state can only move from the ``exploration'' to ``constrained exploitation'' phase.

\subsection{How to choose $\wh \Delta^k$}
As we discussed in the previous section, $\wh \Delta^k$ is chosen to adapt to the input parameter $\CR$ or $\lambda$. In the setting of Theorem \ref{thm:main-deltaconst}, where the user inputs a parameter $\CR$ representing the target worst-case regret bound, the choice of $\wh \Delta^k$ is extremely simple (Algorithm \ref{alg:delta-const}, \deltaconst): for all $k$, we set $\wh \Delta^k := \CR / (KH)$. In the setting of Theorem \ref{thm:main-deltaincr}, where the user inputs a parameter $\lambda$ specifying a trade-off between the worst-case and ideal-case settings, $\wh \Delta^k$ is set (in Algorithm \ref{alg:delta-incr}, \deltaincr) to a more complex expression which is a surrogate for the worst-case regret bound $\frac{1}{\lambda} \CC_{M, T,\lambda}$ (thus overcoming the challenge that the algorithm does not know $\CC_{M,T,\lambda}$). This surrogate uses the \emph{frozen range function} (defined in Definition \ref{def:frozen}), denoted $\frzQ_h^k(x,a)$, as a proxy for the action-value gaps $\Delta_h(x,a)$, for all $(x,a,h) \in \MS \times \MA \times [H]$. We refer the reader to Section \ref{sec:lambda-overview} for further details.

\begin{algorithm}[!htp]
  \caption{\bf \algname}\label{alg:main}
  \KwIn{State space $\MS$, action space $\MA$, horizon $H$, number of episodes $K$, predictions $\til Q_h : \MS \times \MA \ra [0,H]$ for all $h \in [H]$, parameter $\lambda \in [0,1]$.
    For some constant $C_1 > 0$ and $1 \leq h \leq H+1$, set, for $a > 0$, $\ggapfn{h}{a} = C_1 \cdot \left(1 + \frac 1H \right)^{4(H+1-h)} \cdot a$.} %
  \begin{enumerate}[leftmargin=14pt,rightmargin=20pt,itemsep=1pt,topsep=1.5pt]
\item \label{it:init-params} Initialize $N_h^1(x,a) = 0$, $\Ro_h^1(x,a) = \Qo_h^1(x,a) = \Vo_h^1(x,a) = \qo_h^1(x,a)= H$, $\Qu_h^1(x,a) = \Vu_h^1(x,a) = \qu_h^1(x,a)= 0$ for all $(x,a,h) \in \MS \times \MA \times [H]$. Also set $\til Q_h^1 = \til Q_h$ and $\til V_h^1(x) = \max_{a' \in \MA} \til Q_h^1(x,a')$ for all $(x,h) \in \MS \times [H]$. 
  Set $A_h^1(x) = \MA$ for all $(x,h) \in \MS \times [H]$, and $\MG_h^1 = \emptyset$ for all $h \in [H]$. %
  Set $\wh \Delta^1 \gets 0$. %
\item For episode $1 \leq k \leq K$:
  \begin{enumerate}
  \item  Receive $\pi^k$ and the policy rollout $(x_1^k, a_1^k), \ldots, (x_H^k, a_H^k)$ %
    from the \polsimple algorithm.\label{it:call-polsel}
  \item \label{it:update-wc} For each $h = 1, 2, \ldots, H$ such that $x_h^k \not \in \MG_h^k$: %
    \begin{enumerate}
      \item Set $N_h^{k+1} \gets N_h^k(x_h^k, a_h^k) + 1$, $n = N_h^{k+1}(x_h^k, a_h^k)$, and write $b_n = C_0 \sqrt{H^3 \iota/n}$.\label{it:nhk}
      \item Let $x_{h'}^k$ be the first state in the episode after $x_h^k$ so that $x_{h'}^k \not\in \MG_{h'}^k$ (if such $h'$ does not exist, set $h' = H+1$).
      \item Let $\wh r_h^k = \sum_{h'' = h}^{h'-1} r_{h''}(x_{h''}^k, a_{h''}^k)$.
      \item Set $\qo_h^{k+1}(x_h^k, a_h^k) \gets \min (1 - \alpha_n) \cdot \qo_h^k(x_h^k, a_h^k) + \alpha_n \cdot (\wh r_h^k + \Vo_{h'}^k(x_{h'}^k) + b_n)$.\label{it:qo-update-gap}
      \item Set $\Qo_h^{k+1}(x_h^k, a_h^k) \gets \min_{k' \leq k+1} \left\{\qo_h^{k'}(x_h^k, a_h^k) \right\}$. \label{it:Qo-qo-min}
      \item Set $\qu_h^{k+1}(x_h^k, a_h^k) \gets (1 - \alpha_n) \cdot \qo_h^k(x_h^k, a_h^k) + \alpha_n \cdot (\wh r_h^k + \Vu_{h'}^k(x_{h'}^k) - b_n)$.\label{it:qu-update-gap} %
      \item Set $\Qu_h^{k+1}(x_h^k, a_h^k) \gets \max_{k' \leq k+1} \left\{ \qu_h^{k'}(x_h^k, a_h^k) \right\}$. \label{it:Qu-qu-max}
      \item Set $\Vu_h^{k+1}(x_h^k) \gets \max_{a' \in A_h^k(x_h^k)} \{\Qu_h^{k+1}(x_h^k, a')\}$. \label{it:vu-define}%
      \item Set $\Vo_h^{k+1}(x_h) \gets \max_{a' \in A_h^k(x_h)} \{\Qo_h^{k+1}(x_h, a')\}$. \label{it:vo-define} %
      \end{enumerate}
    \item For all $(x,h) \in \MS \times [H]$, set $A_h^{k+1}(x) \gets \{ a' \in A_h^k(x) : \Qo_h^{k+1}(x,a') \geq \Vu_h^{k+1}(x) \}$.\label{it:def-active-set}
      
      \item For each $h = 1, 2, \ldots, H$:\label{it:update-preds}
        \begin{enumerate}
      \item Set $\Ro_h^{k+1}(x_h^k, a_h^k) \gets (1 - \alpha_{n}) \cdot \Ro_h^k(x_h^k, a_h^k) + \alpha_{n} \cdot (r_h(x_h^k,a_h^k) + \til V_{h+1}^k(x_{h+1}) + b_{n})$. \label{it:def-rohk}%
      \item %
        Set $\til Q_h^{k+1}(x_h^k,a_h^k) \gets \min \{\Ro_h^{k+1}(x_h^k,a_h^k), \til Q_h^k(x_h^k,a_h^k), \Qo_h^{k+1}(x_h^k, a_h^k) \}$.\label{it:def-tilq-hk}
      \item Set $\til V_h^{k+1}(x_h^k) \gets %
        \max_{a' \in A_h^{k+1}(x_h^k)} \max\{\til Q_h^{k+1}(x_h^k, a'), \Qu_h^{k+1}(x_h^k, a') \}$. %
      \end{enumerate}
    \item For all $h$ and  all $(x,a) \neq (x_h^k, a_h^k)$ set $N_h^{k+1}(x,a), \Qo_h^{k+1}(x,a), \Qu_h^{k+1}(x,a), \qo_h^{k+1}(x,a),\newline  \qu_h^{k+1}(x,a), \Vo_h^{k+1}(x), \Vu_h^{k+1}(x), \Ro_h^{k+1}(x,a), \til Q_h^{k+1}(x,a), \til V_h^{k+1}(x)$ equal to their values at episode $k$.  \label{it:copy-qv}
    \item For all $h \in [H]$, set $\MG_h^{k+1} \gets \{ x \in \MS : | A_h^{k+1}(x)| = 1 \}$.\label{it:def-ghk}

    \item Choose $\wh \Delta^{k+1}$ according to either \deltaconst or \deltaincr. %
  \end{enumerate}
  \end{enumerate}
\end{algorithm}

\begin{algorithm}
  \caption{\bf \polsimple}\label{alg:pol-simple}
  \KwIn{Internal state of the algorithm \algname at the beginning of episode $k$ (including, in particular, the previous pollicy rollouts, and the functions $\til Q_h^k, \Qu_h^k, \Qo_h^k$, as well as $\wh \Delta^k, \MG_h^k, A_h^k$).} %
  \begin{enumerate}[leftmargin=14pt,rightmargin=20pt,itemsep=1pt,topsep=1.5pt]
    \item For $h \in [H]$, construct $\RanV_h^k(\cdot)$ per Definition \ref{def:range-function}.
      \item Define the policy $\pi^k$ by, for $(x,h) \in \MS \times [H]$:\label{it:choose-policy}
    \begin{align}
      \pi^k_h(x) := \begin{cases}
        \mbox{The action in $A_h^k(x)$} \quad & \mbox{ if } |A_h^k(x)| = 1 \\
        \argmax_{a \in A_h^k(x)} \{ \max\{\til Q_h^k(x,a), \Qu_h^k(x,a) \} \} \quad & \mbox{ if } \RanV_h^k(x) \leq  \ggapfn{h}{\widehat \Delta^k}\\
        \argmax_{a \in A_h^k(x)} \{ \Qo_h^k(x,a) - \Qu_h^k(x,a) \} \quad & \mbox{ if } \RanV_h^k(x) >  \ggapfn{h}{\widehat \Delta^k}  %
      \end{cases}\label{eq:explore-exploit}
    \end{align}
  \item Let $(x_1^k, a_1^k), \ldots, (x_H^k, a_H^k)$ be a policy rollout obtained by following $\pi^k$.

  \item Return the policy $\pi^k$ and the policy rollout $(x_1^k, a_1^k), \ldots, (x_H^k, a_H^k)$. %
    \end{enumerate}
  \end{algorithm}

    \begin{algorithm}
    \caption{\bf \deltaconst}\label{alg:delta-const}
    \KwIn{Episode number $k$, input regret bound $\CR$ of \algname, and total number $K$ of episodes.
      }
\begin{enumerate}[leftmargin=14pt,rightmargin=20pt,itemsep=1pt,topsep=1.5pt]
\item Return
        \begin{align}
\wh \Delta^{k+1} := \CR / (KH) \label{eq:define-whdelta-const}.
      \end{align}
\end{enumerate}
  \end{algorithm}

  \begin{algorithm}
    \caption{\bf \deltaincr}\label{alg:delta-incr}
    \KwIn{Internal state of the algorithm $\algname$ at the beginning of episode $k+1$ (in particular, the necessary information to compute the frozen $Q$-range function), and parameters $\lambda \in [0,1]$ and $\tildm \geq 0$ (which is guaranteed to satisfy $\tildm \leq \delmin$). 
      }
\begin{enumerate}[leftmargin=14pt,rightmargin=20pt,itemsep=1pt,topsep=1.5pt]
\item \label{it:tfrzq-define} For $h \in [H]$, construct $\tfrzQ_h^{k+1}(\cdot)$ which is defined identically to $\frzQ_h^{k+1}(\cdot)$ per Definition \ref{def:frozen}, except with the parameter $\tildm$ replacing $\delmin$ in the clipped value functions $\clipV_h^{k+1}, \clipQ_h^{k+1}$. %
\item Return
        \begin{align}
\wh \Delta^{k+1} := \min \left\{\frac{H^5 \iota^2}{\lambda \cdot K} \cdot \sum_{(x,a,h)} \frac{1}{\max\left\{ \frac{\tfrzQ_h^{k+1}(x,a)}{2H}, \frac{\tildm}{4H^2} \right\} }, \sqrt{\frac{SAH^8\iota^2}{\lambda \cdot K}} \right\} \label{eq:define-whdelta-incr}.
      \end{align}
\end{enumerate}
  \end{algorithm}

  \section{Proof overview}
  \label{sec:proofs-overview}
  In this section we overview the proofs of Theorems \ref{thm:main-deltaincr} and \ref{thm:main-deltaconst}; we focus mainly on Theorem \ref{thm:main-deltaconst} since its proof is slightly simpler. At a high level, the key tools needed in the proof of Theorem \ref{thm:main-deltaconst} are as follows:
  \begin{enumerate}
  \item First, we need to define the the \emph{clipped range functions} (Definition \ref{def:half-clipped}), as in \cite{xu_fine-grained_2021}, which aid in proving gap-based bounds. 
  \item To prove the $O(\CR)$ regret bound for worst-case predictions (i.e., robustness, first item of Theorem \ref{thm:main-deltaconst}), we first prove a regret decomposition (Lemma \ref{lem:regret-decomposition}) showing that regret can be bounded in terms of the clipped $V$-range functions.
    
    In Lemma \ref{lem:w-bound-gap}, our main technical lemma for the worst-case regret bound, we then show how to bound the clipped $V$-range functions in the presence of arbitrary predictions $\til Q$ using certain \emph{monotonicity} properties of the value functions. %
  \item To establish the improved regret bounds for the case that $\til Q$ is an approximate distillation (second item of Theorem \ref{thm:main-deltaconst}), we first need to bound the number of episodes during which the predictions $\til Q$ are \emph{not} used to choose the policy. %

  Then we upper bound the value functions $\Ro_h^k, \til V_h^k$ at the set of episodes $k$ where the predictions \emph{are} used and show that doing so suffices to bound regret.
  \end{enumerate}

  The proof of Theorem \ref{thm:main-deltaincr} is similar to that of Theorem \ref{thm:main-deltaconst}. One additional tool needed (which shows up in the algorithm \deltaincr) is a variation of the clipped range functions that we call the \emph{frozen range functions} (Definition \ref{def:frozen}). %

In Sections \ref{sec:clipped-overview} through \ref{sec:distillation-overview} we expand upon the above items to overview the proof of Theorem \ref{thm:main-deltaconst}. In Section \ref{sec:lambda-overview} we overview the changes that must be made to \algname and the proof to establish Theorem \ref{thm:main-deltaincr}.
  \subsection{Clipped range functions}
  \label{sec:clipped-overview}
We begin by defining the \emph{clipped $Q$-value and $V$-value functions}, which were originally introduced in \cite{xu_fine-grained_2021} to obtain gap-based bounds on the regret (they play a similar role in this paper). %
For real numbers $x,y$, define the \emph{clip function} as follows: $\clip{x}{y} := \One[x \geq y] \cdot x$.
\begin{defn}[Clipped range function, \cite{xu_fine-grained_2021}]
  \label{def:half-clipped}
  For all $(x,a,h,k) \in \MS \times \MA \times [H] \times [K]$ for which $x \not \in \MG_h^k$ and $a \in A_h^k(x)$, define the \emph{clipped range $Q$-function} as follows: set $\clipQ_h^0(x,a) = H$, and
  \begin{align}
    \clipQ_h^k(x,a) := & \min \left\{ \clipQ_h^{k-1}(x,a),\ \alpha_n^0 H + \clip{\beta_n}{\frac{\delmin}{4H^2}} + \sum_{i=1}^n \alpha_n^i \cdot \clipV_{h'(k_h^i, h)}^{k_h^i}(x_{h'(k_h^i,h)}^{k_h^i}) \right\} \label{eq:clipq-define}\\
    & \quad \mbox{where}
     \quad n = N_h^k(x,a),\ k_h^i = k_h^i(x,a) \ \forall i \in [n]\nonumber.
  \end{align}
  Moreover, for $(x,h,k) \in \MS \times [H] \times [K]$ for which $x \not \in \MG_h^k$, define the \emph{clipped range $V$-function} as follows: set $\clipV_h^0(x) = H$, and
  \begin{align}
\clipV_h^k(x) := \min \{ \clipV_h^{k-1}(x),\ \clipQ_h^k(x, a^\st)\} \quad \mbox{for} \quad a^\st = \argmax_{a' \in A_h^k(x)} \Qo_h^k(x,a') - \Qu_h^k(x,a')\nonumber.
  \end{align}
  Finally, define $\clipV_{H+1}^k(x) = \clipQ_{H+1}^k(x,a) = 0$ for all $x,a,k$.
\end{defn}
The clipped range functions $\clipQ_h^k(x,a), \clipV_h^k(x)$ are defined to satisfy a similar recursion as the quantities $\Qo_h^k(x,a) - \Qu_h^k(x,a)$ and $\Vo_h^k(x) - \Vu_h^k(x)$ (see (\ref{eq:qo-qu-gap})). Unlike in (\ref{eq:qo-qu-gap}), in the definition of $\clipQ_h^k, \clipV_h^k$, the bonuses $\beta_n$ are clipped, leading $\clipV_h^k, \clipQ_h^k$ to be smaller than their unclipped counterparts (Lemma \ref{lem:clip-lt-range}), which aids in obtaining gap-based regret bounds. Despite this clipping, the combination of Lemmas \ref{lem:range-bound} and \ref{lem:hc-range-bound} shows that, with high probability, for all $x,a,h,k$, the clipped range functions are still approximately lower bounded by the gap between the upper and lower $Q,V$-values, as follows:
\begin{align}
\clipQ_h^k(x,a) \geq \Qo_h^k(x,a) - \Qu_h^k(x,a) - \frac{\delmin}{4H}, \qquad \clipV_h^k(x) \geq \Vo_h^k(x) - \Vu_h^k(x) - \frac{\delmin}{4H}\label{eq:range-bounds-informal}.
\end{align}

\subsection{Worst-case regret bound}
\label{sec:worstcase-overview}
In this section we overview the proof that \algname achieves regret $O(\CR)$ for arbitrary predictions $\til Q$ in the setting of Theorem \ref{thm:main-deltaconst}. %
For all $(h,k)$ for which $x_h^k \not \in \MG_h^k$, define
$ 
\clipdelta_h^k = \clipV_h^k(x_h^k).
$
Using (\ref{eq:range-bounds-informal}), the following regret decomposition is straightforward to prove (it is similar to that in Lemma B.6 of \cite{xu_fine-grained_2021}).
\begin{replemma}{lem:regret-decomposition}[Regret decomposition; abbreviated]
There is an event $\MEwc$ that occurs with probability at least $1-1/(H^2 K)$ so that the regret of \algname may be bounded as follows:
    \begin{align}
\sum_{k=1}^K \E \left[ V_1^\st(x_1^k) - V_1^{\pi^k}(x_1^k)\right] \leq 1 + 4 \cdot \E \left[ \left. \sum_{(k,h) : a_h^k \not \in \Aopt_{h,0}(x_h^k)} \clipdelta_h^k \right| \MEwc \right] \nonumber.
    \end{align}
    (The right-hand side of the above expression makes sense since under the event $\MEwc$, it turns out that for all $(k,h)$ so that $a_h^k \not \in \Aopt_{h,0}(x_h^k)$, $x_h^k \not \in \MG_h^k$, i.e., $\clipdelta_h^k$ is well-defined.)
  \end{replemma}

  Lemma \ref{lem:regret-decomposition} reduces the problem of bounding the regret to bounding the clipped value functions $\clipdelta_h^k$ for $h,k$ such that $x_h^k \not \in \MG_h^k$. In turn, we bound $\clipdelta_h^k$ in Lemma \ref{lem:w-bound-gap}, of which a simplified version combining it with Lemma \ref{lem:frzq-gap} is presented below:
  \begin{replemma}{lem:w-bound-gap}[Abbreviated \& combined with Lemma \ref{lem:frzq-gap}]
    Fix any $h \in [H]$, any set $\MW \subset [K]$ so that for all $k \in \MW$, $x_h^k \not \in \MG_h^k$, and any $k^\st \geq \max_{k \in \MW} \{ k \}$. Then
    \begin{align}
\sum_{k \in \MW} \clipdelta_h^k \leq & | \MW| \cdot \ggapfn{h}{\wh \Delta^{k^\st}} + O \left( \min \left\{ \sqrt{H^5 SA |\MW| \iota}, \ H^5 \iota \cdot \left( \sum_{(x,a,h')} \frac{1}{\Delta_{h'}(x,a)} + \frac{|\Amul|}{\delmin} \right)\right\} \right)\label{eq:frzq-gap-informal}.
    \end{align}
  \end{replemma}
  The starting point for the proof of Lemma \ref{lem:w-bound-gap} is to use the definition of the clipped value functions $\clipV_h^k$ together with reverse induction on $h$ (i.e., bounding the values $\clipdelta_h^k$ in terms of $\clipdelta_{h'}^k$ for $h' > h$) in a similar manner as was done in Lemma B.8 of \cite{xu_fine-grained_2021}.
  However, the proof of Lemma \ref{lem:w-bound-gap} must depart from that of \cite[Lemma B.8]{xu_fine-grained_2021} since in the \polsimple subroutine of \algname, we do not always choose the action $a \in A_h^k(x)$ maximizing the \emph{confidence interval}, i.e., maximizing $\Qo_h^k(x,a) - \Qu_h^k(x,a)$. Typically such an action choice maximizing the confidence interval is necessary to upper bound the values $\clipdelta_h^k$. %
  We are able to nevertheless bound $\clipdelta_h^k$ using the fact that for steps $(h,k)$ for which we do \emph{not} choose the action $a$ maximizing the confidence interval (i.e., we are in the constrained exploitation phase), it must hold that $\clipdelta_h^k = \clipV_h^k(x_h^k) \leq \RanV_h^k(x_h^k) \leq \ggapfn{h}{\wh \Delta^k}$. This observation leads to the quantity $|\MW| \cdot \ggapfn{h}{\wh \Delta^{k^\st}}$ on the right-hand side of (\ref{eq:frzq-gap-informal}).

  The proof of Lemma \ref{lem:w-bound-gap} is made somewhat more complex by the fact that the choice of action at step $h$ affects $\clipdelta_{h'}^k$ for $h' < h$ and various $k$ (via the reverse induction argument), and without care one will end up with a multiplier of $\ggapfn{h}{\wh \Delta^{k^\st}}$ in (\ref{eq:frzq-gap-informal}) that is much larger than $|\MW|$. 
  To avoid this complication, we must carefully account for the effect the values $\clipdelta_h^k$ have on the bounds we can prove on $\clipdelta_{h'}^k$ for $h' < h$. To do so we make use of a monotonicity propery of the clipped value functions (Lemma \ref{lem:clip-monotonicity}, showing that $\clipV_h^k(x)$ is non-increasing with $k$) and introduce the notion of \emph{level-$h$} sets (Definition \ref{def:levelh}) which are intermediate sets $\MW'$ of tuples $(k',h')$ for which we need to bound $\sum_{(k',h') \in \MW'} \clipdelta_{h'}^{k'}$ in the course of the induction.

  \paragraph{Completing the worst-case regret bound}
  The proof of item \ref{it:worstcase-deltaconst} of Theorem \ref{thm:main-deltaconst} is fairly straightforward given the above components; the details are worked out in Lemma \ref{lem:worst-case-const}. The dominant term in the bound (\ref{eq:frzq-gap-informal}) turns out to be $| \MW| \cdot \ggapfn{h}{\wh \Delta^{k^\st}}$, which due to the choice $\wh \Delta^k = \CR/(HK)$ and the bound $|\MW| \leq K$, leads to the bound $O(\CR)$ on regret. 

  \subsection{Exploration-constrained exploitation cutoffs}
  \label{sec:cutoffs-overview}
  Before discussing the proof of the improved regret bound for the case of $\til Q$ being an approximate distillation, we introduce the following notation relating to the exploration and constrained exploitation phases in \algname that we discussed above. %
  For $(k,h) \in [K] \times [H]$, define $\tau_h^k \in \{0,1\}$ as follows:
  \begin{align}
    \label{eq:define-tau}
    \tau_h^k = \begin{cases}
      0 \quad & \mbox{ if } x_h^k \in \MG_h^k \mbox{ or } \RanV_h^k(x_h^k) \leq \ggapfn{h}{\wh \Delta^k} \\
      1 \quad & \mbox{ otherwise}.
    \end{cases}
  \end{align}
  The parameter $\tau_h^k$ is the indicator of whether \algname is in the exploration or constrained exploitation step at step $h$ of episode $k$: if $\tau_h^k = 0$, then we have either determined the optimal action at $x_h^k$ (i.e., $x_h^k \in \MG_h^k$), or else the range function $\RanV_h^k(x_h^k)$ is sufficiently small, so we engage in constrained exploitation (see the choice of $\pi_h^k$ in (\ref{eq:explore-exploit}), which chooses $ a' \in A_h^k(x_h^k)$ maximizing $\max \{ \til Q_h^k(x_h^k, a'), \Qu_h^k(x_h^k, a') \}$), and otherwise, if $\tau_h^k =1$, we use optimistic exploration, choosing $a' \in A_h^k(x_h^k)$ to maximize the confidence interval.

  Note that the parameters $\tau_h^k$ depend on the \emph{unclipped} range functions $\RanV_h^k$; as we have discussed above, in order obtain our gap-based bounds, it is necessary to bound the \emph{clipped} range functions $\clipV_h^k$. Therefore, when reasoning about the exploration and constrained exploitation phases, we will additionally introduce the parameters $\sigma_h^k \in \{0,1\}$ (for $(k,h) \in [K] \times [H]$), which are defined similarly to $\tau_h^k$ except with respect to $\clipV_h^k$:
  \begin{align}
    \label{eq:define-sigma}
    \sigma_h^k = \begin{cases}
      0 \quad & \mbox{ if } x_h^k \in \MG_h^k \mbox{ or } \clipV_h^k(x_h^k) \leq \frac{1}{1+\frac 1H} \cdot \ggapfn{h}{\wh \Delta^k} \\
      1 \quad & \mbox{ otherwise}.
    \end{cases}
  \end{align}
  The parameters $\sigma_h^k$ can be thought of as a proxy for the true exploration parameters $\tau_h^k$. As discussed in the following section, in order to establish improved regret bounds for the case that $\til Q$ is an approximate distillation, we need to, loosely speaking, upper bound the number of episodes in which we engage in exploration (i.e., in which the predictions $\til Q$ are \emph{not} used). For technical reasons, it turns out to be more convenient to bound the number of $(k,h)$ so that $\sigma_h^k = 1$ (as opposed to bounding the number of $(k,h)$ so that $\tau_h^k = 1$).

\subsection{Proofs for $\til Q$ an approximate distillation}
\label{sec:distillation-overview}
Now we discuss the proof of item \ref{it:distillation-deltaconst} of Theorem \ref{thm:main-deltaconst}; the proof of item \ref{it:distillation-deltaincr} of Theorem \ref{thm:main-deltaincr} is very similar (see Section \ref{sec:proofs-distillation} for the full proof). %
As discussed in the previous section, the first step is to bound the number of episodes for which we do \emph{not} engage in constrained exploitation; in particular, for each $h$, we bound the number of $k$ for which $\sigma_h^k = 1$: %
\begin{replemma}{lem:bound-sigma}
Suppose \algname is run with \deltaconst to choose the values $\wh \Delta^k$. Then for all $h \in [H]$, the number of episodes $k \in [K]$ for which $\sigma_h^k = 1$ is at most $\max\{ SAH^3, \wh \lambda \cdot K\}$. (Recall that $\wh \lambda$ is chosen so that $\CR = \frac{1}{\wh \lambda} \cdot \CC_{M,T,\wh \lambda}$.)
\end{replemma}
We write $\wh \Delta = \wh \Delta^k$ (as all $\wh \Delta^k$ are equal). 
Also, for any $h \in [H]$, write $\MY_h := \{ k : \sigma_h^k = 1 \}$. The main tool in the proof of Lemma \ref{lem:bound-sigma} is Lemma \ref{lem:w-bound-gap}, which upper bounds $\sum_{k \in \MY_h} \clipdelta_h^k$ by the sum of $|\MY_h| \cdot (1+1/H)^2 \cdot \ggapfn{h+1}{\wh \Delta}$ and some additional terms. On the other hand, that $\sigma_h^k = 1$ implies that $\clipdelta_h^k \geq \frac{1}{1+1/H} \cdot \ggapfn{h}{\wh \Delta}$. These facts (together with the fact that $\ggapfn{h}{\wh \Delta}$ is greater than $\ggapfn{h+1}{\wh \Delta}$ by a factor of $(1+1/H)^4$) allow us to upper bound $|\MY_h|$ in terms of an expression which ultimately simplifies to $\max \{ SAH^3, \wh \lambda \cdot K \}$. 

\paragraph{Regret decomposition and induction}
Given Lemma \ref{lem:bound-sigma}, we proceed to complete the proof of item \ref{it:distillation-deltaconst} of Theorem \ref{thm:main-deltaconst}. The first step is the following regret decomposition (stated in (\ref{eq:2tilv-ub})), which follows from the fact that $\til Q$ is an $\ep$-approximate distillation as well as the definition of $\til V_h^k$ in \algname: for any $\ep' > 0$, we have
\begin{align}
  & \sum_{k=1}^K \E\left[V_1^\st(x_1^k) - V_1^{\pi^k}(x_1^k)\right] \nonumber\\
  \leq & O(KH(\ep H + \ep')) + \E \left[\left.  \sum_{k=1}^K \sum_{h=1}^H  (1 - \tau_h^k) \cdot \One[a_h^k \not \in \Aopt_{h,\ep'}(x_h^k)] \cdot (\Ro_h^k(x_h^k, a_h^k) - Q_h^\st(x_h^k,a_h^k)) + \sum_{k=1}^K \sum_{h=1}^H 4 \sigma_h^k \clipdelta_h^k \right| \MEwc \right]\label{eq:dis-reg-informal}.
\end{align}
The above regret decomposition reduces bounding the regret to bounding the following two types of quantities (under the event $\MEwc$):
\begin{enumerate}
\item The quantity $\Ro_h^k(x_h^k, a_h^k) - Q_h^\st(x_h^k, a_h^k)$, for $(k,h)$ satisfying $\tau_h^k = 0$ and $a_h^k \not \in \Aopt_{h,\ep'}(x_h^k)$;\label{it:ro-qh-tau}
\item The quantity $\clipdelta_h^k$ for $(k,h)$ satisfying $\sigma_h^k = 1$.\footnote{Note that $\sigma_h^k = 1$ implies that $x_h^k \not \in \MG_h^k$, which implies that $\clipdelta_h^k$ is indeed well-defined.} \label{it:clipdelta-sigma}
\end{enumerate}
The latter of these quantities (i.e., item \ref{it:clipdelta-sigma}) is straightforward to control: for each $h \in [H]$, we use Lemma \ref{lem:w-bound-gap} with the set $\MW$ equal to the set of $k$ so that $\sigma_h^k = 1$ and $k^\st = K$. Crucially, the conclusion of Lemma \ref{lem:bound-sigma} above gives that $|\MW| \leq \max \{ SAH^3, \wh \lambda \cdot K \}$, which, together with the inequality (\ref{eq:frzq-gap-informal}) of Lemma \ref{lem:w-bound-gap}, gives us that $\sum_{k=1}^K \sum_{h=1}^H \sigma_h^k \clipdelta_h^k$ may be bounded by $O(\CC_{M,T,\wh \lambda})$. This argument is carried out formally in Lemma \ref{lem:sigma-delta-const}.

\paragraph{Bounding $\Ro_h^k, \til V_h^k$ on non-exploratory episodes}
We next describe how the quantity in item \ref{it:ro-qh-tau} above is bounded. Our general strategy is to use the definition of $\Ro_h^k$ in terms of $\til V_{h+1}^k$ (step \ref{it:def-rohk} of \algname) to bound the gaps $\Ro_h^k(x_h^k, a_h^k) - Q_h^\st(x_h^k, a_h^k)$ at step $h$ in terms of the gaps $\til V_{h+1}^{k'}(x_{h+1}^{k'}) - V_{h+1}^\st(x_{h+1}^{k'})$ at step $h+1$, for appropriate choices of $k'$. In turn, we will bound the gaps $\til V_{h+1}^{k'}(x_{h+1}^{k'}) - V_{h+1}^\st(x_{h+1}^{k'})$ in terms of the gaps $\Ro_{h+1}^{k'}(x_{h+1}^{k'}, a_{h+1}^{k'}) - Q_{h+1}^\st(x_{h+1}^{k'}, a_{h+1}^{k'})$, completing the inductive step. When proving these bounds, we must take care to meet our goal of obtaining a regret bound (see (\ref{eq:fooling-regret})) that only has terms corresponding to tuples $(x,a,h)$ belonging to the fooling set $\fool{\ep(H+1)}{\ep'}$. To do so, we use the following claim:
\begin{repclaim}{clm:aopt-fe}
  For any $(k,h)$ satisfying $\tau_h^k = 0$, if  either
  \begin{enumerate}
  \item $a_h^k \not \in \Aopt_{h,\ep'}(x_h^k)$; or
  \item $(\til V_h^k - V_h^\st)(x_h^k) > \ep'$,
  \end{enumerate}
  then under the event $\MEwc$ it holds that $(x_h^k, a_h^k, h) \in \fool{\ep(H+1)}{\ep'}$. 
\end{repclaim}
Claim \ref{clm:aopt-fe} allows us to upper bound the term $\sum_{k=1}^K \sum_{h=1}^H  (1 - \tau_h^k) \cdot \One[a_h^k \not \in \Aopt_{h,\ep'}(x_h^k)] \cdot (\Ro_h^k(x_h^k, a_h^k) - Q_h^\st(x_h^k,a_h^k))$ in the regret decomposition (\ref{eq:dis-reg-informal}) with a sum of $\Ro_h^k(x,a) - Q_h^\st(x,a)$ over only those $(x,a,h) \in \fool{\ep(H+1)}{\ep'}$. In turn, for such tuples $(x,a,h)$, it is possible to upper bound $\Ro_h^k(x,a) - Q_h^\st(x,a)$ in terms of the sum of $\beta_n$ (for $n = N_h^k(x,a)$) and a weighted sum of $(\til V_{h+1}^{k'} - V_{h+1}^\st)(x_{h+1}^{k'})$ for certain values of $k'$ (see Lemma \ref{lem:ep-good-preds}). The terms $\beta_n$ in this sum form the main contribution to the regret bound (\ref{eq:fooling-regret}); crucially we use the fact that we only have such terms for $(x,a,h) \in \fool{\ep(H+1)}{\ep'}$.

Finally, when completing the inductive step by bounding the gaps  $\til V_{h+1}^{k'}(x_{h+1}^{k'}) - V_{h+1}^\st(x_{h+1}^{k'})$, we again have to ensure that we only use terms of the form $(\Ro_{h+1}^{k'}(x,a) - Q_{h+1}^\st(x,a))$ in our upper bound for which $(x,a,h+1) \in \fool{\ep(H+1)}{\ep'}$. For this we again use Claim \ref{clm:aopt-fe} (with the second option). We refer the reader to Section \ref{sec:regret-distillation} for further details. %

\subsection{Proof of Theorem \ref{thm:main-deltaincr}: implicit-$\lambda$ bound}
\label{sec:lambda-overview}
The proof of Theorem \ref{thm:main-deltaincr} is similar to that of Theorem \ref{thm:main-deltaconst}. The main difference is that, because the algorithm is not given as input the target worst-case regret bound $\CR$ (which in turn is used to choose $\wh \Delta^k$ in \deltaconst for the proof of Theorem \ref{thm:main-deltaconst}), it must construct a proxy value to assign to $\wh \Delta^k$. This proxy is constructed in \deltaincr (Algorithm \ref{alg:delta-incr}): for each episode $k$, $\wh \Delta^k$ is set in (\ref{eq:define-whdelta-incr}) to equal an expression which resembles the definition of $\frac{1}{\lambda} \cdot \CC_{M,T,\lambda}$ in (\ref{eq:lambda-complexity}), except that (a) the minimum gap $\delmin$ is replaced with the provided lower bound $\tildm$, and (b) the gaps $\Delta_h(x,a)$ are replaced the the \emph{frozen range function} $\frzQ_h^k(x,a)$, defined below:%
\begin{defn}[Frozen range function]
  \label{def:frozen}
  For all $(x,a,h,k) \in \MS \times \MA \times [H] \times [K]$, define the \emph{frozen $Q$-function}, $\frzQ_h^k(x,a)$, as follows: given $(x,a,h)$, choose $k'  \leq k$ as large as possible so that $(x_h^{k'}, a_h^{k'}) = (x,a)$ and $\tau_h^{k'} = 1$ (if no such $k'$ exists, set $k' = 1$). Then set $\frzQ_h^k(x,a) = \clipQ_h^{k'}(x,a)$.
\end{defn}
In Lemma \ref{lem:frzq-gap} we show, roughly speaking, that the frozen range function at the final episode, namely $\frzQ_h^K(x,a)$, is still lower bounded by the gap $\Delta_h(x,a)$, justifying its use a surrogate for the gaps. The main challenge in the proof of Theorem \ref{thm:main-deltaincr}, beyond those from Theorem \ref{thm:main-deltaconst}, is the fact that $\wh \Delta^k$ changes as $k$ increases (in fact, as shown in Lemma \ref{lem:misc-monotonicity}, $\wh \Delta^k$ is non-decreasing with $k$). Most notably, this affects the proof of our bound on the number of episodes $k$ for which $\sigma_h^k = 1$ (Lemma \ref{lem:bound-sigma-incr}; the analogous lemma for \deltaconst is Lemma \ref{lem:bound-sigma}). To prove Lemma \ref{lem:bound-sigma-incr}, we partition $[K]$ into $O(\iota \cdot H)$ contiguous intervals so that inside each interval, $\wh \Delta^k$ increases by a factor of at most $1+1/H$. For each such interval $I \subset [K]$, we bound the number of $k \in I$ so that $\sigma_h^k = 1$; this leads to an increase in our regret bounds by a factor of $O(\iota H)$.

\section{Proofs for worst-case result}
\label{sec:proofs-robustness}
In this section we establish the robustness upper bounds of Theorems \ref{thm:main-deltaincr} and \ref{thm:main-deltaconst}, giving a regret bound for \algname when the user provides \emph{arbitrary} predictions $\til Q_h$. 
\subsection{Bounds on confidence intervals}
\label{sec:bound-cis}
We begin by establishing various basic guarantees on the bounds $\Qo_h^k, \Qu_h^k, \Vo_h^k, \Vu_h^k$ maintained by \algname. The first such result is Lemma \ref{lem:qo-qu-gap}, which establishes that, with high probability, $\Qo_h^k$ is an upper bound on $Q_h^\st$, $\Qu_h^k$ is a lower bound on $Q_h^\st$, and similarly for $\Vo_h^k, \Vu_h^k$ (with respect to $V_h^\st$). Before stating it, we introduce the following notation: for each $k \in [K]$, let $\MH_{k}$ denote the $\sigma$-algebra generated by all random variables up to step $H$ of episode $k$. For each $k \in [K]$ as well as $(x,a,h) \in \MS \times \MA \times [H]$, define the quantities $Q_h^{\st,k,b}(x,a)$ and $Q_h^{\st,k,r}(x,a)$ (as in \cite{xu_fine-grained_2021}) as follows: suppose we start in state $x$ at level $h$, and follow the optimal policy $\pi^\st$, generating the (random) trajectory $x_h = x, x_{h+1}, \ldots, x_H$. Choose $h' \geq h+1$ as small as possible so that $x_{h'} \not \in\MG_{h'}^k$, and write
\begin{align}
Q_h^{\st,k,b}(x,a) := \E \left[ \sum_{\ell=h}^{h'-1}  r_\ell(x_\ell, \pi_\ell^\st(x_\ell)) | \MH_{k-1} \right], \qquad Q_h^{\st, k,r}(x,a) := \E[V_{h'}^\st(x_{h'}) | \MH_{k-1}]\label{eq:qstar-decomposition-def}.
\end{align}
(Note that $\MG_{h'}^k$ is $\MH_{k-1}$-measurable for all $h'$ and $k$.) 
It is immediate that
\begin{align}
Q_h^\st(x,a) = Q_h^{\st,k,b}(x,a) + Q_h^{\st,k,r}(x,a)\nonumber.
\end{align}
As in \cite{xu_fine-grained_2021}, we use the quantities $\wh r_h^k$ as an unbiased estimate of $Q_h^{\st,k,b}(x_h^k, a_h^k)$. Recall that for some constant $C_0 > 1$, we use exploration bonuses $b_n = C_0 \sqrt{H^3 \iota /n}$, and recall the definition of the aggregated bonuses $\beta_n$ in (\ref{eq:define-beta}). %
Notice that item \ref{it:alpha-sum-sup} of Lemma \ref{lem:alpha} gives that
\begin{align}
2C_0 \sqrt{H^3 \iota/n} \leq \beta_n \leq 4 C_0 \sqrt{H^3 \iota/n}\label{eq:beta-bound}.
\end{align}
For future reference, we will also define the constants
\begin{align}
  \label{eq:define-constants}
  C_2 = 8C_0,
  \qquad C_1 = 32e^2 C_2^2 .
\end{align}

\begin{lemma}
  \label{lem:qo-qu-gap}
Set $p = 1/(H^2 K)$. For a sufficiently large choice of the constant $C_0$, there is an event $\MEwc$ occurring with probability $1-p$ so that the following holds under the event $\MEwc$, for all episodes $k \in [K]$:
  \begin{enumerate}
  \item \label{it:3-ineqs}   For any $(x,a,h) \in \MS \times \MA \times [H]$ so that $x \not\in \MG_h^k$ and $a \in A_h^k(x)$, suppose the episodes $k'$ in which $(x,a)$ as previously taken at step $h$ are denoted $k^1, \ldots, k^n \leq k$. Then the following inequalities hold:
  \begin{align}
    \Qo_h^{k+1}(x,a) - \Qu_h^{k+1}(x,a) \leq & \alpha_n^0 \cdot H + \sum_{i=1}^n \alpha_n^i \cdot \left( (\Vo_{h'(k^i,h)}^{k^i} - \Vu_{h'(k^i,h)}^{k^i})(x_{h'(k^i,h)}^{k^i}) \right) + \beta_n \label{eq:qo-qu-gap}\\
    \Qo_h^{k+1}(x,a) \geq & Q_h^\st(x,a) \geq \Qu_h^{k+1}(x,a) \label{eq:qo-qstar-qu-gap} \\% \alpha_n^0 \cdot (H - Q_h^\st(x,a)) + \sum_{i=1}^n \alpha_n^i \cdot \left( (\Vo_{h'(k^i,h)}^{k^i} - V_{h'(k^i,h)}^\st)(x_{h'(k^i,h)}^{k^i}) \right) + \sum_{i=1}^n \alpha_n^i \cdot b_i \nonumber\\
\Vo_h^{k+1}(x) \geq & V_h^\st(x) \geq \Vu_h^{k+1}(x) \label{eq:vo-vstar-vu-gap}. %
  \end{align}
\item \label{it:correct-elimination} Second, for all $(x,h) \in \MS \times [H]$ all optimal actions $a$ (i.e., those $a$ satisfying $\Delta_h(x,a) = 0$) are in $A_h^{k+1}(x)$. In particular, for all $x \in \MG_h^{k+1}$, $A_h^{k+1}(x)$ contains the unique optimal action at $x$. 
\end{enumerate}
\end{lemma}
\begin{proof}
For $k \in [K]$, we let $\MEwc_k$ denote the event that items \ref{it:3-ineqs} and \ref{it:correct-elimination} of the lemma statement hold for all episodes $j \leq k$. We wish to show that $\Pr[\MEwc_K] \geq 1-p$. 
  
We use induction on $k$ to show that for all $k$, $\Pr[\MEwc_k] \geq 1-pk/K$. The base case $k = 0$ (i.e., $k+1=1$) follows from the fact that $\alpha_0^0 = 1$, $\Qo_h^1(x,a) = H, \Qu_h^1(x,a) = 0$, and that for any choice of $(x,a,h)$ we necessarily have $n = 0$ (in particular, $\Pr[\MEwc_0] =1$). So choose any $k \geq 1$, and assume that $\Pr[\MEwc_{k-1}] \geq 1 - p(k-1)/K$. 
  
  By the algorithm's update rule in steps \ref{it:qo-update-gap} and \ref{it:qu-update-gap}, it holds that, for all $(x,a,h,k) \in \MS \times \MA \times [H] \times [K]$ so that $x \not \in \MG_h^k$, letting $n = N_h^{k+1}(x,a)$,
  \begin{align}
    \qo_h^{k+1}(x,a) =& \begin{cases}
     (1-\alpha_n) \cdot \qo_h^k(x,a) + \alpha_n \cdot \left( \wh r_h^k + \Vo_{h'(k,h)}^k(x_{h'(k,h)}^k) + b_n \right) \quad : \quad (x,a) = (x_h^k, a_h^k) \\
      \qo_h^k(x,a) \quad : \quad \mbox{otherwise}
    \end{cases} \nonumber\\
    \qu_h^{k+1}(x,a) =& \begin{cases}
     (1-\alpha_n) \cdot \qu_h^k(x,a) + \alpha_n \cdot \left( \wh r_h^k + \Vu_{h'(k,h)}^k(x_{h'(k,h)}^k) - b_n \right)\quad : \quad (x,a) = (x_h^k, a_h^k) \\
      \qu_h^k(x,a) \quad : \quad \mbox{otherwise}.
    \end{cases}\nonumber
  \end{align}
  By iterating the above, we obtain that
  \begin{align}
    \qo_h^{k+1}(x,a) =& \alpha_n^0 \cdot H + \sum_{i=1}^n \alpha_n^i \cdot \left( \wh r_h^{k^i} + \Vo_{h'(k^i,h)}^{k^i}(x_{h'(k^i,h)}^{k^i}) + b_n \right)\label{eq:little-qo-expand}\\
    \qu_h^{k+1}(x,a) =& \sum_{i=1}^n \alpha_n^i \cdot \left( \wh r_h^{k^i} + \Vu_{h'(k^i,h)}^{k^i}(x_{h'(k^i,h)}^{k^i}) -b_n \right)\label{eq:little-qu-expand},
  \end{align}
  where $k^1, \ldots, k^n \leq k$ denote all previous episodes during which $(x,a,h)$ has been visited.

  To see that (\ref{eq:qo-qu-gap}) holds, we first take the difference of (\ref{eq:little-qo-expand}) and (\ref{eq:little-qu-expand}) and use the definition of $\beta_n$ in (\ref{eq:define-beta}) to get that
  \begin{align}
\qo_h^{k+1}(x,a) - \qu_h^{k+1}(x,a) =& \alpha_n^0 \cdot H +\beta_n +  \sum_{i=1}^n \alpha_n^i \cdot \left( (\Vo_{h'(k^i,h)}^{k^i} - \Vu_{h'(k^i,h)}^{k^i})(x_{h'(k^i,h)}^{k^i})  \right)\nonumber.
  \end{align}
  Now (\ref{eq:qo-qu-gap}) follows by noting that $\Qo_h^{k+1}(x,a) \leq \qo_h^{k+1}(x,a)$ and $\Qu_h^{k+1}(x,a) \geq \qu_h^{k+1}(x,a)$ (note in particular that (\ref{eq:qo-qu-gap}) holds with probability 1).

  We proceed to analyze the event under which (\ref{eq:qo-qstar-qu-gap}) and (\ref{eq:vo-vstar-vu-gap}) hold. We may compute
  \begin{align}
    & \qo_h^{k+1}(x,a) - Q_h^\st(x,a) \nonumber\\
    =& \alpha_n^0 \cdot H + \sum_{i=1}^n \alpha_n^i \cdot \left( \wh r_h^{k^i} + \Vo_{h'(k^i,h)}^{k^i}(x_{h'(k^i,h)}^{k^i}) + b_n\right) - Q_h^\st(x,a) \nonumber\\
     =& \alpha_n^0 \cdot (H - Q_h^\st(x,a)) + \sum_{i=1}^n \alpha_n^i  \cdot \left( \wh r_h^{k^i} +  \Vo_{h'(k^i,h)}^{k^i}(x_{h'(k^i,h)}^{k^i}) - Q_h^\st(x,a) \right)  + \sum_{i=1}^n \alpha_n^i \cdot b_n \tag{Using item \ref{it:alpha-sum-one} of Lemma \ref{lem:alpha} and $b_0 = 0$}\nonumber\\
    =& \alpha_n^0 \cdot (H - Q_h^\st(x,a)) + \beta_n/2 + \sum_{i=1}^n \alpha_n^i \cdot \left( \wh r_h^{k^i} - Q_h^{\st,k^i,b}(x,a) \right) \nonumber\\
    & + \sum_{i=1}^n \alpha_n^i \cdot \left( \Vo_{h'(k^i,h)}^{k^i}(x_{h'(k^i,h)}^{k^i}) - V_{h'(k^i,h)}^{\st}(x_{h'(k^i,h)}^{k^i}) \right) + \sum_{i=1}^n \alpha_n^i \cdot \left( V_{h'(k^i,h)}^{\st}(x_{h'(k^i,h)}^{k^i}) - Q_h^{\st,k^i,r}(x,a)\right) \label{eq:little-qo-qstar}.
  \end{align}
  In a similar manner, we have
  \begin{align}
    & \qu_h^{k+1}(x,a) - Q_h^\st(x,a) \nonumber\\
    =& -\alpha_n^0 \cdot Q_h^\st(x,a) - \beta_n/2 + \sum_{i=1}^n \alpha_n^i \cdot \left( \wh r_h^{k^i} - Q_h^{\st, k^i, b}(x,a) \right) \nonumber\\
    &+ \sum_{i=1}^n \alpha_n^i \cdot \left( \Vu_{h'(k^i,h)}^{k^i}(x_{h'(k^i,h)}^{k^i}) - V_{h'(k^i,h)}^{\st}(x_{h'(k^i,h)}^{k^i}) \right) + \sum_{i=1}^n \alpha_n^i \cdot \left( V_{h'(k^i,h)}^\st(x_{h'(k^i,h)}^{k^i}) - Q_h^{\st,k^i,r}(x,a) \right)\label{eq:little-qu-qstar}.
  \end{align}

  \begin{claim}
    \label{clm:azuma-1}
  There is an event $\ME_k \subset \MEwc_{k-1}$ so that $\Pr[\ME_k] \geq 1-pk/K$ and the following holds under $\ME_k$: for all $h \in [H]$, all $x \in \MS \backslash \MG_h^k$, and all $a \in A_h^k(x)$, letting $n = N_h^{k+1}(x,a)$ and $k^1,\ldots, k^n \leq k$ denote all the previous episodes in which $(x,a,h)$ was previously visited,
    \begin{align}
      \left| \sum_{i=1}^n \alpha_n^i \cdot \left( \wh r_h^{k^i} - Q_h^{\st,k^i,b}(x,a) \right) \right| \leq &\sqrt{\frac{H^3}{n} \cdot \log \left( \frac{4SAHK}{p} \right)}  \label{eq:rhat-qstar-r}\\
      \left| \sum_{i=1}^n \alpha_n^i \cdot \left( V_{h'(k^i,h)}^\st(x_{h'(k^i,h)}^{k^i}) - Q_h^{\st, k^i,r}(x,a) \right) \right| \leq & \sqrt{\frac{H^3}{n} \cdot \log \left( \frac{4SAHK}{p} \right)}\label{eq:vstar-qstar-b}.
    \end{align}
  \end{claim}
  \begin{proof}[Proof of Claim \ref{clm:azuma-1}]

    For $1 \leq k' \leq k$, recall that $\MH_{k'}$ denotes the $\sigma$-algebra generated by all random variables up to (step $H$ of) episode $k'$. Note that $k^1, \ldots, k^n$ are all stopping times with respect to the filtration $\MH_{k'}$. Moreover, it is evident that $\MEwc_{k'}$ is $\MF_{k'}$-measurable for all $k' \leq k$. Next, for $1 \leq i \leq n$, define the filtration $\MF_i$ by $\MF_i := \MH_{k^i-1}$. Moreover, for $i \in [n]$, define
    \begin{align}
M_i := \alpha_n^i \cdot \left( \wh r_h^{k^i} - Q_h^{\st, k^i, b}(x,a) \right) \cdot \One[\MEwc_{k^{i}-1}]\nonumber.
    \end{align}
    Since $k^{i+1} > k^i$, $M_i$ is $\MF_{i+1}$-measurable for each $i$ (as a matter of convention we set $k^{n+1} = k+1$, so $M_i$ is $\MF_{i+1}$-measurable even for $i=n$). Moreover, we claim that for each $i$,
    \begin{align}
\E[M_i | \MF_i] = \alpha_n^i \cdot \E\left[ \left( \wh r_h^{k^i} - Q_h^{\st, k^i,b}(x,a) \right) \cdot \One[\MEwc_{k^i-1}] | \MF_i\right] = 0\label{eq:martingale1}.
    \end{align}
    To see that (\ref{eq:martingale1}) holds, first note that conditioned on $\MF_i$, $\One[\MEwc_{k^i-1}] \cdot \wh r_h^{k^i}$ is distributed identically to $\One[\MEwc_{k^i-1}] \cdot \sum_{\ell=h}^{h'-1} r_\ell(x_\ell, \pi_\ell^\st(x_\ell))$ where $x_\ell$ is the sequence of states visited starting at $x_h = x$ and following the optimal policy $\pi^\st$ and $h'$ is as small as possible so that $x_{h'} \not \in \MG_{h'}^{k^i}$ (this holds since item \ref{it:correct-elimination} at episode $k^i-1$ gives that the unique action in $A_{\ell}^{k^i}(x_\ell^{k^i})$, for $h \leq \ell < h'(k^i,h)$ is the optimal action, namely $\pi_\ell^\st(x_\ell^{k^i})$).
    Recall from (\ref{eq:qstar-decomposition-def}) that $\E \left[ \sum_{\ell=h}^{h'-1} r_\ell(x_\ell, \pi_\ell^\st(x_\ell)) - Q_h^{\st, k^i, b}(x,a) | \MF_i\right] = 0$; then the fact that  $\MEwc_{k^i-1}$ is $\MF_i$-measurable gives (\ref{eq:martingale1}).

    Next, for $i \in [n]$, define
    \begin{align}
N_i := \alpha_n^i \cdot \left( V_{h'(k^i,h)}^\st(x_{h'(k^i,h)}^{k^i}) - Q_h^{\st, k^i, r}(x,a) \right) \cdot \One[\MEwc_{k^i-1}]\nonumber.
    \end{align}
    Since $k^{i+1} > k^i$, $N_i$ is $\MF_{i+1}$-measurable for each $i$. Moreover, we claim that for each $i$,
    \begin{align}
\E[N_i | \MF_i] = \alpha_n^i \cdot \E \left[ \left( V_{h'(k^i,h)}^\st(x_{h'(k^i,h)}^{k^i}) - Q_h^{\st, k^i, r}(x,a) \right) \cdot \One[\MEwc_{k^i-1}] | \MF_i \right] = 0\label{eq:martingale2}.
    \end{align}
    The validity of (\ref{eq:martingale2}) is verified in the same way as that of (\ref{eq:martingale1}): conditioned on $\MF_i$, $\One[\MEwc_{k^i-1}] \cdot V_{h'(k^i,h)}^\st(x_{h'(k^i,h)}^{k^i})$ is distributed identically to $\One[\MEwc_{k^i-1}] \cdot V_{h'}^\st(x_{h'})$, where $x_h, \ldots, x_{h'}$ is defined as above, namely it is the sequence of states visited starting at $x_h = x$ and following the optimal policy $\pi^\st$, and $h'$ is as small as possible so that $x_{h'} \not\in \MG_{h'}^{k^i}$ (again we use that item \ref{it:correct-elimination} holds at episode $k^i-1$). Now (\ref{eq:qstar-decomposition-def}) gives that $\E[V_{h'}^\st(x_{h'}) - Q_h^{\st,k^i,r}(x,a) | \MF_i] = 0$ and using this together with the fact that $\MEwc_{k^i-1}$ is $\MF_i$-measurable gives (\ref{eq:martingale2}).

    Equations (\ref{eq:martingale1}) and (\ref{eq:martingale2}) give that $M_i$ and $N_i$ are martingales with respect to the filtration $\MF_{i+1}$. The fact that $\sum_{i=1}^n (\alpha_n^i)^2 \leq \frac{2H}{n}$ (item \ref{it:alpha-sum-square} of Lemma \ref{lem:alpha}) together with the Azuma-Hoeffding inequality then gives that, for fixed $x,a,h$, with probability $1-p/(SAHK)$, both of the below inequalities hold:
    \begin{align}
      \left| \sum_{i=1}^n \alpha_n^i \cdot \left( \wh r_h^{k^i} - Q_h^{\st,k^i,b}(x,a) \right) \cdot \One[\MEwc_{k^i-1}] \right| \leq & \sqrt{\frac{H^3 \cdot \log \left( \frac{4SAHK}{p} \right)}{n}}\label{eq:mi-azuma}\\
      \left| \sum_{i=1}^n \alpha_n^i \cdot \left( V_{h'(k^i,h)}^\st(x_{h'(k^i,h)}^{k^i}) - Q_h^{\st, k^i,r}(x,a) \right) \cdot \One[\MEwc_{k^i-1}]\right| \leq &\sqrt{\frac{H^3 \cdot \log \left( \frac{4SAHK}{p} \right)}{n}}\label{eq:ni-azuma}.
    \end{align}
    Let $\ME_{k}$ denote the intersection of the probability $1-p/K$-event that both (\ref{eq:mi-azuma}) and (\ref{eq:ni-azuma}) hold for all $x,a,h$ and the event $\MEwc_{k-1}$. Then using the inductive hypothesis that $\Pr[\MEwc_{k-1}] \geq 1-p(k-1)/K$, we get that $\Pr[\ME_k] \geq 1-pk/K$. Thus, under the event $\ME_k$, we have that
    \begin{align}
      \left| \sum_{i=1}^n \alpha_n^i \cdot \left( \wh r_h^{k^i} - Q_h^{\st,k^i,b}(x,a) \right) \right| \leq & \sqrt{\frac{H^3}{n} \cdot \log \left( \frac{4SAHK}{p} \right)}\nonumber\\
            \left| \sum_{i=1}^n \alpha_n^i \cdot \left( V_{h'(k^i,h)}^\st(x_{h'(k^i,h)}^{k^i}) - Q_h^{\st, k^i,r}(x,a) \right)\right| \leq &\sqrt{\frac{H^3}{n} \cdot \log \left( \frac{4SAHK}{p} \right)}\nonumber,
      \end{align}
    completing the proof of the claim.
  \end{proof}
  Next we show that, on the event $\ME_k$, both (\ref{eq:qo-qstar-qu-gap}) and (\ref{eq:vo-vstar-vu-gap}) hold at episode $k$, for all $x,a,h$. %
  Note that, by (\ref{eq:beta-bound}), for all $n$,
  \begin{align}
\beta_n/4 \geq  \frac{C_0}{2} \cdot \sqrt{\frac{H^3 \iota}{n}} \geq \sqrt{\frac{H^3}{n} \cdot \log \left( \frac{4SAHK}{p} \right)}\nonumber
  \end{align}
  as long as the constant $C_0$ is chosen to be large enough. Thus, by (\ref{eq:little-qo-qstar}) and Claim \ref{clm:azuma-1}, under the event $\ME_k$, we have that
  \begin{align}
    & \qo_h^{k+1}(x,a) - Q_h^\st(x,a) \nonumber \\
    \geq & \beta_n/2 - \beta_n/4 + \sum_{i=1}^n \alpha_n^i \cdot \left( \Vo_{h'(k^i,h)}^{k^i}(x_{h'(k^i,h)}^{k^i}) - V_{h'(k^i,h)}^{\st}(x_{h'(k^i,h)}^{k^i}) \right) - \beta_n/4 \nonumber\\
    =& \sum_{i=1}^n \alpha_n^i \cdot \left( \Vo_{h'(k^i,h)}^{k^i}(x_{h'(k^i,h)}^{k^i}) - V_{h'(k^i,h)}^{\st}(x_{h'(k^i,h)}^{k^i}) \right)\geq 0\label{eq:qk1-qstar-pos},
  \end{align}
  where the final inequality follows from the fact that $\ME_k \subset \MEwc_{k-1}$ and under $\MEwc_{k-1}$, (\ref{eq:vo-vstar-vu-gap}) holds (in particular, at step $h'(k^i,h)$ for state $x_{h'(k^i,h)}^{k^i}$). Using the fact that $\Qo_h^{k+1}(x,a) = \min_{k' \leq k+1} \left\{ \qo_h^{k'}(x,a) \right\}$ (step \ref{it:Qo-qo-min} of \algname) together with the fact that $\Qo_h^k(x,a) \geq Q_h^\st(x,a)$ under $\MEwc_{k-1}$, we see from (\ref{eq:qk1-qstar-pos}) that $\Qo_h^{k+1}(x,a) \geq Q_h^\st(x,a)$ under the event $\ME_k$ (for all $a \in A_h^k(x)$). Since $\Vo_h^{k+1}(x) = \max_{a \in A_h^k(x)} \Qo_h^{k+1}(x,a)$ (step \ref{it:vo-define} of \algname), it follows that $\Vo_h^{k+1}(x) \geq V_h^\st(x)$ under the event $\ME_k$.

  Thus we have verified the first inequality in each of (\ref{eq:qo-qstar-qu-gap}) and (\ref{eq:vo-vstar-vu-gap}) at episode $k$. The proof of the second inequality in each follows identically: (\ref{eq:little-qu-qstar}) together with Claim \ref{clm:azuma-1} gives that under the event $\ME_k$, we have that $\qu_h^{k+1}(x,a) - Q_h^\st(x,a) \leq 0$ for all $x,a,h$. Then it follows that $\Qu_h^{k+1}(x,a) \leq Q_h^\st(x,a)$, and using the fact that $\Vu_h^{k+1}(x) = \max_{a \in A_h^k(x)} \Qu_h^{k+1}(x,a)$ (step \ref{it:vu-define} of \algname), it follows that $\Vu_h^{k+1}(x) \leq V_h^\st(x)$ under the event $\ME_k$. Thus we have  verified that (\ref{eq:qo-qu-gap}), (\ref{eq:qo-qstar-qu-gap}), and (\ref{eq:vo-vstar-vu-gap}) hold (for any choice of $x,a,h$ with $x \not \in \MG_h^k$ and $a \in A_h^k(x)$) at episode $k$, under the event $\ME_k$. 

  Finally we verify that item \ref{it:correct-elimination} holds at episode $k$ under the event $\ME_k$. Suppose, to the contrary, that there were some optimal action $a^\st$ for some state $(x,h)$ so that $a^\st \not \in A_h^{k+1}(x)$. Since $\ME_{k} \subset \MEwc_{k-1}$, we have that $a^\st \in A_h^k(x)$, meaning that by the definition of $A_h^{k+1}(x)$ in step \ref{it:def-active-set} of \algname, we must have that $\Qo_h^{k+1}(x,a^\st) < \Vu_h^{k+1}(x)$. But we have just shown that under the event $\ME_k$, weh have that $\Vu_h^{k+1}(x) \leq V_h^\st(x)$ and $\Qo_h^{k+1}(x,a^\st) \geq Q_h^\st(x, a^\st)$, which implies that $Q_h^\st(x, a^\st) < V_h^\st(x)$, contradicting the fact that $a^\st$ is an optimal action at $(x,h)$.

  Thus we have shown that all statements in items \ref{it:3-ineqs} and \ref{it:correct-elimination} for episode $k$ hold under the event $\ME_k$, and $\ME_k \subset \MEwc_{k-1}$ as well as $\Pr[\ME_k] \geq 1-pk/K$. Thus $\MEwc_k \supset \ME_k$, meaning that $\Pr[\MEwc_k] \geq 1-pk/K$, which completes the proof of the inductive step.
\end{proof}

The following lemma shows that the upper and lower confidence bounds satisfy a monotonicity property with respect to the number of episodes $k$ that have elapsed: in particular, the upper confidence bounds on $V_h^\st, Q_h^\st$ maintained by \algname are non-increasing, and the lower confidence bounds on $V_h^\st, Q_h^\st$ are non-decreasing. Note that in many previous works studying $Q$-learning algorithms (such as \cite{jin_is_2018,xu_fine-grained_2021}), these monotonicity properties do not necessarily hold -- it is necessary to modify the $Q$- and $V$-value updates in \algname appropriately to ensure that Lemma \ref{lem:qv-monotonicity} holds.
\begin{lemma}
  \label{lem:qv-monotonicity}
  For all $k \in [K]$, the following inequalities hold for all $(x,a,h) \in \MS \times \MA \times [H]$:
  \begin{align}
    \Qo_h^{k+1}(x,a) \leq & \Qo_h^k(x,a)\label{eq:qo-mon}\\
    \Qu_h^{k+1}(x,a) \geq & \Qu_h^k(x,a)\label{eq:qu-mon}\\
    \Vo_h^{k+1}(x) \leq & \Vo_h^k(x) \label{eq:vo-mon}\\
    \Vu_h^{k+1}(x) \geq & \Vu_h^k(x) \label{eq:vu-mon}\\
    \til Q_h^{k+1}(x,a) \leq & \til Q_h^k(x,a) \label{eq:tilq-mon}\\
    \til V_h^{k+1}(x) \leq & \til V_h^k(x) \label{eq:tilv-mon}.
  \end{align}
\end{lemma}
\begin{proof}
Fix any $k \in [K]$ and $h \in [H]$. First note that step \ref{it:copy-qv} of \algname verifies (\ref{eq:qo-mon}) through (\ref{eq:tilv-mon}) for all $(x,a) \neq (x_h^k, a_h^k)$. So it remains to to consider the case that $x=x_h^k$ and $a =a_h^k$.
  
First note that (\ref{eq:qo-mon}) and (\ref{eq:qu-mon}) are directly verified by steps \ref{it:Qo-qo-min} and \ref{it:Qu-qu-max}, respectively, of \algname. To verify (\ref{eq:vo-mon}), note that $A_h^k(x_h^k) \subset A_h^{k-1}(x_h^k)$, meaning that
\begin{align}
\Vo_h^{k+1}(x_h^k) = \max_{a' \in A_h^k(x_h^k)} \{ \Qo_h^{k+1}(x_h^k, a') \} \leq \max_{a' \in A_h^{k-1}(x_h^k)} \{ \Qo_h^{k+1}(x_h^k, a') \} \leq \max_{a' \in A_h^{k-1}(x_h^k)} \{ \Qo_h^k(x_h^k, a') \} = \Vo_h^k(x_h^k),\nonumber
\end{align}
where the second inequality uses (\ref{eq:qo-mon}) and the last equality uses steps \ref{it:vu-define} and \ref{it:copy-qv} of \algname (in particular, note that $\Vo_h^k(x_h^k)$ and $\Qo_h^k(x_h^k, \cdot)$ remain unchanged from the previous episode before $k$ at which $x_h^k$ was visited). 

Next we verify (\ref{eq:vu-mon}); choose $a \in A_h^k(x_h^k)$ so that $\Vu_h^{k+1}(x_h^k) = \Qu_h^{k+1}(x_h^k, a)$, and $a' \in A_h^{k-1}(x_h^k)$ so that $\Vu_h^k = \Qu_h^k(x_h^k, a')$. If (\ref{eq:vu-mon}) did not hold, we would have that $\Qu_h^{k+1}(x_h^k, a') \geq \Qu_h^k(x_h^k, a') > \Qu_h^{k+1}(x_h^k, a)$, which must mean that $a' \not\in A_h^k(x_h^k)$. But this is impossible since $\Qo_h^{k+1}(x_h^k, a') \geq \Qu_h^{k+1}(x_h^k,a') = \Vu_h^{k+1}(x_h^k)$, so by step \ref{it:def-active-set} of \algname $a'$ must belong to $A_h^k(x_h^k)$.

Finally, (\ref{eq:tilq-mon}) is verified by step \ref{it:def-tilq-hk} of \algname, and to see that (\ref{eq:tilv-mon}) holds, we again use that $A_h^k(x_h^k) \subset A_h^{k-1}(x_h^k)$ to get that
\begin{align}
  \til V_h^{k+1}(x_h^k) =& \max_{a' \in A_h^k(x_h^k)} \max \left\{ \til Q_h^{k+1}(x_h^k, a'), \Qu_h^{k+1}(x_h^k, a') \right\} \nonumber\\
  \leq &  \max_{a' \in A_h^{k-1}(x_h^k)} \max \left\{ \til Q_h^{k+1}(x_h^k, a'), \Qu_h^{k+1}(x_h^k, a') \right\} \nonumber\\
  \leq &  \max_{a' \in A_h^{k-1}(x_h^k)} \max \left\{ \til Q_h^{k}(x_h^k, a'), \Qu_h^{k}(x_h^k, a') \right\} \label{eq:use-qu-tilq-mon}\\
  = & \til V_h^k(x_h^k)\nonumber,
\end{align}
where (\ref{eq:use-qu-tilq-mon}) uses both (\ref{eq:qu-mon}) and (\ref{eq:tilq-mon}). 
\end{proof}

\subsection{Range functions and clipped range functions}
Recall the definition of the range functions $\RanV_h^k, \RanQ_h^k$ in Definition \ref{def:range-function}, as well as the clipped range functions $\clipV_h^k, \clipQ_h^k$ in Definition \ref{def:half-clipped}. In this section we establish some basic guarantees of these functions. The first such result, Lemma \ref{lem:range-bound}, shows that the range functions are upper bounds on the gap between the upper and lower estimates for the $Q$- and $V$-value functions.
\begin{lemma}
  \label{lem:range-bound}
  For all $(x,h,k,a) \in \MS \times \MA \times [H] \times [K]$ for which $x \not \in \MG_h^k$ and $a \in A_h^k(x)$, the range functions satisfy the following under the event $\MEwc$:
  \begin{align}
    \RanQ_h^k(x,a) \geq & \Qo_h^k(x,a) - \Qu_h^k(x,a) \nonumber\\
    \RanV_h^k(x) \geq & \Vo_h^k(x) - \Vu_h^k(x)\nonumber.
  \end{align}
\end{lemma}
\begin{proof}
  The proof closely follows that of \cite[Lemma B.3]{xu_fine-grained_2021}. We use reverse induction on $h$. The base case $h=H+1$ is immediate since $\Qo_{H+1}^k, \Qu_{H+1}^k, \Vo_{H+1}^k, \Vu_{H+1}^k$ are defined to be identically 0 for all $k \in [K]$. Next fix $h \leq H$ and suppose that the statement of the lemma holds for all $(x,h',k,a)$ for which $h' > h$. For any $(x,k,a)$ for which $x \not \in \MG_h^k$ and $a \in A_h^k(x)$, note that, for $n = N_h^k(x,a)$ and $k_h^i = k_h^i(x,a)$,
  \begin{align}
    \Qo_h^k(x,a) -\Qu_h^k(x,a) \leq & \alpha_{n}^0 \cdot H + \sum_{i=1}^{n} \alpha_{n}^i \cdot (\Vo_{h'(k_h^i,h)}^{k_h^i} - \Vu_{h'(k_h^i,h)}^{k_h^i})(x_{h'(k_h^i,h)}^{k_h^i}) + \beta_n\label{eq:use-qo-qu-gap}\\
    \leq & \alpha_{n}^0 \cdot H + \sum_{i=1}^{n} \alpha_{n}^i \cdot \RanV_{h'(k_h^i,h)}^{k_h^i}(x_{h'(k_h^i,h)}^{k_h^i}) + \beta_{n}\label{eq:get-range-function}\\
    = & \RanQ_h^k(x,a)\label{eq:back-to-range}.
  \end{align}
  where (\ref{eq:use-qo-qu-gap}) follows from (\ref{eq:qo-qu-gap}) of Lemma \ref{lem:qo-qu-gap} (in particular, we use the validity of (\ref{eq:qo-qu-gap}) for episode $k-1$) and  (\ref{eq:get-range-function}) uses the inductive hypothesis (since $h'(k_h^i,h) > h$). %
  Thus the inductive step for $\RanQ_h^k$ is verified. To lower bound $\RanV_h^k$, we first note that by Definition \ref{def:range-function} for any $x \not\in \MG_h^k$, there is some $k' \leq k$ so that $\RanV_h^k(x) = \RanQ_h^{k'}(x,a^\st)$ for $a^\st = \argmax_{a' \in A_h^{k'}(x)} \{ \Qo_h^{k'}(x,a') - \Qu_h^{k'}(x,a') \}$. Then
  \begin{align}
    \Vo_h^k(x) - \Vu_h^k(x) \leq & \Vo_h^{k'}(x) - \Vu_h^{k'}(x) \tag{Using (\ref{eq:vo-mon}) of Lemma \ref{lem:qv-monotonicity}} \\
    = & \left( \max_{a' \in A_h^{k'}(x)} \Qo_h^{k'}(x,a') \right) - \left( \max_{a' \in A_h^{k'}(x)} \Qu_h^{k'}(x,a') \right) \label{eq:use-vovu-defns}\\
    \leq & \max_{a' \in A_h^{k'}(x)} \left\{ \Qo_h^{k'}(x,a') - \Qu_h^{k'}(x,a') \right\} \nonumber\\
    \leq & \Qo_h^{k'}(x,a^\st) - \Qu_h^{k'}(x,a^\st) \tag{For $a^\st = \argmax_{a' \in A_h^{k'}(x)} \{ \Qo_h^{k'}(x,a') - \Qu_h^{k'}(x,a')\}$}\\
    \leq & \RanQ_h^{k'}(x,a^\st) \tag{Using (\ref{eq:back-to-range})}\\
    = & \RanV_h^k(x)\label{eq:use-range-def},
  \end{align}
  where (\ref{eq:use-vovu-defns}) follows from steps \ref{it:vu-define} and \ref{it:vo-define} of \algname and the final equality (\ref{eq:use-range-def}) follows from the definition of $k'$. %
\end{proof}

Lemma \ref{lem:hc-range-bound} below shows that despite the clipping of the bonus $\beta_n$ in the definition of the clipped range function (see (\ref{eq:clipq-define})), the clipped range functions $\clipQ_h^k, \clipV_h^k$ remain approximate upper bounds on the range functions $\RanQ_h^k, \RanV_h^k$.
\begin{lemma}
  \label{lem:hc-range-bound}
  For all $(x,h,k,a) \in \MS \times \MA \times [H] \times [K]$ for which $x \not \in \MG_h^k$ and $a \in A_h^k(x)$, the partially-clipped range functions satisfy the following:
  \begin{align}
    \clipQ_h^k(x,a) \geq & \RanQ_h^k(x,a) - \frac{\delmin}{4H} \nonumber\\
    \clipV_h^k(x) \geq & \RanV_h^k(x) - \frac{\delmin}{4H}\nonumber.
  \end{align}
\end{lemma}
\begin{proof}
  The lemma follows in a similar manner to \cite[Proposition B.5]{xu_fine-grained_2021}. %
  We prove by reverse induction on $h$ and forward induction on $k$ that
  \begin{align}
    \clipQ_h^k(x,a) \geq \RanQ_h^k(x,a) - \frac{(H+1-h)}{H} \cdot \frac{\delmin}{4H}\label{eq:clipq-ranq-ind}
  \end{align}
  and
  \begin{align}
\clipV_h^k(x) \geq \RanV_h^k(x) - \frac{(H+1-h)}{H} \cdot \frac{\delmin}{4H}\label{eq:clipv-ranv-ind}.
  \end{align}

  The base case $h=H+1$ is immediate since $\clipQ_{H+1}^k, \RanQ_{H+1}^k, \clipV_{H+1}^k, \RanV_{H+1}^k$ are identically 0. The base case $k=0$ is also immediate since $\clipQ_h^0(x,a) = \clipV_h^0(x) = \RanQ_h^0(x,a) = \RanV_h^0(x) = H$ for all $x,a,h \in \MS \times \MA \times [H]$. To establish the inductive step, note that, for any $(x,a,h,k)$ for which $x \not \in \MG_h^k$ and $a \in A_h^k(x)$, letting $n = N_h^k(x,a)$ and $k_h^i = k_h^i(x,a)$ for $i \in [n]$, we have
  \begin{align}
     \clipQ_h^k(x,a) =& \min \left\{ \clipQ_h^{k-1}(x,a),\ \alpha_n^0 H + \clip{\beta_n}{\frac{\delmin}{4H^2}} + \sum_{i=1}^n \alpha_n^i \cdot \clipV_{h'(k_h^i, h)}^{k_h^i}(x_{h'(k_h^i,h)}^{k_h^i}) \right\}\nonumber\\
    \geq &\min \left\{ \RanQ_h^{k-1}(x,a) - \frac{(H+1-h)\delmin}{4H^2},\right.\nonumber\\
                      & \left. \alpha_n^0 H + \beta_n + \sum_{i=1}^n \alpha_n^i \cdot \clipV_{h'(k_h^i, h)}^{k_h^i}(x_{h'(k_h^i,h)}^{k_h^i}) - \frac{\delmin}{4H^2} \right\}\nonumber\\
        \geq &\min \left\{ \RanQ_h^{k-1}(x,a) - \frac{(H+1-h)\delmin}{4H^2},\right.\nonumber\\
                      & \left. \alpha_n^0 H + \beta_n + \sum_{i=1}^n \alpha_n^i \cdot \RanV_{h'(k_h^i, h)}^{k_h^i}(x_{h'(k_h^i,h)}^{k_h^i}) -\frac{(H-h) \delmin}{4H^2} -\frac{\delmin}{4H^2} \right\}\label{eq:remove-clip}\\
    = &  \RanQ_h^k(x,a) - \frac{(H+1-h)\delmin}{4H^2}\nonumber,
  \end{align}
  where (\ref{eq:remove-clip}) used the inductive hypothesis (in particular, (\ref{eq:clipv-ranv-ind}) at steps $h' > h$). This establishes the inductive step for (\ref{eq:clipq-ranq-ind}).

  We proceed to lower bound $\clipV_h^k(x,a)$. For fixed $x,a$, set $a^\st = \argmax_{a' \in A_h^k(x)} \{ \Qo_h^k(x,a') - \Qu_h^k(x,a') \}$. Using Definition \ref{def:half-clipped}, the inductive hypothesis, and the validity of (\ref{eq:clipq-ranq-ind}) for step $h$ at episode $k$, we have
  \begin{align}
    \clipV_h^k(x) =& \min \{ \clipV_h^{k-1}(x), \clipQ_h^k(x, a^\st) \} \nonumber\\
    \geq &  \min \{ \RanV_h^{k-1}(x), \RanQ_h^k(x,a^\st) \} - \frac{(H+1-h) \delmin}{4H^2} \nonumber,
  \end{align}
  which completes the inductive step for (\ref{eq:clipv-ranv-ind}).
\end{proof}

The following straightforward lemma shows that the clipped range functions are smaller than the range functions.
\begin{lemma}
  \label{lem:clip-lt-range}
  For all $(x,h,k,a) \in \MS \times \MA \times [H] \times [K]$ for which $x \not \in \MG_h^k$ and $a \in A_h^k(x)$, it holds that
  \begin{align}
\clipQ_h^k(x,a) \leq \RanQ_h^k(x,a) \qquad \mbox{ and } \qquad \clipV_h^k(x) \leq \RanV_h^k(x)\nonumber.
  \end{align}
\end{lemma}
\begin{proof}
  The lemma is a straightforward consequence of Definitions \ref{def:range-function} and \ref{def:half-clipped} and induction on $h,k$ (in particular, forward induction on $k$ and reverse induction on $h$): in particular, for any $h,k$, having established the statement for all $(h',k')$ with either $h' > h$ or $k' < k$, we have that $\clipQ_h^k(x,a) \leq \RanQ_h^k(x,a)$ since
  \begin{align}
    & \alpha_n^0 H + \clip{\beta_n}{\frac{\delmin}{4H^2}} + \sum_{i=1}^n \alpha_n^i \cdot \clipV_{h'(k_h^i(x,a), h)}^{k_h^i(x,a)}(x_{h'(k_h^i(x,a),h)}^{k_h^i(x,a)}) \nonumber\\
    \leq &  \alpha_n^0 H + \beta_n + \sum_{i=1}^n \alpha_n^i \cdot \clipV_{h'(k_h^i(x,a), h)}^{k_h^i(x,a)}(x_{h'(k_h^i(x,a),h)}^{k_h^i(x,a)}) \nonumber.
  \end{align}
  It then follows immediately that $\clipV_h^k(x) \leq \RanV_h^k(x)$. 
\end{proof}

Lemma \ref{lem:clip-monotonicity} establishes some monotonicity (with respect to $k$) properties of the range functions, analogously to Lemma \ref{lem:qv-monotonicity}.
\begin{lemma}
  \label{lem:clip-monotonicity}
  For all $(x,a,h,k)$ so that $x \not \in \MG_h^{k+1}$ and $a \in A_h^{k+1}(x)$, the following inequalities hold true:
  \begin{align}
    \clipQ_h^{k+1}(x,a) \leq & \clipQ_h^{k}(x,a) \label{eq:clipq-mon}\\
    \clipV_h^{k+1}(x) \leq & \clipV_h^k(x)\nonumber\\
    \RanQ_h^{k+1}(x,a) \leq & \RanQ_h^k(x,a) \nonumber\\
    \RanV_h^{k+1}(x) \leq & \RanV_h^k(x)\label{eq:ranv-mon}.
  \end{align}
  Moreover, for all $(x,a,h,k) \in \MS \times \MA \times [H] \times [K]$, it holds that
  \begin{align}
\frzQ_h^{k+1}(x,a) \leq \frzQ_h^k(x,a)\label{eq:frzq-mon}.
  \end{align}
\end{lemma}
\begin{proof}
  The first four inequalities are  immediate from Definitions \ref{def:range-function} and \ref{def:half-clipped}. The final inequality follows from Definition \ref{def:frozen} and (\ref{eq:clipq-mon}). 
\end{proof}

Lemma \ref{lem:misc-monotonicity} establishes some further monotonicity properties for \algname.
\begin{lemma}
  \label{lem:misc-monotonicity}
  The following statements hold true:
  \begin{enumerate}
  \item When \algname is run with either \deltaconst or \deltaincr, for all $k \in [K]$, it holds that $\wh \Delta^{k+1} \geq \wh \Delta^k$. 
  \item  For any $h \in [H]$ and $k < k'$ for which $x_h^k = x_h^{k'}$, we have $\tau_h^{k'} \leq \tau_h^k$.
  \end{enumerate}
\end{lemma}
\begin{proof}
  We begin with the first statement; it is immediate for \deltaconst. In the case of \deltaincr, we note that by (\ref{eq:frzq-mon}) of Lemma \ref{lem:clip-monotonicity}, $\frzQ_h^k(x,a)$ is non-increasing as a function of $k$ for all $x,a,h$. It is clear that the same is true of $\tfrzQ_h^k(x,a)$ (defined in step \ref{it:tfrzq-define} of Algorithm \ref{alg:delta-incr}).
  Thus the expression in (\ref{eq:define-whdelta-incr}) is non-decreasing as a function of $k$.
  
  To see the second statement, note that if $\tau_h^k = 0$, then either $|A_h^k(x_h^k)| = 1$, in which case it will hold that $|A_h^{k'}(x_h^k)| = 1$ (and so $\tau_h^{k'} = 0$), or $\RanV_h^k(x_h^k) \leq \ggapfn{h}{\wh \Delta^k}$, in which case it holds that $\RanV_n^{k'}(x_h^k) \leq \ggapfn{h}{\wh \Delta^k} \leq \ggapfn{h}{\wh \Delta^{k'}}$ (and so $\tau_h^{k'} = 0$), by (\ref{eq:ranv-mon}) and the first item of this lemma. 
\end{proof}

Recall that the clip function is defined as follows: for real numbers $x,y$, we have $\clip{x}{y} = x \cdot \One[x \geq y]$. We next state some lemmas establishing useful properties of the clip function in Lemmas \ref{lem:clip}, \ref{lem:clip-sum}, and \ref{lem:clip-sum-gap} below.
\begin{lemma}[Claim A.8, \cite{xu_fine-grained_2021}]
  \label{lem:clip}
  For any positive integers $a,b,c$ so that $a + b \geq c$ and any $x \in (0,1)$, it holds that
  \begin{align}
a+b \leq \clip{a}{\frac{xc}{2}} + (1+x)b \nonumber.
  \end{align}
\end{lemma}

\begin{lemma}[Claim A.13, \cite{xu_fine-grained_2021}]
  \label{lem:clip-sum}
  For any $c, \ep > 0$, it holds that
  \begin{align*}
\sum_{n=1}^\infty \clip{\frac{c}{\sqrt{n}}}{\ep} \leq \frac{4c^2}{\ep}.
  \end{align*}
\end{lemma}

\begin{lemma}
  \label{lem:clip-sum-gap}
  Fix some $c > 0$ and $h \in [H]$. For $n \in \BN$, write $\gamma_n = c/\sqrt{n}$. Then for any function $\theta : \MS \times \MA \ra \BR_{\geq 0}$, and any subset $\MW \subset [K]$ of size $M := |\MW|$, it holds that
  \begin{align}
\sum_{k \in \MW} \clip{\gamma_{n_h^k}}{\theta(x_h^k,a_h^k)} \leq \min \left\{ 2c \sqrt{SAM}, \sum_{(x,a) \in \MS \times \MA} \frac{4c^2}{\theta(x,a)}\right\}\nonumber.
  \end{align}
\end{lemma}
\begin{proof}
  For $(x,a) \in \MS \times \MA$, let $\MW_{x,a} := \{ k \in \MW : (x_h^k, a_h^k) = (x,a) \}$. Then, on the one hand, we have
  \begin{align}
    \sum_{k \in \MW} \clip{\gamma_{n_h^k}}{\theta(x_h^k, a_h^k)} =& \sum_{(x,a) \in \MS \times \MA} \sum_{k \in \MW_{x,a}} \clip{\frac{c}{\sqrt{n_h^k}}}{\theta(x,a)} \nonumber\\
    \leq & \sum_{(x,a) \in \MS \times \MA} \frac{4c^2}{\theta(x,a)}\nonumber,
  \end{align}
  where the inequality uses Lemma \ref{lem:clip-sum}. On the other hand, we have
  \begin{align}
    \sum_{k \in \MW} \clip{\gamma_{n_h^k}}{\theta(x_h^k, a_h^k)} =& \sum_{(x,a) \in \MS \times \MA} \sum_{k \in \MW_{x,a}} \clip{\frac{c}{\sqrt{n_h^k}}}{\theta(x,a)} \nonumber\\
    \leq & \sum_{(x,a) \in \MS \times \MA} \sum_{i=1}^{|\MW_{x,a}|} \frac{c}{\sqrt{i}}\nonumber\\
    \leq & \sum_{(x,a) \in \MS \times \MA} 2c \sqrt{|\MW_{x,a}|}\nonumber\\
    \leq & 2c \sqrt{SAM}\nonumber,
  \end{align}
  where the final inequality follows since $\sum_{(x,a) \in \MS \times \MA} |\MW_{x,a}| = M$.
\end{proof}

\subsection{Bounding the clipped range functions}
For all $(h,k) \in [H] \times [K]$ so that $x_h^k \not \in \MG_h^k$ (so that $\RanV_h^k(x_h^k)$ is defined), write
\begin{align}
\clipdelta_h^k := \clipV_h^k(x_h^k) \nonumber.
\end{align}
Since $a_h^k \in A_h^k(x_h^k)$, $\clipQ_h^k(x_h^k, a_h^k)$ is defined, and we may thus further write
\begin{align}
\cliptheta_h^k :=  \clipQ_h^k(x_h^k, a_h^k)\nonumber.
\end{align}
Per the regret decomposition in Lemma \ref{lem:regret-decomposition}, we will bound the regret $\E \left[\sum_{k=1}^K (V_1^\st(x_1) - V_1^{\pi^k}(x_1))\right]$ by the quantity $\sum_{(k,h) : a_h^k \not \in \Aopt_{h,0}(x_h^k)} \clipdelta_h^k$ (conditioned on the high-probability event $\MEwc$). In this section we prove an upper bound on this latter quantity. In fact, we prove a more general result which upper bounds $\sum_{(k,h) \in \MW} \clipdelta_h^k$ for various sets $\MW \subset [K] \times [H]$; we will need this more general result in order to establish improved regret bounds in the case when the predictions $\til Q_h$ are accurate.

We begin by defining the type of set $\MW \subset [K] \times [H]$ for which we obtain such an upper bound, namely \emph{level-$h$ sets}. %
\begin{defn}
  \label{def:levelh}
Fix $h \in [H]$. We say that a subset $\MW \subset [K] \times [H]$ is a \emph{level-$h$} set if the following conditions hold:
\begin{enumerate}
\item For each $k \in [K]$, there is at most one element $(\widetilde k, \widetilde h) \in \MW$ so that $\widetilde k = k$.\label{it:levelh-1}
\item For all $(\widetilde k,\widetilde h) \in \MW$, it holds that both $\tau_{\widetilde h}^{\widetilde k} = 1$ and $\widetilde h \geq h$.\label{it:levelh-2}
\item For each $(\widetilde k, \widetilde h) \in \MW$ for which $\td h > h$, for $h \leq h' < \td h$, it holds that $x_{h'}^{\td k}\in \MG_{h'}^{\td k}$. \label{it:levelh-3}
  \end{enumerate}
\end{defn}

For a level-$h$ set, $\MW$, we next define its \emph{reduction}, which replaces each element $(\td k, \td h) \in \MW$ with $\td h = h$ with another element $(k', h)$, where $k' \leq \td k$ is as small as possible subject to $(x_h^{\td k}, a_h^{\td k}) = (x_h^{k'}, a_h^{k'})$, and to the constraint that all elements in $\MW$ are distinct. The reason that we will want to perform this operation is that our bounds on confidence intervals (in particular, (\ref{eq:qo-qu-gap}), which manifests in Definitions \ref{def:range-function} and \ref{def:half-clipped}) are given in terms of the \emph{first} $n$ times a particular state-action pair $(x,a,h)$ is visited. The reduction $\MR_h(\MW)$ has the property that for any $(x,a)$, if there are $m$ elements $(\td k, h) \in \MR_h(\MW)$ with $(x_h^{\td k}, a_h^{\td k}) = (x,a)$, then those values $\td k$ represent the \emph{first} $m$ episodes at which $(x,a,h)$ is visited. 
\begin{defn}
  \label{def:reduction}
  Fix $h \in [H]$, and consider a level-$h$ set $\MW \subset [K] \times [H]$. The \emph{level-$h$ reduction} of $\MW$, denoted $\MR_h(\MW)$, is defined as follows: starting with $\MW$, perform the following procedure:
  \begin{itemize}
  \item For each $(x,a) \in \MS \times \MA$, let $\MS(x,a)$ denote the set of elements $(\td k, h) \in \MW$ for which $(x_h^{\td k}, a_h^{\td k}) = (x,a)$. Remove the elements of $\MS(x,a)$ from $\MW$, and insert the elements
    \begin{align}
      (k_h^1(x,a), h), (k_h^2(x,a), h), \ldots, (k_h^{|\MS(x,a)|}(x,a), h)
    \end{align}
    into $\MW$. (Recall that, for any $s > 0$, $k_h^1(x,a), \ldots, k_h^s(x,a)$ are the smallest $s$ positive integers $\td k$ so that $(x_h^{\td k}, a_h^{\td k}) = (x,a)$.)
\end{itemize}
\end{defn}
Note that the level-$h$ reduction satisfies the following inequality:
\begin{align}
\max_{(\td k, \td h) \in \MW} \{ \td k \} \geq \max_{(\td k, \td h) \in \MR_h(\MW)} \{ \td k \}\label{eq:reduction-k}.
\end{align}

The following lemma shows that the level-$h$ reduction of a level-$h$ set is a level-$h$ set.
\begin{lemma}
  \label{lem:reduction-levelh}
Suppose that $\MW \subset [K] \times [H]$ is a level-$h$ set for some $h \in [H]$. Then the level-$h$ reduction $\MR_h(\MW)$ is also a level-$h$ set.
\end{lemma}
\begin{proof}
  We first verify that $\MR_h(\MW)$ satisfies property \ref{it:levelh-1} of Definition \ref{def:levelh}. Using the notation of Definition \ref{def:reduction}, we must check that, for each $(x,a) \in \MS \times \MA$, for $1 \leq i \leq |\MS(x,a)|$, there is no $\td h > h$ so that $(k_h^i(x,a), \td h) \in \MW$. However, if this were the case for some $i$ and $\td h$, since $\MW$ is a level-$h$ set, item \ref{it:levelh-3} of Definition \ref{def:levelh} gives us that $x = x_h^{k_h^i(x,a)} \in \MG_h^{k_h^i(x,a)}$. But then we must have $\tau_h^{k_h^i(x,a)} = 0$, which contradicts the fact that for some $\td k \geq k_h^i(x,a)$ so that $x_h^{\td k} = x$, $\tau_h^{\td k} = 1$ and Lemma \ref{lem:misc-monotonicity}.

  That the conditions of item \ref{it:levelh-2} of Definition \ref{def:levelh} hold for $\MR_h(\MW)$ follows directly from Lemma \ref{lem:misc-monotonicity}, and $\MR_h(\MW)$ satisfies the conditions of item \ref{it:levelh-3} since $\MW$ does.
\end{proof}

We are now ready to state and prove Lemma \ref{lem:w-bound-gap}, which is the main technical component of the worst-case (i.e., robustness) regret bounds in item \ref{it:worstcase-deltaincr-full} of Theorem \ref{thm:main-deltaincr} and item \ref{it:worstcase-deltaconst-full} of Theorem \ref{thm:main-deltaconst}. The first part (item \ref{it:w-sum-gap}) of Lemma \ref{lem:w-bound-gap} bounds $\sum_{(k, \td h)} \clipdelta_{\td h}^k$ for any level-$h$ set $\MW$ (for any $h \in [H]$), via a quantity (denoted by $f$ below) that depends on $|\MW|$, the step index $h$, and the largest episode number in $\MW$. The second part (item \ref{it:w-sum-all}) of the lemma then extends this upper bound to a somewhat more general family of subsets $\MW \subset [K] \times [H]$. %
\begin{lemma}
  \label{lem:w-bound-gap}
 For all $h \in [H]$, the following statements hold:
  \begin{enumerate}
  \item \label{it:w-sum-gap} For any level-$h$ set $\MW \subset [K] \times [H]$ and any $k^\st$ satisfying $k^\st \geq k$ for all $(k,\td h)\in \MW$, it holds that $\sum_{(k,\td h) \in \MW} \clipdelta_{\td h}^k \leq f(|\MW|, h,k^\st)$, where for $M \in \BN,\ h \in [H]$,
    \begin{align}
      f(M,h,k^\st) :=& M \cdot \left(1 + \frac 1H \right)^2 \cdot \ggapfn{h+1}{\wh \Delta^{k^\st}} \nonumber\\
      & + \sum_{h'=h}^H \left(1 + \frac 1H \right)^{2(H-h)} \left(SAH + \min \left\{ C_2 \sqrt{H^3 SAM\iota}, \sum_{(x,a) \in \MS \times \MA} \frac{C_2^2 H^3 \iota}{\gapfinal{h'}}\right\} \right) \nonumber\\
      \leq & M \cdot \ggapfn{h}{\wh \Delta^{k^\st}} + e^2 SAH^2 + \min \left\{ e^2 C_2 \sqrt{H^5 SA M\iota}, \sum_{(x,a,h') \in \MS \times \MA \times [H]} \frac{e^2 C_2^2 H^3 \iota}{\gapfinal{h'}} \right\} \label{eq:f-gap-based}.
    \end{align}
    for the constant $C_2 = 8C_0$.
    \item \label{it:w-sum-all} Fix any set $\MW \subset [K] \times [H]$ (not necessarily a level-$h$ set), so that for all $(k,\td h) \in \MW$, $\td h = h$ and $x_h^{k} \not \in \MG_h^k$. For any $k^\st$ so that $k^\st \geq k$ for all $(k,h) \in \MW$, it holds that $\sum_{(k,\td h) \in \MW} \clipdelta_{\td h}^k \leq f(|\MW|, \td h, k^\st)$. 
  \end{enumerate}
\end{lemma}
\begin{proof}
  In the proof of the lemma we will often use the following fact: for all $(k,h) \in [K] \times [H]$ for which $\tau_h^k = 1$, by Definition \ref{def:half-clipped} and the choice of $a_h^k$, it holds that %
  \begin{align}
    \clipdelta_h^k = \clipV_h^k(x_h^k) \leq  \clipQ_h^k(x_h^k, a_h^k) = \cliptheta_h^k\label{eq:theta-delta-compare}. 
\end{align}
  
  We will use reverse induction on $h$ to prove the statement of the lemma. The base case $h = H+1$ is immediate from the convention that $\clipdelta_{H+1}^k = 0$ for all $k \in [K]$. %

  Now we treat the inductive case. Fix $h \leq H$, and suppose that the lemma statement holds for all $h' > h$. For $(x,a) \in \MS \times \MA$, let $\MZ_h(x,a)$ denote the set of all episodes $k \in [K]$ for which $(x_h^k, a_h^k) = (x,a)$ and $\tau_h^k = 1$. For a positive integer $m$, let $\MZ_h^m(x,a)$ denote the set consisting of the $m$ smallest elements of $\MZ_h(x,a)$ (or all of $\MZ_h(x,a)$, if $m > |\MZ_h(x,a)|$).

  Fix any $k,h$ so that $\tau_h^k = 1$, and 
  write $k_h^i := k_h^i(x_h^k, a_h^k)$. Then, by Definition \ref{def:half-clipped} and (\ref{eq:theta-delta-compare}),
  \begin{align}
    \clipdelta_h^k\leq \cliptheta_h^k  = & \alpha_{n_h^k}^0 \cdot H + \sum_{i=1}^{n_h^k} \alpha_{n_h^k}^i \cdot \left(  \RanV_{h'(k_h^i, h)}^{k_h^i} (x_{h'(k_h^i, h)}^{k_h^i})\right) + \clip{\beta_n}{\frac{\delmin}{4H^2}} \label{eq:alpha-decompose-gap}\\
    = & \alpha_{n_h^k}^0 \cdot H + \sum_{i=1}^{n_h^k} \alpha_{n_h^k}^i \cdot \left(  \clipdelta_{h'(k_h^i, h)}^{k_h^i}\right) + \clip{\beta_n}{\frac{\delmin}{4H^2}}
                   \label{eq:vo-vu-delta-gap}.
  \end{align}

  Fix any $m \in \BN$, as well as any $(x,a) \in \MS \times \MA$. As before we abbreviate $k_h^i = k_h^i(x,a)$. We next work towards an upper bound on $\sum_{k \in \MZ_h^m(x,a)} \clipdelta_h^k $, using (\ref{eq:vo-vu-delta-gap}) for each $k \in \MZ_h^m(x,a)$. We first sum the first term of (\ref{eq:vo-vu-delta-gap}) over all $k \in \MZ_h^m(x,a)$:
  \begin{align}
\sum_{k \in \MZ_h^m(x,a)} H \cdot \alpha_{n_h^k}^0 \leq & H \sum_{k \in \MZ_h(x,a)} \One[n_h^k = 0] \leq H\label{eq:zhm-term1-gap},
  \end{align}
  where the first inequality follows since $\alpha_0^0 = 1$ and $\alpha_t^0 = 0$ for $t > 0$, and the second inequality follows since for all $k \in \MZ_h(x,a)$, we have $(x_h^k, a_h^k) = (x,a)$ and there can only be a single episode in $\MZ_h(x,a)$ during which we first visit $(x,a)$. %

  The sum of the second and third terms of (\ref{eq:vo-vu-delta-gap}) may be bounded as follows: if $n_h^k > 0$,
  \begin{align}
    &   \sum_{i=1}^{n_h^k} \alpha_{n_h^k}^i \cdot \clipdelta_{h'(k_h^i,h)}^{k_h^i} + \clip{\beta_{n_h^k}}{\frac{\delmin}{4H^2}}\nonumber\\
    \leq & \clip{ \clip{\beta_{n_h^k}}{\frac{\delmin}{4H^2}}}{\frac{\cliptheta_h^k}{2H}} + \left( 1 + \frac{1}{H} \right) \cdot     \sum_{i=1}^{n_h^k} \alpha_{n_h^k}^i \cdot \clipdelta_{h'(k_h^i,h)}^{k_h^i}\label{eq:do-clipping}\\
    \leq &  \clip{\beta_{n_h^k}}{\max\left\{\frac{\delmin}{4H^2}, \frac{\cliptheta_h^k}{2H}\right\}} + \left( 1 + \frac{1}{H} \right) \cdot     \sum_{i=1}^{n_h^k} \alpha_{n_h^k}^i \cdot \clipdelta_{h'(k_h^i,h)}^{k_h^i}\label{eq:simplify-clipping},
  \end{align}
  where (\ref{eq:do-clipping}) follows from Lemma \ref{lem:clip} and (\ref{eq:vo-vu-delta-gap}) as well as the fact that $\alpha_{n_h^k}^0 = 0$ as $n_h^k > 0$. In the case that $n_h^k = 0$, we have that $\sum_{i=1}^{n_h^k} \alpha_{n_h^k}^i \cdot \clipdelta_{h'(k_h^i,h)}^{k_h^i} + \clip{\beta_{n_h^k}}{\frac{\delmin}{4H}} = 0$ since $\beta_{0} = 0$ by definition (see (\ref{eq:define-beta})). %

  Next, summing (\ref{eq:vo-vu-delta-gap}) over all $k \in \MZ_h^m(x,a)$, and using (\ref{eq:zhm-term1-gap}) and (\ref{eq:simplify-clipping}), we see that
  \begin{align}
    \sum_{k \in \MZ_h^m(x,a)} \clipdelta_h^k \leq & H + \sum_{k \in \MZ_h^m(x,a)}  \left(\clip{\beta_{n_h^k}}{\max\left\{\frac{\delmin}{4H^2}, \frac{\cliptheta_h^k}{2H}\right\} } + \left( 1 + \frac{1}{H} \right) \cdot     \sum_{i=1}^{n_h^k} \alpha_{n_h^k}^i \cdot \clipdelta_{h'(k_h^i,h)}^{k_h^i}\right) \nonumber\\
    \leq & H + \sum_{k \in \MZ_h^m(x,a)} \clip{\beta_{n_h^k}}{ \max\left\{\frac{\delmin}{4H^2}, \frac{\cliptheta_h^k}{2H}\right\}} + \left( 1 + \frac{1}{H} \right) \cdot \sum_{k' \in \MZ_h^m(x,a)} \clipdelta_{h'(k',h)}^{k'} \sum_{t=n_h^{k'}}^\infty \alpha_t^{n_h^{k'}}\label{eq:hprime-exchange-order}\\
    =& H + \sum_{k \in \MZ_h^m(x,a)} \clip{\beta_{n_h^k}}{ \max\left\{\frac{\delmin}{4H^2}, \frac{\cliptheta_h^k}{2H}\right\}} + \left( 1 + \frac{1}{H} \right)^2 \cdot \sum_{k' \in \MZ_h^m(x,a)} \clipdelta_{h'(k',h)}^{k'}\label{eq:hprime-use-alpha}\\
    \leq & H + \sum_{k \in \MZ_h^m(x,a)} \clip{\beta_{n_h^k}}{ \max\left\{\frac{\delmin}{4H^2}, \frac{\cliptheta_h^k}{2H}\right\}}  \nonumber\\
    &+ \left( 1 + \frac 1H \right)^2 \cdot \left( \sum_{k' \in \MZ_h^m(x,a) : \ \tau_{h'(k',h)}^{k'} = 0} \ggapfn{h'(k',h)}{\wh \Delta^{k'}} + \sum_{k' \in \MZ_h^m(x,a):\ \tau_{h'(k',h)}^{k'} = 1} \clipdelta_{h'(k',h)}^{k'}  \right)\label{eq:partition-by-tau},
  \end{align}
  where (\ref{eq:hprime-exchange-order}) follows from exchanging the order of summation, (\ref{eq:hprime-use-alpha}) uses item \ref{it:alpha-sum-sub} of Lemma \ref{lem:alpha}, and (\ref{eq:partition-by-tau}) uses the fact that for all $k' \in \MZ_h^m(x,a)$ for which $\tau_{h'(k',h)}^{k'} = 0$, we have that $\clipdelta_{h'(k',h)}^{k'} \leq \RanV_{h'(k',h)}^{k'}(x_{h'(k',h)}^{k'}) \leq \ggapfn{h'(k',h)}{\wh \Delta^{k'}}$ (by Lemma \ref{lem:clip-lt-range}), since by definition of $h'(k',h)$, either $h'(k',h) = H+1$ or else $x_{h'(k',h)}^{k'} \not \in \MG_{h'(k',h)}^{k'}$. 

  Consider any level-$h$ set $\MW \subset [K] \times [H]$, and consider any $k^\st \geq \max_{(\td k, \td h) \in \MW} \{ \td k \}$. %
  For each $(x,a) \in \MS \times \MA$, let $m(x,a)$ denote the number of elements $(\widetilde k,h) \in \MW$ for which $(x_h^{\widetilde k}, a_h^{\widetilde k}) = (x,a)$. %
  Let $M_1$ be the number of $(\widetilde k, \widetilde h) \in \MW$ so that either $\widetilde h > h$ or  $\widetilde h = h$ and $\tau_{h'(\widetilde k,h)}^{\widetilde k} = 1$, $M_0$ be the  number of $(\widetilde k, \widetilde h) \in \MW$ so that $\widetilde h = h$ and $\tau_{h'(\widetilde k,h)}^{\widetilde k} = 0$, and $M := M_0 + M_1 = |\MW|$. Then
  \begin{align}
   \sum_{(\td k, \td h) \in \MW} \clipdelta_{\td h}^{\td k} \leq & \sum_{(\widetilde k, \widetilde h) \in \MR_h(\MW)} \clipdelta_{\widetilde h}^{\widetilde k} \tag{By Lemma \ref{lem:clip-monotonicity}}\nonumber\\
    = & \sum_{(x,a) \in \MS \times \MA} \sum_{i=1}^{m(x,a)} \clipdelta_h^{k_h^i(x,a)}  + \sum_{(\td k, \td h) \in \MR_h(\MW) : \td h > h} \clipdelta_{\td h}^{\td k}\nonumber\\
    \leq & SAH + \sum_{(x,a) \in \MS \times \MA} \sum_{k \in \MZ_h^{m(x,a)}(x,a)} \clip{\beta_{n_h^k}}{\max\left\{\frac{\delmin}{4H^2}, \frac{\cliptheta_h^k}{2H}\right\} } + \sum_{(\td k, \td h) \in \MR_h(\MW) : \td h > h} \clipdelta_{\td h}^{\td k} \nonumber\\
                                                             & + \left( 1 + \frac{1}{H} \right)^2 \cdot \left( M_0 \cdot \ggapfn{h+1}{\wh \Delta^{k^\st}} +  \sum_{(x,a) \in \MS \times \MA} \ \sum_{k' \in \MZ_h^{m(x,a)}(x,a) : \tau_{h'(k',h)}^{k'} = 1} \clipdelta_{h'(k',h)}^{k'} \right)\tag{By (\ref{eq:partition-by-tau})}\\
    \leq & SAH + \sum_{(x,a) \in \MS \times \MA} \sum_{k \in \MZ_h^{m(x,a)}(x,a)} \clip{\beta_{n_h^k}}{\max\left\{\frac{\delmin}{4H^2}, \frac{\frzQ_h^{k^\st}(x,a)}{2H}\right\} }\nonumber\\
    &+ \left( 1 + \frac 1H \right)^2 \cdot M_0 \cdot \ggapfn{h+1}{\wh \Delta^{k^\st}} + \left( 1 + \frac 1H \right)^2 \cdot \sum_{(\td k, \td h) \in \MW'} \clipdelta_{\td h}^{\td k}\label{eq:wprime-ub-gap},
  \end{align}
  where
  \begin{align}
\MW' := \left\{ (\td k, \td h) \in \MR_h(\MW) : \td h > h \right\} \cup \left\{ (\td k, h'(\td k, h)) : (\td k, h) \in \MR_h(\MW), \ \tau_{h'(\td k, h)}^{\td k} = 1 \right\}\nonumber,
  \end{align}
  so that $|\MW'| = M_1$ and $\max_{(\td k, \td h) \in \MW'} \{ \td k\} \leq k^\st$.  Moreover, in (\ref{eq:wprime-ub-gap}), we have used that by Lemma \ref{lem:clip-monotonicity}, $\cliptheta_h^k = \clipQ_h^k(x_h^k, a_h^k) \geq \frzQ_h^{k^\st}(x_h^k, a_h^k)$ for all $k \in \MZ_h^{m(x,a)}$ (since $\tau_h^k = 1$ for all $k \in \MZ_h^{m(x,a)}(x,a)$).
  We claim that $\MW'$ is a level-$(h+1)$ set.  For any $k \in [K]$, if $(k, \td h) \in \MR_h(\MW)$ for some $\td h > h$, then since $\MR_h(\MW)$ is a level-$h$ set (Lemma \ref{lem:reduction-levelh}), it must hold that $x_h^k \in \MG_h^k$, meaning that $\tau_h^k = 0$, and thus $(k, h) \not\in \MR_h(\MW)$. This verifies that $\MW'$ satisfies condition \ref{it:levelh-1} of Definition \ref{def:levelh}. It is immediate that for all $(\td k, \td h) \in \MW'$, we have $\tau_{\td h}^{\td k} = 1$ and $\td h \geq h+1$ (condition \ref{it:levelh-2}), and condition \ref{it:levelh-3} follows from the corresponding condition for $\MR_h(\MW)$ as well as the fact that for all $(\td k, h) \in \MR_h(\MW)$, for all $h'$ satisfying $h+1 \leq h' < h'(\td k, h)$, we have $x_{h'}^{\td k} \in \MG_{h'}^{\td k}$.

  Thus, we may apply the inductive hypothesis for the set $\MW'$, which gives, together with (\ref{eq:wprime-ub-gap}) and Lemma \ref{lem:clip-sum-gap}, with $\theta(x,a) = \gapfinal{h}$ %
  and the set $\MW$ in Lemma \ref{lem:clip-sum-gap} set to $\{ k : \exists (x,a) \in \MS \times \MA \mbox{ s.t. } k \in \MZ_h^{m(x,a)}(x,a)\}$, 
  \begin{align}
    \sum_{(\td k, \td h) \in \MW} \clipdelta_{\td h}^{\td k} \leq & SAH + \min \left\{ 8C_0 \sqrt{H^3 SAM \iota}, \sum_{(x,a) \in \MS \times \MA} \frac{64 C_0^2 H^3 \iota}{\gapfinal{h}}\right\} \nonumber\\
                                                              & + M_0 \cdot \left( 1 + \frac 1H \right)^2 \cdot \ggapfn{h+1}{\wh \Delta^{k^\st}} + \left( 1 + \frac 1H \right)^2 \cdot f(M_1, h+1,k^\st) \nonumber\\
    \leq & M \cdot \left(1 + \frac 1H \right)^2 \cdot \ggapfn{h+1}{\wh \Delta^{k^\st}} \nonumber\\
                                                                  &+ \sum_{h'=h}^{H} \left( 1 + \frac 1H \right)^{2(H-h)} \cdot \left( SAH + \min \left\{ C_2 \sqrt{H^3 SAM\iota} , \sum_{(x,a) \in \MS \times \MA} \frac{C_2^2 H^3 \iota}{\frzQ_{h'}^{k^\st}(x,a)}\right\}\right)\nonumber\\
    =& f(M, h,k^\st)\nonumber,
  \end{align}
  thus establishing item \ref{it:w-sum-gap} of the lemma.

  Next we establish item \ref{it:w-sum-all} of the lemma. Fix any set $\MW \subset [K] \times [H]$ so that for all $(k, \td h) \in \MW$, $\td h = h$ and $x_h^k \not \in \MG_h^k$. Suppose further that $k^\st$ satisfies $k^\st \geq k$ for all $(k,h) \in \MW$. Thus, for all $(k,h) \in \MW$, either $\tau_h^k = 1$ or $\clipdelta_h^k = \clipV_h^k(x_h^k) \leq \ggapfn{h}{\wh \Delta^k}$. Note also that $\MW' := \{ (k,h) \in \MW : \tau_h^k = 1\}$ is a level-$h$ set. Then, using item \ref{it:w-sum-gap} on the set $\MW'$,
  \begin{align}
    \sum_{(k,h) \in \MW} \clipdelta_h^k \leq & \sum_{(k,h) \in \MW :\ \tau_h^k = 0} \ggapfn{h}{\wh \Delta^k} +  \sum_{(k,h) \in \MW'} \clipdelta_h^{k} \nonumber\\
    \leq & |\MW \backslash \MW'| \cdot \ggapfn{h}{\wh \Delta^{k^\st}} + f(|\MW'|, h, k^\st)\nonumber\\
    \leq & f(|\MW|, h, k^\st)\nonumber,
  \end{align}
  as desired.
\end{proof}

\subsection{Establishing the robustness regret bounds}
In this section we prove a regret decomposition in Lemma \ref{lem:regret-decomposition} and combine it with Lemma \ref{lem:w-bound-gap}, which will suffice for proving the robustness regret bounds in Theorems \ref{thm:main-deltaincr} and \ref{thm:main-deltaconst}. 
Lemma \ref{lem:bound-by-clipdelta} below is needed to prove the regret decomposition bound. It states that the loss incurred by choosing any non-optimal action $a_h^k$ at a state $x_h^k$ may be bounded by the clipped value function $\clipdelta_h^k$; the statement (and proof) is similar to that of in Lemma 4.4 of \cite{xu_fine-grained_2021}.
\begin{lemma}
  \label{lem:bound-by-clipdelta}
  For all $(h,k) \in [H] \times [K]$ for which $x_h^k \not \in \MG_h^k$ and $a_h^k \not \in \Aopt_{h,0}(x_h^k)$, it holds, under the event $\MEwc$, that
  \begin{align}
    V_h^\st(x_h^k) - Q_h^\st(x_h^k, a_h^k)
    \leq 4 \cdot \clipdelta_h^k\nonumber.
  \end{align}
\end{lemma}
\begin{proof}
We assume throughout the proof that the event $\MEwc$ holds (in particular, this allows us to apply Lemma \ref{lem:range-bound}).   Since $a_h^k \in A_h^k(x_h^k)$, we have $\Qo_h^k(x_h^k, a_h^k) \geq \Vu_h^k(x_h^k)$, and so
  \begin{align}
V_h^\st(x_h^k) - Q_h^\st(x_h^k, a_h^k) \leq \Vo_h^k(x_h^k) - \Vu_h^k(x_h^k) + \Qo_h^k(x_h^k, a_h^k) - \Qu_h^k(x_h^k, a_h^k) \nonumber.
  \end{align}
  We may bound $\Vo_h^k(x_h^k) - \Vu_h^k(x_h^k)$ as follows:
  \begin{align}
    & \Vo_h^k(x_h^k) - \Vu_h^k(x_h^k)\nonumber\\
    \leq & \RanV_h^k(x_h^k) \tag{By Lemma \ref{lem:range-bound} and since $x_h^k \not \in \MG_h^k$}\\
    \leq & \clipV_h^k(x_h^k) + \frac{\delmin}{4}=\clipdelta_h^k + \frac{\delmin}{4} \tag{By Lemma \ref{lem:hc-range-bound}}.
  \end{align}
  By Definition \ref{def:half-clipped}, there is some $k' \leq k$ so that $\clipdelta_h^k = \clipV_h^k(x_h^k) = \clipQ_h^{k'}(x_h^{k'}, a^\st)$ for $a^\st = \argmax_{a' \in A_h^{k'}(x_h^k)} \Qo_h^{k'}(x, a') - \Qu_h^{k'}(x,a')$. Now we have
  \begin{align}
    & \Qo_h^k(x_h^k, a_h^k) - \Qu_h^k(x_h^k, a_h^k) \nonumber\\
    \leq & \Qo_h^{k'}(x_h^k, a_h^k) - \Qu_h^{k'}(x_h^k, a_h^k) \tag{By Lemma \ref{lem:qv-monotonicity}}\\
    \leq & \Qo_h^{k'}(x_h^k, a^\st) - \Qu_h^{k'}(x_h^k, a^\st) \tag{Since $a_h^k \in A_h^k(x_h^k) \subseteq A_h^{k'}(x_h^k)$}\\
    \leq & \RanQ_h^{k'}(x_h^k, a^\st) \tag{By Lemma \ref{lem:range-bound} and since $x_h^k \not \in \MG_h^{k'}$}\\
    \leq & \clipQ_h^{k'}(x_h^k, a^\st) + \frac{\delmin}{4} \tag{By Lemma \ref{lem:hc-range-bound}} \\
    = & \clipdelta_h^k + \frac{\delmin}{4}\nonumber.
  \end{align}
  Since $a_h^k \not\in \Aopt_{h,0}(x_h^k)$, we have that $V_h^\st(x_h^k) - Q_h^\st(x_h^k, a_h^k) \geq \delmin$. Thus $\delmin / 2 \leq  (V_h^\st(x_h^k) - Q_h^\st(x_h^k, a_h^k) )/2$, meaning that $V_h^\st(x_h^k) - Q_h^\st(x_h^k, a_h^k) \leq 4 \cdot \clipdelta_h^k$, as desired.
  
\end{proof}

Next we state and prove the regret decomposition bound which is used to bound the worst-case regret.
  \begin{lemma}[Regret decomposition for worst-case bound]
    \label{lem:regret-decomposition}
    For the choice $p = 1/(H^2 K)$, the regret of \algname may be bounded as follows:
    \begin{align}
\sum_{k=1}^K \E \left[ V_1^\st(x_1^k) - V_1^{\pi^k}(x_1^k)\right] \leq 1 + 4 \cdot \E \left[ \left. \sum_{(k,h) : a_h^k \not \in \Aopt_{h,0}(x_h^k)} \clipdelta_h^k \right| \MEwc \right] \nonumber.
    \end{align}
  \end{lemma}
  \begin{proof}
  Note that
  \begin{align}
    & \sum_{k=1}^K \E \left[(V_1^\st - V_1^{\pi^k})(x_1^k)\right] \nonumber\\
    = & \sum_{k=1}^K \E_{\pi^k} \left[ \sum_{h=1}^H V_h^\st(x_h^k) - Q_h^\st(x_h^k ,a_h^k) \right]\nonumber\\
    \leq & \sum_{k=1}^K \E_{\pi^k} \left[ \sum_{h=1}^H \One[a_h^k \not \in \Aopt_{h,0}(x_h^k)] \cdot (V_h^\st(x_h^k) - Q_h^\st(x_h^k, a_h^k)) \right]\nonumber\\
    \leq & \sum_{k=1}^K \E_{\pi^k} \left[ \sum_{h=1}^H \One[a_h^k \not \in \Aopt_{h,0}(x_h^k)] \cdot (V_h^\st(x_h^k) - Q_h^\st(x_h^k, a_h^k))  \ |\ \MEwc  \right] + KH^2 \cdot \Pr[\overline{\MEwc }]\nonumber.
  \end{align}
Note that $\Pr[\overline{\MEwc }] \leq p$, which may be bounded above by $1/(H^2K)$ if we choose $p = 1/(H^2K)$. 
  
  Now let us condition on the event $\MEwc $. Since $a_h^k \in A_h^k(x_h^k)$ for all $h,k$, and in the event that $|A_h^k(x_h^k)| = 1$ it must be the case that $A_h^k(x_h^k)$ contains the optimal action at $x_h^k$ (Lemma \ref{lem:qo-qu-gap}, item \ref{it:correct-elimination}) under the event $\MEwc$, $a_h^k \not \in \Aopt_{h,0}(x_h^k)$ implies that $x_h^k \not \in \MG_h^k$ under $\MEwc$. Thus, conditioned on $\MEwc$, using Lemma \ref{lem:bound-by-clipdelta}, we have that
  \begin{align}
    \sum_{k=1}^K \sum_{h=1}^H \One[a_h^k \not \in \Aopt_{h,0}(x_h^k)] \cdot (V_h^\st(x_h^k) - Q_h^\st(x_h^k, a_h^k))
    \leq  4 \cdot \sum_{(k,h) : a_h^k \not\in \Aopt_{h,0}(x_h^k)} \clipdelta_h^k. \nonumber%
  \end{align}
  This completes the proof of the lemma.
\end{proof}

 To combine Lemma \ref{lem:w-bound-gap} with the regret decomposition result of Lemma \ref{lem:regret-decomposition}, we need a way of upper bounding the left-hand side of (\ref{eq:f-gap-based}) from Lemma \ref{lem:w-bound-gap}, which looks much like the gap-based bound quantity in (\ref{eq:lambda-complexity}) used in Theorems \ref{thm:main-deltaincr} and \ref{thm:main-deltaconst}, but with the actual gaps $\Delta_h(x,a)$ replaced by the proxies $\frzQ_h^{k^\st}(x,a) \geq \frzQ_h^K(x,a)$. Lemma \ref{lem:frzq-gap} below shows that the proxies $\frzQ_h^{K}(x,a)$ are indeed upper bounds on the true gaps $\Delta_h(x,a)$. 
  \begin{lemma}
    \label{lem:frzq-gap}
    Consider any $(x,a,h) \in \MS \times \MA \times [H]$, and suppose the event $\MEwc$ holds. Then
    \begin{align}
\max\left\{ \frac{\frzQ_h^{K}(x,a)}{2H}, \frac{\delmin}{4H^2} \right\} \geq \max \left\{ \frac{\Delta_h(x,a)}{8H}, \One[\Aopt_{h,0}(x) = \{ a \}] \cdot \frac{\delmins{h}(x)}{8H}, \frac{\delmin}{4H^2} \right\}\label{eq:frzq-delmin-lb}.
    \end{align}
    Further, for any $\tildm \leq \delmin$, recalling the definition of $\tfrzQ_h^k$ in step \ref{it:tfrzq-define} of Algorithm \ref{alg:delta-incr}, it holds that 
    \begin{align}
\max\left\{ \frac{\tfrzQ_h^{K}(x,a)}{2H}, \frac{\tildm}{4H^2} \right\} \geq \max \left\{ \frac{\Delta_h(x,a)}{8H}, \One[\Aopt_{h,0}(x) = \{ a \}] \cdot \frac{\delmins{h}(x)}{8H}, \frac{\tildm}{4H^2} \right\}\label{eq:tfrzq-delmin-lb}.
    \end{align}
  \end{lemma}
  \begin{proof}
    Suppose the event $\MEwc$ holds (this allows us to apply Lemmas \ref{lem:range-bound} and \ref{lem:bound-by-clipdelta}). 
    By definition, there is some $k \in [K]$ so that $\frzQ_h^K(x,a) = \clipQ_h^k(x,a)$ and either $(x_h^k, a_h^k) = (x,a)$ and $\tau_h^k = 1$ or else $k = 1$. In the case $k = 1$, we have $\frzQ_h^K(x,a) = H \geq \Delta_h(x,a)$.  Otherwise, we consider two cases:
  \begin{itemize}
  \item Suppose $a \not \in \Aopt_{h,0}(x)$.  Then
    \begin{align}
\frzQ_h^K(x,a) = \clipQ_h^k(x_h^k,a_h^k) \geq  \clipdelta_h^k \geq \frac{1}{4} \cdot \Delta_h(x_h^k, a_h^k) = \frac{1}{4} \cdot \Delta_h(x,a)\nonumber,
    \end{align}
    where the first inequality follows from (\ref{eq:theta-delta-compare}) and the second inequality follows from Lemma \ref{lem:bound-by-clipdelta}.
  \item Suppose that $a$ is the unique action in $\Aopt_{h,0}(x)$, i.e., that $\Aopt_{h,0}(x) = \{ a \}$. Since $\tau_h^k = 1$, we have that $x \not \in \MG_h^k$, meaning that there is some sub-optimal action remaining in $A_h^k(x)$, which we denote by $a'$. Then
    \begin{align}
      \Qo_h^k(x_h^k, a') - \Qu_h^k(x_h^k, a') \leq & \Qo_h^k(x_h^k, a_h^k) - \Qu_h^k(x_h^k, a_h^k) \tag{Since $a_h^k$ maximizes the confidence interval}\\
      \leq & \RanQ_h^k(x_h^k, a_h^k) \tag{By Lemma \ref{lem:range-bound}}\\
      \leq & \clipQ_h^k(x_h^k, a_h^k) + \frac{\delmin}{4} \tag{By Lemma \ref{lem:hc-range-bound}}.
    \end{align}
    Moreover, as in the proof of Lemma \ref{lem:bound-by-clipdelta}, we have, by Lemmas \ref{lem:range-bound} and \ref{lem:hc-range-bound} as well as (\ref{eq:theta-delta-compare}),
    \begin{align}
\Vo_h^k(x_h^k) - \Vu_h^k(x_h^k) \leq \RanV_h^k(x_h^k) \leq \clipV_h^k(x_h^k) + \frac{\delmin}{4} \leq \clipQ_h^k(x_h^k, a_h^k) + \frac{\delmin}{4}.\nonumber
    \end{align}
    Combining the above displays, we obtain
    \begin{align}
      \delmins{h}(x) \leq & \Delta_h(x,a') \nonumber\\
      \leq & (\Vo_h^k(x_h^k) - \Vu_h^k(x_h^k))  + (\Qo_h^k(x_h^k, a') - \Qu_h^k(x_h^k, a'))\nonumber\\
      \leq & 2 \cdot \clipQ_h^k(x_h^k, a_h^k) + \frac{\delmin}{2}\nonumber,
    \end{align}
    which implies that $\frzQ_h^K(x,a) = \clipQ_h^k(x_h^k, a_h^k) \geq \frac{\delmins{h}(x)}{4}$. 
  \end{itemize}
  The above two cases imply that $\frac{\frzQ_h^K(x,a)}{2H} \geq \max \left\{ \frac{\Delta_h(x,a)}{8H}, \One[\Aopt_{h,0}(x) = \{ a \}] \cdot \frac{\delmins{h}(x)}{8H} \right\}$. The first inequality, (\ref{eq:frzq-delmin-lb}), follows immediately.

  To establish the second inequality, (\ref{eq:tfrzq-delmin-lb}), of the lemma, we simply note that all arguments of this lemma (including Lemmas \ref{lem:hc-range-bound} and \ref{lem:bound-by-clipdelta}) go through without modification if $\delmin$ is replaced with any lower bound $\tildm$ of $\delmin$ in the definitions of $\clipV_h^k, \clipQ_h^k, \frzQ_h^k$. 
\end{proof}

The following lemma presents the worst-case regret bound for \algname with the sub-procedure \deltaconst used to choose $\wh \Delta^k$.
\begin{lemma}
  \label{lem:worst-case-const}
  Suppose $T \geq SAH^3$. 
  When given as input any prediction function $\til Q$, the regret of \algname (with \deltaconst and input parameter $\CR \geq \max \{ SAH^3, \CC_{M,T,1}\}$) satisfies:
  \begin{align}
\E \left[\sum_{k=1}^K (V_1^\st - V_1^{\pi^k})(x_1^k) \right]\leq O (\CR)\nonumber.
    \end{align}
  \end{lemma}
  \begin{proof}
 We first note that the regret decomposition of Lemma \ref{lem:regret-decomposition} gives
  \begin{align}
\sum_{k=1}^K \E \left[ V_1^\st(x_1^k) - V_1^{\pi^k}(x_1^k)\right] \leq & 1 + 4 \cdot \E \left[ \left. \sum_{(k,h): a_h^k \not \in \Aopt_{h,0}(x_h^k)} \clipdelta_h^k \right| \MEwc \right]\nonumber.
  \end{align}
  Recall that under the event $\MEwc$, $a_h^k \not \in \Aopt_{h,0}(x_h^k)$ implies that $x_h^k \not \in \MG_h^k$. 
  Thus, conditioned on $\MEwc$, we may bound $4\sum_{(k,h): a_h^k \not \in \Aopt_{h,0}(x_h^k)} \clipdelta_h^k$ as follows:
  \begin{align}
    & 4 \cdot \sum_{(k,h): a_h^k \not \in \Aopt_{h,0}(x_h^k)} \clipdelta_h^k\nonumber\\
    \leq & 4H \cdot \left( K \cdot \ggapfn{1}{\wh \Delta^K} + e^2 SAH^2 + \min \left\{ e^2 C_2 \sqrt{H^5 SA K\iota}, \sum_{(x,a,h) \in \MS \times \MA \times [H]} \frac{e^2 C_2^2 H^3 \iota}{\max \left\{ \frac{\frzQ_h^K(x,a)}{2H}, \frac{\delmin}{4H^2} \right\}} \right\}\right)\tag{Using item \ref{it:w-sum-all} of Lemma \ref{lem:w-bound-gap}}\\
    \leq & O(\CR) + O(SAH^3) + O \left(\min \left\{ \sqrt{H^7 SAK\iota}, H^6 \iota\cdot \left(\sum_{(x,a,h) : a \not \in \Aopt_{h,0}(x)} \frac{1}{\Delta_h(x,a)} + \frac{|\Amul|}{\delmin} \right)\right\} \right)\tag{By the definition of $\wh \Delta^K$ in \deltaconst, Algorithm \ref{alg:delta-const}, and Lemma \ref{lem:frzq-gap}}\\
    \leq & O(\CR) +  O\left(\CC_{M,T,1}\right) \leq O(\CR)\nonumber,
  \end{align}
  where the second-to-last inequality follows from the fact that $\CR \geq SAH^3$. %
  \end{proof}

The following lemma presents the worst-case regret bound for \algname with the sub-procedure \deltaincr used to choose $\wh \Delta^k$. 
\begin{lemma}
  \label{lem:worst-case-incr}
Suppose $T \geq SAH^3$.  When given as input any prediction function $\til Q$, the regret of the \algname with input parameter $\lambda$ (used with \deltaincr and input parameter $\tildm \leq \delmin$) satisfies:
  \begin{align}
\E \left[\sum_{k=1}^K (V_1^\st - V_1^{\pi^k})(x_1^k) \right]\leq O \left(\min \left\{ \sqrt{\frac{SAH^9 T \iota^2}{\lambda}}, \frac{H^8\iota^2}{\lambda} \cdot \left( \sum_{(x,a,h) : a \not \in \Aopt_{h,0}(x)} \frac{1}{\Delta_h(x,a)} + \frac{|\Amul|}{\tildm}\right) \right\} \right)\nonumber.
  \end{align}
\end{lemma}
\begin{proof}
  We first note that the regret decomposition of Lemma \ref{lem:regret-decomposition} gives
  \begin{align}
\sum_{k=1}^K \E \left[ V_1^\st(x_1^k) - V_1^{\pi^k}(x_1^k)\right] \leq & 1 + 4 \cdot \E \left[ \left. \sum_{(k,h): a_h^k \not \in \Aopt_{h,0}(x_h^k)} \clipdelta_h^k \right| \MEwc \right]\nonumber.
  \end{align}
  The guarantee that $\tildm \leq \delmin$ gives that $\tfrzQ_h^K(x,a) \leq \frzQ_h^K(x,a)$ for all $x,a,h$.  
  Recall that under the event $\MEwc$, $a_h^k \not \in \Aopt_{h,0}(x_h^k)$ implies that $x_h^k \not \in \MG_h^k$. 
  Thus, conditioned on $\MEwc$, we may bound $4\sum_{(k,h): a_h^k \not \in \Aopt_{h,0}(x_h^k)} \clipdelta_h^k$ as follows:
  \begin{align}
    & 4 \cdot \sum_{(k,h): a_h^k \not \in \Aopt_{h,0}(x_h^k)} \clipdelta_h^k\nonumber\\
   \leq & 4H \cdot \left( K \cdot \ggapfn{1}{\wh \Delta^K} + e^2 SAH^2 + \min \left\{ e^2 C_2 \sqrt{H^5 SA K\iota}, \sum_{(x,a,h) \in \MS \times \MA \times [H]} \frac{e^2 C_2^2 H^3 \iota}{\max \left\{ \frac{\frzQ_h^K(x,a)}{2H}, \frac{\delmin}{4H^2} \right\}} \right\}\right)\tag{Using item \ref{it:w-sum-all} of Lemma \ref{lem:w-bound-gap}}\\
    \leq & O \left(SAH^3 +  KH \cdot \min  \left\{\frac{H^5 \iota^2}{\lambda K} \cdot \sum_{(x,a,h)} \frac{1}{\max\left\{ \frac{\tfrzQ_h^{K}(x,a)}{2H}, \frac{\tildm}{4H^2} \right\} }, \sqrt{\frac{SAH^8\iota^2}{\lambda K}} \right\}  \right)\tag{By the definition of $\wh \Delta^K$ in \deltaincr, Algorithm \ref{alg:delta-incr} and $\tfrzQ_h^K \leq \frzQ_h^k$}.
  \end{align}

By Lemma \ref{lem:frzq-gap} (in particular, (\ref{eq:tfrzq-delmin-lb})), we conclude that
\begin{align}
& \E \left[ \sum_{k=1}^K (V_1^\st - V_1^{\pi^k})(x_1^k)\right] \nonumber\\
  \leq & O \left( SAH^3 + \min \left\{ \sqrt{\frac{SAH^{10} K \iota^2}{\lambda}}, \frac{H^8 \iota^2}{\lambda} \cdot \left(\sum_{(x,a,h) : a \not \in \Aopt_{h,0}(x)} \frac{1}{\Delta_h(x,a)} \right.\right.\right.\nonumber\\
  & \left.\left.\left. + \sum_{(x,h) : |\Aopt_{h,0}(x)| = 1} \frac{1}{\delmins{h}(x)} + \sum_{(x,a,h) \in \Amul} \frac{1}{\tildm} \right)\right\} \right)\label{eq:regdec-3terms}\\
  \leq & O \left( SAH^3 + \min \left\{ \sqrt{\frac{SAH^9 T \iota^2}{\lambda}}, \frac{H^8\iota^2}{\lambda} \cdot \left( \sum_{(x,a,h) : a \not \in \Aopt_{h,0}(x)} \frac{1}{\Delta_h(x,a)} + \frac{|\Amul|}{\tildm}\right) \right\} \right)\nonumber.
\end{align}
    Finally, the fact that $T \geq SAH^3$ implies that the term $SAH^3$ in the above expression is dominated by the second term (see also Lemma \ref{lem:cc-lb}).
\end{proof}
We remark that in the proof of Lemma \ref{lem:worst-case-incr}, if $\delmin = |\Amul| = 0$, then the term $\frac{|\Amul|}{\delmin}$ can be interepreted as 0. This follows from the fact that in the inequality (\ref{eq:regdec-3terms}), the summation in the third term $\sum_{(x,a,h) \in \Amul} \frac{1}{\delmin}$ is over an empty set. 

\section{Proofs for approximate distillation bound}
\label{sec:proofs-distillation}
In this section we establish the upper bounds in item \ref{it:distillation-deltaincr} of Theorem \ref{thm:main-deltaincr} and item \ref{it:distillation-deltaconst} of Theorem \ref{thm:main-deltaconst}, which give a regret bound for \algname when the predictions $\til Q_h$ are an $\ep$-approximate distillation of $Q_h^\st$.

\subsection{Bounding the number of exploration episodes}
\label{sec:bound-exploration}
A key challenge in establishing these bounds is to show that \algname does not spend too many episodes ignoring the predictions $\til Q_h$ as part of the exploration phase. 
To this end, we bound the number of episodes $k$ for which $\sigma_h^k = 1$ (for each $h \in [H]$). Note that this is not exactly the same as the number of episodes $k$ for which $\tau_h^k = 1$, and that it is the parameters $\tau_h^k$ (not $\sigma_h^k$) which correspond to whether the policy $\pi_h^k$ (defined in (\ref{eq:explore-exploit})) engages in exploration or constrained exploitation. We will show, however (in Claim \ref{clm:sig-tilv}), that those episodes $k$ for which $\sigma_h^k = 0$ but $\tau_h^k = 1$ only contribute a small amount to the overall regret; this is in turn a consequence of Lemma \ref{lem:clipv-delmin-lb}, which shows that if there is a non-optimal action in $A_h^k(x)$, then $\RanV_h^k(x)$ (which is used to define $\tau_h^k$) and $\clipQ_h^k(x)$ (which is used to define $\sigma_h^k$) must be close.

Recall the definition of $\wh \lambda$ in Theorem \ref{thm:main-deltaconst}. Lemma \ref{lem:bound-sigma} treats the case where \deltaconst is used to choose $\wh \Delta^k$; it bounds, for each $h \in [H]$, the number of episodes $k$ for which $\sigma_h^k = 1$, as a function of $\wh \lambda$. The main tool in the proof is Lemma \ref{lem:w-bound-gap}, which is used to show that the parameters $\clipdelta_h^k = \clipV_h^k(x_h^k)$ decrease sufficiently fast to $\clipV_h^k(x_h^k) \leq \frac{1}{1+\frac 1H} \cdot \ggapfn{h}{\wh \Delta^k}$, i.e., $\sigma_h^k = 0$, for most episodes $k$.
  \begin{lemma}
    \label{lem:bound-sigma}
Suppose \algname is run with \deltaconst to choose the values $\wh \Delta^k$. Then for all $h \in [H]$, the number of episodes $k \in [K]$ for which $\sigma_h^k = 1$ is at most $\max\{ SAH^3, \wh \lambda \cdot K\}$. 
  \end{lemma}
  \begin{proof}
Per $\deltaconst$, we have that $\wh \Delta^k = \CR / (KH)$ for all $k \in [K]$. Therefore, throughout the proof of this lemma we will drop the superscript $k$ and write $\wh \Delta := \wh \Delta^k$ (which holds for all $k \in [K]$).
    
    For any $(h,k) \in [H] \times [K]$, note that $\sigma_h^k = 1$ implies that $\ggapfn{h}{\wh \Delta} < (1+1/H) \cdot \clipV_h^k(x_h^k) = (1 +1/H) \cdot \clipdelta_h^k$. %
    Write $\MY_h := \{ k : \sigma_h^k = 1 \}$. 
    Then for each $h$, we have that
    \begin{align}
\sum_{k \in \MY_h} \clipdelta_h^k \geq \frac{1}{1+1/H} \cdot \ggapfn{h}{\wh \Delta} \cdot |\MY_h|\nonumber.
    \end{align}
    Using the above inequality and item \ref{it:w-sum-all} of Lemma \ref{lem:w-bound-gap} with the set $\MW = \{ (k,h) : \sigma_h^k = 1 \}$ (which satisfies the requirement that each $(k,h) \in \MW$ satisfies $x_h^k \not \in \MG_h^k$), we get that
    \begin{align}
      & \frac{\ggapfn{h}{\wh \Delta} \cdot |\MY_h|}{1 + 1/H} \nonumber\\
      \leq & \sum_{k \in \MY_h} \clipdelta_h^k \nonumber\\
      \leq & |\MY_h| \cdot \left( 1 + 1/H\right)^2 \cdot \ggapfn{h+1}{\wh \Delta} + e^2 SAH^2 + \min \left\{ e^2 C_2 \sqrt{H^5 SA |\MY_h| \iota}, \sum_{(x,a,h') \in \MS \times \MA \times [H]} \frac{e^2 C_2^2 H^3 \iota}{\max \left\{ \frac{\frzQ_{h'}^K(x,a)}{2H}, \frac{\delmin}{4H^2} \right\} }\right\} \nonumber.
    \end{align}
    Rearranging and using the fact that $\ggapfn{h}{\frac{\wh \Delta}{1+1/H}} - (1+1/H)^2 \cdot \ggapfn{h+1}{\wh \Delta} \geq \frac{\ggapfn{h+1}{\wh \Delta}}{H}$, we obtain that
    \begin{align}
\frac{|\MY_h| \cdot \ggapfn{h+1}{\wh \Delta}}{H} \leq e^2 SAH^2 +  \min \left\{ e^2 C_2 \sqrt{H^5 SA |\MY_h| \iota}, \sum_{(x,a,h') \in \MS \times \MA \times [H]} \frac{e^2 C_2^2 H^3 \iota}{\max \left\{ \frac{\frzQ_{h'}^K(x,a)}{2H}, \frac{\delmin}{4H^2} \right\} }\right\}\nonumber.
    \end{align}
     By Lemma \ref{lem:frzq-gap}, it follows that 
     \begin{align}
       \frac{|\MY_h| \cdot \ggapfn{h+1}{\wh \Delta}}{H} \leq & e^2 SAH^2 +  \min \left\{ e^2 C_2 \sqrt{H^5 SA |\MY_h| \iota}, 16 e^2 C_2^2 H^5 \iota \cdot \left( \sum_{(x,a,h') \in \MS \times \MA \times [H] : a \not \in \Aopt_{h',0}(x)} \frac{1}{\Delta_h(x,a)} + \frac{|\Amul|}{\delmin} \right)\right\}\nonumber\\
       \leq & \min \left\{2 e^2 C_2 \sqrt{H^5 SA |\MY_h| \iota}, 32 e^2 C_2^2 H^5 \iota \cdot \left( \sum_{(x,a,h') \in \MS \times \MA \times [H] : a \not \in \Aopt_{h',0}(x)} \frac{1}{\Delta_h(x,a)} + \frac{|\Amul|}{\delmin} \right)\right\}\label{eq:yh-size-ub},
     \end{align}
     where the second inequality above follows from the fact that, assuming $|\MY_h| \geq SAH^3$, both terms in the minimum are bounded below by $e^2 SAH^2$. 
     Recall that $\wh \lambda \geq SAH^3/ K$ is defined to be as small as possible so that
     \begin{align}
\CR \geq \min \left\{ \sqrt{\frac{H^9SAK \iota}{\wh \lambda}}, \frac{1}{\wh \lambda} \cdot H^7 \iota \cdot  \left( \sum_{(x,a,h') \in \MS \times \MA \times [H] : a \not \in \Aopt_{h',0}(x)} \frac{1}{\Delta_h(x,a)} + \frac{|\Amul|}{\delmin} \right) \right \},\nonumber
     \end{align}
     and that $\wh \Delta = \CR / (KH)$ (per \deltaconst). We next consider two cases:
     \begin{enumerate}
     \item $\CR \geq \sqrt{\frac{H^9SAK \iota}{\wh \lambda}}$. Then from (\ref{eq:yh-size-ub}),
       \begin{align}
         |\MY_h| \leq  \frac{2e^2 C_2 \sqrt{H^7 SA |\MY_h| \iota}}{\ggapfn{h+1}{\wh \Delta}}
         \leq  \frac{2e^2 C_2}{C_1} \cdot \frac{\sqrt{H^7 SA |\MY_h| \iota}}{\sqrt{H^7 SA \iota/ (\wh \lambda \cdot K)}}\nonumber,
       \end{align}
       which implies that
       \begin{align}
\sqrt{|\MY_h|} \leq \frac{2e^2 C_2}{C_1} \cdot \sqrt{\wh \lambda \cdot K}\nonumber,
       \end{align}
       and in turn we get that $|\MY_h| \leq \wh \lambda \cdot K$ since $C_1$ is chosen so that $2e^2 C_2 \leq C_1$. 
     \item $\CR \geq \frac{1}{\wh \lambda} \cdot H^7 \iota \cdot  \left( \sum_{(x,a,h') \in \MS \times \MA \times [H] : a \not \in \Aopt_{h',0}(x)} \frac{1}{\Delta_h(x,a)} + \frac{|\Amul|}{\delmin} \right)$. Then from (\ref{eq:yh-size-ub}),
       \begin{align}
         |\MY_h| \leq & \frac{32 e^2 C_2^2 H^6 \iota}{\ggapfn{h+1}{\wh \Delta}} \cdot  \left( \sum_{(x,a,h') \in \MS \times \MA \times [H] : a \not \in \Aopt_{h',0}(x)} \frac{1}{\Delta_h(x,a)} + \frac{|\Amul|}{\delmin} \right)\nonumber\\
         \leq & \frac{32 e^2 C_2^2}{C_1} \cdot K \cdot \wh \lambda \leq K \cdot \wh \lambda\nonumber,
       \end{align}
       where the final inequality follows since $C_1$ is chosen so that $32 e^2 C_2^2 \leq C_1$. 
     \end{enumerate}
   \end{proof}

   Lemma \ref{lem:bound-sigma-incr} establishes the same result as Lemma \ref{lem:bound-sigma}, except for the choice of \deltaincr in \algname. The proof is more subtle, though, because of the more complex nature of the parameters $\wh \Delta^k$ in \deltaincr. In particular, to establish Lemma \ref{lem:bound-sigma-incr}, we need to divide the set of episodes into different phases, so that within each phase the value of $\wh \Delta^k$ only changes by a small multiplicative factor. 
        \begin{lemma}
    \label{lem:bound-sigma-incr}
Suppose \algname is run with \deltaincr (Algorithm \ref{alg:delta-incr}) to choose the values $\wh \Delta^k$. Then for all $h \in [H]$, the number of episodes $k \in [K]$ for which $\sigma_h^k = 1$ is at most $\max\{ SAH^3,  \lambda \cdot K\}$. 
  \end{lemma}
  \begin{proof}
    Since $\clipQ_h^1(x,a) = H$ for all $(x,a,h)$, it holds that $\wh \Delta^1 \geq \frac{SAH}{\lambda K}$. Also note that by definition we have $\wh \Delta^k \leq \sqrt{\frac{SAH^8\iota^2}{\lambda K}}$ for all $k$.  Note that $\sqrt{\frac{SAH^8\iota^2}{\lambda K}} \cdot \frac{\lambda K}{SAH} \leq \sqrt{\lambda KH^6\iota^2}$. For $0 \leq i \leq \left\lceil \log_{1 + \frac 1H} (\sqrt{\lambda  KH^6 \iota^2})\right\rceil$, set $\omega_i := 2^i \cdot \frac{SAH}{\lambda K}$. 
  
    For $(h,k) \in [H] \times [K]$, note that $\sigma_h^k = 1$ implies that $\ggapfn{h}{\wh \Delta^k} < (1+1/H) \cdot \clipV_h^k(x_h^k) = (1+1/H) \cdot \clipdelta_h^k$. For each $1 \leq i \leq \lceil \log_{1+1/H}(\sqrt{\lambda KH^6 \iota^2})\rceil$ and $h \in [H]$ set $\MY_h^i := \{ k \in [K]:\ \sigma_h^k = 1,\ \omega_{i-1} \leq \wh \Delta^k \leq \omega_i \}$.

    Then for each $h \in [H]$ and  $0 \leq i \leq \lceil \log_{1+ \frac 1H} (\sqrt{\lambda KH^6\iota^2})\rceil$,
 \begin{align}
   \label{eq:cliptheta-whdelta}
\sum_{k \in \MY_h^i} \clipdelta_h^k \geq |\MY_h^i| \cdot \frac{\ggapfn{h}{\omega_{i-1}}}{1+\frac 1H}.%
 \end{align}

 Set $\MY_h := %
 \bigcup_i \MY_h^i$. 
Fix any $h \in [H]$ and $i$ satisfying $1 \leq i \leq \lceil \log_{1+\frac 1H}(\sqrt{\lambda KH^6 \iota^2})\rceil$. Using (\ref{eq:cliptheta-whdelta}) and the statement of item \ref{it:w-sum-all} of Lemma \ref{lem:w-bound-gap} for $\MW = \{ (k,h) : k \in \MY_h^i\}$, noting that for $k^\st = \max_{k \in \MY_h^i} \{  k \}$, we have $\wh \Delta^{k^\st} \leq \omega_i$, we see that
  \begin{align}
    &   \frac{  |\MY_h^i| \cdot  \ggapfn{h}{\omega_{i-1}}}{1+1/H} \nonumber\\
    \leq & \sum_{k \in \MY_h^i} \clipdelta_h^k \nonumber\\
    \leq & |\MY_h^i| \cdot \left( 1 + \frac 1H \right)^2 \cdot \ggapfn{h+1}{\omega_i} + e^2 SAH^2 + \min \left\{ e^2 C_2 \sqrt{H^5 SA |\MY_h^i| \iota}, \sum_{(x,a,h') \in \MS \times \MA \times [H]} \frac{e^2 C_2^2 H^3 \iota }{\gapfinal{h'}}\right\} \nonumber.
  \end{align}
  Rearranging and using the fact that $\frac{\ggapfn{h}{\omega_{i-1}}}{1+1/H} - \left( 1 + \frac 1H \right)^2 \cdot \ggapfn{h+1}{\omega_{i-1}} \geq \ggapfn{h+1}{\omega_{i-1}}/H$, we obtain that
  \begin{align}
    \frac{|\MY_h^i|}{H} \cdot \ggapfn{h+1}{\omega_{i-1}} \leq & e^2 SAH^2 + \min \left\{ e^2 C_2 \sqrt{H^5 SA |\MY_h^i| \iota}, \sum_{(x,a,h') \in \MS \times \MA \times [H]} \frac{e^2 C_2^2 H^3 \iota }{\gapfinal{h'}}\right\}\nonumber\\
    \leq &  e^2 SAH^2 + \min \left\{ e^2 C_2 \sqrt{H^5 SA |\MY_h^i| \iota}, \sum_{(x,a,h') \in \MS \times \MA \times [H]} \frac{e^2 C_2^2 H^3 \iota }{\tgapfinal{h'}}\right\} \label{eq:zh-h-ub},
  \end{align}
  where the second inequality above follows from $\tildm \leq \delmin$ and therefore $\tfrzQ_{h'}^{k^\st}(x,a) \leq \frzQ_{h'}^{k^\st}(x,a)$ for all $x,a,h'$. 
  We now consider two cases, based on the value of $\wh \Delta^{k^\st}$ (depending on which of the two terms in the minimum in (\ref{eq:define-whdelta-incr}) in the algorithm \deltaincr is smaller):
  \begin{enumerate}
  \item Suppose $\wh \Delta^{k^\st} = \frac{H^5 \iota^2}{\lambda  K} \cdot \sum_{(x,a,h)} \frac{1}{\tgapfinal{h}}$. Note that
    \begin{align}
      \ggapfn{h+1}{\omega_{i-1}} \geq C_1 \cdot \omega_{i-1} \geq C_1/(1+1/H) \cdot \omega_i \geq C_1/(1+1/H) \cdot \wh \Delta^{k^\st} \geq C_1/2 \cdot \wh \Delta^{k^\st},\label{eq:omegai-lb}
    \end{align}
    as well as $\lceil \log_{1+1/H}(\sqrt{\lambda KH^6\iota^2}) \rceil \leq 8 H\iota$. 
    Then using (\ref{eq:zh-h-ub}), we get that
    \begin{align}
      |\MY_h^i| \leq & \frac{1}{\ggapfn{h+1}{\omega_{i-1}}} \cdot \left( e^2 SAH^3 + \sum_{(x,a,h')} \frac{e^2 C_2^2 H^4 \iota}{\tgapfinal{h'}}\right)\nonumber\\
      \leq & \frac{1}{\ggapfn{h+1}{\omega_{i-1}}} \cdot \sum_{(x,a,h')} \frac{2e^2 C_2^2 H^4 \iota}{\tgapfinal{h'}} \label{eq:bound-sah-by-2}\\
      \leq & \frac{2 \cdot 7 H \iota}{C_1 \cdot \wh \Delta^{k^\st} \cdot \lceil \log_{1+1/H}(\sqrt{\lambda KH^6\iota^2}) \rceil} \cdot  \sum_{(x,a,h')} \frac{2e^2 C_2^2 H^4 \iota}{\tgapfinal{h'}}\nonumber\\
      = & \frac{32 e^2 C_2^2 \cdot \lambda K}{C_1 \cdot \lceil \log_{1+1/H}(\sqrt{\lambda KH^6\iota^2}) \rceil} \leq \frac{\lambda K}{\lceil \log_{1+1/H}(\sqrt{\lambda KH^6\iota^2}) \rceil},\label{eq:zh-bound-gaps}
    \end{align}
    where (\ref{eq:bound-sah-by-2}) follows since $\tgapfinal{h'} \leq H$ for all $(x,a,h')$, and (\ref{eq:zh-bound-gaps}) follows since $C_1$ is chosen so that $C_1 \geq 32e^2 C_2^2$ (see (\ref{eq:define-constants})). Therefore, $|\MY_h| \leq \sum_{i=1}^{\lceil \log_{1+1/H}(\sqrt{\lambda KH^6\iota^2}) \rceil} |\MY_h^i| \leq \lambda K$. 
  \item Otherwise, by the definition of $\wh \Delta^{k}$ in (\ref{eq:define-whdelta-incr}), we have $\wh \Delta^{k^\st} = \sqrt{\frac{SAH^8\iota^2}{\lambda K}} = \wh \Delta^K$. %
    Note that (\ref{eq:omegai-lb}) still holds, and so, using (\ref{eq:zh-h-ub}), we get that, for each $i$,
    \begin{align}
      |\MY_h^i| \leq & \frac{1}{\ggapfn{h+1}{\omega_{i-1}}} \cdot \left( e^2 SAH^3 + e^2 C_2 \sqrt{H^7 SA |\MY_h^i| \iota} \right)\nonumber\\
      \leq & \frac{1}{\ggapfn{h+1}{\omega_{i-1}}} \cdot 2e^2 C_2 \sqrt{H^7 SA |\MY_h^i| \iota} \label{eq:bound-sah-by-2-2}\\
      \leq & \frac{\sqrt{8H\iota}}{C_1 \cdot \wh \Delta^{k^\st} \cdot\sqrt{ \lceil \log_{1+1/H} (\sqrt{\lambda KH^7 \iota^2})} \rceil} \cdot 2e^2 C_2 \sqrt{H^7 SA |\MY_h^i| \iota}\nonumber,
    \end{align}
    which implies that
    \begin{align}
\sqrt{|\MY_h^i|} \leq \frac{6e^2 C_2}{C_1} \cdot \frac{1}{\sqrt{\lceil \log_{1+1/H} (\sqrt{\lambda KH^7 \iota^2}) \rceil}} \cdot \sqrt{\lambda K}\nonumber,
    \end{align}
    and since $C_1$ is chosen so that $C_1 \geq 6e^2 C_2$, we get that $|\MY_h| \leq \lambda K$, as desired.
    
  \end{enumerate}
  Thus, in both cases, we obtain that $|\MZ_h| \leq \max\{ SAH^3, \lambda K \}$, completing the proof of the lemma.
\end{proof}

\subsection{Bounding the value functions $\Ro_h^k, \til Q_h^k, \til V_h^k$}
\label{sec:bound-rqv}
In this section we establish some basic bounds on the value functions $\Ro_h^k, \til Q_h^k, \til V_h^k$ maintained by \algname to refine the predictions $\til Q_h$. Many of the results are analogous to the bounds on $\Qo_h^k, \Qu_h^k, \Vo_h^k, \Vu_h^k$ proven in Section \ref{sec:bound-cis}. However, since the updating procedures are distinct from those used to update the upper and lower $Q$- and $V$-value functions (in particular, we do not use the multi-step bootstrap of \cite{xu_fine-grained_2021} to update $\til Q_h^k, \til V_h^k$), we cannot derive the results in this section directly from those in Section \ref{sec:bound-cis}.

The first result, Lemma \ref{lem:rq-equality}, is a straightforward consequence of the updates to $\Ro_h^k$ in \algname. 
\begin{lemma}\label{lem:rq-equality}
  For any $(x,a,h,k) \in \MS \times \MA \times [H] \times [K]$, suppose the episodes in which $(x,a)$ was previously taken at step $h$ are denoted $k^1, \ldots, k^n < k$ (in particular, $k^i =  k^i_h(x,a)$ and $n = N_h^k(x,a)$). Then the following identity holds:
  \begin{align}
    (\Ro_h^k - Q_h^\st)(x,a) =& \alpha_n^0 (H - Q_h^\st(x,a)) + \sum_{i=1}^n \alpha_n^i \cdot \left( (\til V_{h+1}^{k^i} - V_{h+1}^\st)(x_{h+1}^{k^i}) + \left( (\hat \BP_h^{k^i} - \BP_h) V_{h+1}^\st \right)(x,a) + b_i \right)\nonumber.
  \end{align}
\end{lemma}
\begin{proof}
  Note that $\Ro_h^k(x,a)$ is updated as follows:
  \begin{align}
    \Ro_h^{k+1}(x,a) = \begin{cases}
      (1 - \alpha_{n}) \cdot \Ro_h^k(x,a) + \alpha_{n} \cdot [r_h(x,a) + \til V_{h+1}^k(x_{h+1}^k) + b_{n}]: & (x,a) = (x_h^k, a_h^k), \\
      \Ro_h^k(x,a) : & \mbox{else},
    \end{cases}\nonumber
  \end{align}
  where $n =  N_h^{k+1}(x,a)$ in the first case above. 
  Iterating the above, we obtain that for any $(x,a,h,k)$, letting $n = N_h^k(x,a)$, 
  \begin{align}
\Ro_h^k(x,a) = \alpha_n^0 \cdot H + \sum_{i=1}^n \alpha_n^i \cdot \left( r_h(x,a) + \til V_{h+1}^{k^i}(x_{h+1}^{k^i}) + b_i \right) \label{eq:ro-r}.
  \end{align}
  Using the Bellman optimality equation $Q_h^\st(x,a) = r_h(x,a) + \BP_h V_{h+1}^\st(x,a)$ together with the fact that $\sum_{i=0}^n \alpha_n^i = 1$ and the notation $(\hat\BP_h^{k^i} V_{h+1})(x,a) = V_{h+1}(x_{h+1}^{k^i})$ for $(x,a) = (x_h^{k^i}, a_h^{k^i})$, we see that, for $n = N_h^k(x,a)$,
  \begin{align}
Q_h^\st(x,a) = \alpha_n^0 \cdot Q_h^\st(x,a) + \sum_{i=1}^n \alpha_n^i \cdot \left( r_h(x,a) + (\BP_h - \hat \BP_h^{k^i}) V_{h+1}^\st(x,a) + V_{h+1}^\st(x_{h+1}^{k^i})\right)\label{eq:qhstar-r}.
  \end{align}
  Subtracting (\ref{eq:qhstar-r}) from (\ref{eq:ro-r}) gives the desired result.
\end{proof}

The following straightforward lemma, which generalizes item \ref{it:correct-elimination} of Lemma \ref{lem:qo-qu-gap}, shows that any \emph{approximately} optimal action $a$ at any state $(x,h)$ either remains in $A_h^k(x)$ at each episode $k$ or else there is some other action in $A_h^k(x)$ with smaller sub-optimality than $a$. 
\begin{lemma}
  \label{lem:action-inclusion}
  Under the event $\MEwc$, for any $\ep > 0$ and every $(x,a,h) \in \MS \times \MA\times [H]$, if it holds that $\Delta_h(x,a) \leq \ep$, then for each $k \in [K]$, at least one of the following must hold true:
  \begin{itemize}
  \item $a \in A_h^k(x)$; or
  \item For some $a^\st \in A_h^k(x)$ (in particular, we may choose $a^\st \in A_h^k(x)$ maximizing $\Qu_h^k(x, a^\st)$), $\Delta_h(x, a^\st) \leq V_h^\st(x) - \Qu_h^k(x, a^\st) < \ep$. 
  \end{itemize}
\end{lemma}
\begin{proof}
If $a \not \in A_h^k(x)$, then it must be the case that for some $k' \leq k$, $\Qo_h^{k'}(x, a) < \Vu_h^{k'}(x)$; by Lemma \ref{lem:qv-monotonicity}, we have $\Vu_h^k(x) \geq \Vu_h^{k'}(x)$, and so some action $a^\st \in A_h^k(x)$ must satisfy $Q_h^\st(x, a^\st) \geq \Qu_h^k(x, a^\st) = \Vu_h^k(x) > \Qo_h^{k'}(x,a) \geq Q_h^\st(x,a)$. Hence $\Delta_h(x, a^\st) = V_h^\st(x) - Q_h^\st(x, a^\st) \leq V_h^\st(x) - \Qu_h^k(x, a^\st) < V_h^\st(x) - Q_h^\st(x,a) = \Delta_h(x,a) \leq \ep$. 
\end{proof}

The next lemma, Lemma \ref{lem:ep-good-preds}, uses Lemmas \ref{lem:rq-equality} and \ref{lem:action-inclusion} above together with a martingale concentration inequality to show bounds on $\til Q_h^k, \til V_h^k$ that hold with high probability. We note that an additional necessary ingredient is the assumption that the input predictions $\til Q_h$ are an $\ep$-approximate distillation of $Q_h^\st$; this is used to show that for all $k \in [K]$, $\til Q_h^k$ is also an approximate distillation with high probability (item \ref{it:tilq-aub}), which in turn is used to show that $\til V_h^k$ is approximately lower bounded by $V_h^\st$ (item \ref{it:tilv-vstar}).
\begin{lemma}
  \label{lem:ep-good-preds}
Set $p = 1/(H^2 K)$.  Suppose that $\til Q$ is an $\ep$-approximate distillation on the optimal value function $Q_h^\st$. Then, there is an event $\MEpred$ with $\Pr[\MEpred] \geq 1-p$ so that the following hold under $\MEpred \cap \MEwc$:
  \begin{enumerate}
  \item \label{it:rq-ub} For $n \in \BN$, recall that $\beta_n = 2 \sum_{i=1}^n \alpha_n^i b_i$. Then for any $(x,a,h,k) \in \MS \times \MA \times [H] \times [K]$, it holds that, for $n =  N_h^k(x,a)$, %
    \begin{align}
      (\Ro_h^k - Q_h^\st)(x,a) \leq  \alpha_n^0 H + \sum_{i=1}^n \alpha_n^i \cdot (\til V_{h+1}^{ k^i_h(x,a)} - V_{h+1}^\st)(x_{h+1}^{ k^i_h(x,a)}) + \beta_n\nonumber.
    \end{align}
  \item For all $(x,a,h,k) \in \MS \times \MA \times [H] \times [K]$, it holds that $\Ro_h^k(x,a) \geq Q_h^\st(x,a) - \ep\cdot (H+1-h)$. \label{it:rq-bound}
  \item For all $(x,h,k) \in \MS \times [H] \times [K]$, there is some $\bar a \in \MA$ so that $\Delta_h(x, \bar a) \leq \ep$ and $\til Q_h^k(x, \bar a) \geq Q_h^\st(x, \bar a) - \ep \cdot (H+1-h)$. In particular, $\til Q^k$ is an $\ep\cdot (H+2-h)$-approximate distillation on $Q^\st$. \label{it:tilq-aub}
  \item For all $(x, h, k) \in \MS \times [H] \times [K]$, it holds that $\til V_h^k(x) \geq V_h^\st(x) - \ep\cdot (H+2-h)$.\label{it:tilv-vstar}
  \item \label{it:rq-ub-clip} For any $(h,k) \in \MS \times \MA \times [H] \times [K]$, it holds that, for $n =  N_h^k(x_h^k,a_h^k)$,
    \begin{align}
      (\Ro_h^k - Q_h^\st)(x_h^k,a_h^k) \leq &  \alpha_n^0 H + \clip{\beta_n}{\frac{[\Delta_h(x_h^k,a_h^k) - 2\ep \cdot (H+1)]_+}{2H}} \nonumber\\
      & + \left( 1 + \frac{1}{H} \right) \cdot  \sum_{i=1}^n \alpha_n^i \cdot (\til V_{h+1}^{k^i_h(x_h^k,a_h^k)} - V_{h+1}^\st)(x_{h+1}^{k^i_h(x_h^k,a_h^k)}) \nonumber.
    \end{align}
  \end{enumerate}
\end{lemma}
\begin{proof}
  Fix any $(x,a,h) \in \MS \times \MA \times [H]$. Set $k^0 = 0$ and for $i \geq 1$,
  \begin{align}
k^i := \min \left( \left\{ k \in [K] : k > k^{i-1} \mbox{ and } (x_h^k, a_h^k) = (x,a) \right\} \cup \{K+1\} \right)\nonumber.
  \end{align}
  Let $\MH_k$ denote the $\sigma$-field generated by all random variables up to and including episode $k$, step $H$; the random variable $k^i$ is a stopping time of the filtration $(\MH_k)_{k \geq 0}$. Let $\MF_i$, $i \geq 0$ be the filtration given by $\MF_i = \MH_{k^i}$. Then $\left( \One[k^i \leq K] \cdot [(\hat \BP_h^{k^i} - \BP_h) V_{h+1}^\st](x,a) \right)_{i=1}^K$ is a martingale difference sequence adapted to the filtration $\MF_i$. By the Azuma-Hoeffding inequality and a union bound over all $m \in [K]$, it holds that, for some constant $C_0 > 0$, with probability at least $1- p/(SAH)$,
  \begin{align}
\forall m \in [K]: \quad \left| \sum_{i=1}^m \alpha_m^i \cdot \One[k^i \leq K] \cdot [(\hat \BP_h^{k^i} - \BP_h) V_{h+1}^\st](x,a) \right| \leq \frac{C_0 H}{4} \sqrt{\sum_{i=1}^m (\alpha_m^i)^2 \cdot \iota} \leq \frac{C_0}{2} \sqrt{\frac{H^3 \iota}{m}}\label{eq:hoeffding-r},
  \end{align}
  where the final inequality follows from item \ref{it:alpha-sum-square} of Lemma \ref{lem:alpha}. Taking a union bound over all $(x,a,h) \in \MS \times \MA \times [H]$, we get that with probability $1-p$, for all $(x,a,h,k) \in \MS \times \MA \times [H] \times [K]$,
  \begin{align}
\left| \sum_{i=1}^n \alpha_n^i \left[ (\hat \BP_h^{ k_h^i(x,a)} - \BP_h) V_{h+1}^\st \right](x,a) \right| \leq \frac{C_0}{2} \sqrt{\frac{H^3 \iota}{n}} \quad \mbox{ where } n=  N_h^k(x,a) \label{eq:hoeffding-r-2}.
  \end{align}
  (Here we have applied (\ref{eq:hoeffding-r}) with $m =  N_h^k(x,a) \leq K$, and used the fact that $\One[ k_h^i(x,a) \leq K] = 1$ for $i \leq  N_h^k(x,a)$.)  Let $\MEpred$ denote the probability $1-p$ event under which (\ref{eq:hoeffding-r-2}) holds. From (\ref{eq:beta-bound}) we have that $\beta_n/2 \geq C_0 \sqrt{H^3 \iota/n}$. %
  Then item \ref{it:rq-ub} of the lemma follows from Lemma \ref{lem:rq-equality} and (\ref{eq:hoeffding-r-2}).

 To establish the remaining items of the lemma statement, we use reverse induction on $h$. The base case $h = H+1$ is immediate since as a matter of convention, all of $\Ro_{H+1}^k, Q_{H+1}^\st, \til Q_{H+1}^k \til V_{H+1}^k, V_{H+1}^\st$ are identically 0. Assuming that items \ref{it:rq-bound}, \ref{it:tilq-aub}, and \ref{it:tilv-vstar} hold for step $h+1$, (\ref{eq:hoeffding-r-2}) and Lemma \ref{lem:rq-equality} give that, under the event $\MEpred$, for each $(x,a,h,k) \in \MS \times \MA \times [H] \times [K]$, for $n = \til N_h^k(x,a)$,
  \begin{align}
    (\Ro_h^k - Q_h^\st)(x,a) \geq & \sum_{i=1}^n \alpha_n^i \cdot \left( (\til V_{h+1}^{\til k_h^i(x,a)} - V_{h+1}^\st)(x_{h+1}^{\til k_h^i(x,a)}) + \left( ( \hat \BP_h^{\til k_h^i(x,a)} - \BP_h) V_{h+1}^\st \right) (x,a) + b_i \right) \nonumber\\
    \geq & - \ep \cdot (H +1- h) + \beta_n / 2 - \frac{C_0}{2} \sqrt{H^3 \iota/n} \nonumber\\
    \geq & -\ep \cdot (H+1-h) \nonumber,
  \end{align}
  thus establishing item \ref{it:rq-bound} of the lemma at step $h$.

  To establish item \ref{it:tilq-aub} at step $h$, we use increasing induction on $k$. The base case $k = 1$ follows from the fact that, by assumption, $\til Q^1$ is an $\ep$-approximate distillation on $Q^\st$. To establish the inductive step, we note that by construction, $\til Q_h^{k}(x,a) = \min \{ \Ro_h^{k}(x,a), \til Q_h^{k-1}(x,a), \Qo_h^{k}(x,a) \}$ for all $(x,a) \in \MS \times \MA$. By the inductive hypothesis (on $k$), for each $x \in \MS$, there is some $\bar a \in \MA$ so that $\Delta_h(x,\bar a) \leq \ep$ %
  and $\til Q_h^{k-1}(x,\bar a) \geq Q_h^\st(x,\bar a) - \ep \cdot (H+1-h)$.  Under the event $\MEwc$ have that $\Qo_h^{k}(x,\bar a) \geq Q_h^\st(x, \bar a)$, and we have already established (item \ref{it:rq-bound}) that $\Ro_h^{k}(x, \bar a) \geq Q_h^\st(x, \bar a) - \ep \cdot (H+1-h)$, which implies that $\til Q_h^{k}(x,\bar a) \geq Q_h^\st(x,\bar a) - \ep \cdot (H+1-h)$.

  Finally we establish item \ref{it:tilv-vstar} at step $h$. Again we use increasing induction on $k$, noting that the base case $k = 1$ follows from the fact that, for all $x \in \MS$, $\til V_h^1(x) = \max_{a \in \MA} \til Q_h^1(x,a) \geq V_h^\st(x) - \ep$, using that $\til Q^k$ is an $\ep$-approximate distillation on $Q^\st$.
  To establish the inductive step (i.e., at episode $k$, assuming that item \ref{it:tilv-vstar} holds at episode $k-1$ and step $h$), note that for any $x \in \MS$, %
  \begin{align}
\til V_{h}^{k}(x) = \max_{a' \in A_h^k(x)} \left\{ \max\{\til Q_h^{k}(x, a'), \Qu_h^{k}(x, a')\} \right\} \label{eq:tilvk-update}
  \end{align}
  Moreover, since $\til Q_h^k$ is an $\ep \cdot (H+2-h)$-approximate distillation on $Q_h^\st$ (item \ref{it:tilq-aub} at step $h$), for any $x$, there is some $\bar a \in \MA$ so that $\Delta_h(x, \bar a) + [ Q_h^\st(x, \bar a) - \til Q_h^k(x, \bar a)]_+ \leq \ep \cdot (H+2-h)$. By Lemma \ref{lem:action-inclusion} with $(x,a,h) = (x, \bar a,h)$, since $\Delta_h(x, \bar a) \leq \ep \cdot (H+2-h)$, it holds that either $\bar a \in A_h^k(x)$ or else there is some $a^\st \in A_h^k(x)$ so that $V_h^\st(x) - \Qu_h^k(x, a^\st) < \ep$. If $\bar a \in A_h^k(x)$, then %
  \begin{align}
\max_{a' \in A_h^k(x)} \left\{\max \{ \til Q_h^k(x, a'), \Qu_h^k(x, a') \} \right\}\geq \til Q_h^k(x, \bar a) \geq V_h^\st(x) - \ep \cdot (H+2-h)\nonumber.
  \end{align}
  Otherwise,  we have
    \begin{align}
   \max_{a' \in A_h^k(x)} \left\{\max\{ \til Q_h^k(x, a_h^k), \Qu_h^k(x, a_h^k) \}\right\} \geq \Qu_h^k(x, a^\st) > V_h^\st(x) - \ep \cdot (H+2-h)\nonumber.
    \end{align}
    Thus, by (\ref{eq:tilvk-update}) and the inductive hypothesis (on $k$), it holds that $\til V_h^{k}(x) \geq V_h^\st(x) - \ep \cdot (H+2-h)$, as desired.

    Next we establish item \ref{it:rq-ub-clip} of the lemma. By item \ref{it:rq-bound} of the lemma, under the event $\MEwc$, we have, for all $(h,k) \in [H] \times [K]$,
    \begin{align}
      \Ro_h^k(x_h^k, a_h^k) \geq Q_h^\st(x_h^k, a_h^k) - \ep \cdot H \geq \Qu_h^k(x_h^k, a_h^k) - \ep \cdot H\nonumber.
    \end{align}
Then by (\ref{eq:tilvk-update}) and the definition of $a_h^k$ (using that $\tau_h^k = 0$), we have that
    \begin{align}
      \Ro_h^k(x_h^k, a_h^k) & \geq \max\{ \Qu_h^k(x_h^k, a_h^k), \til Q_h^k(x_h^k, a_h^k) \} - \ep \cdot H\nonumber\\
                                 & = \max_{a' \in A_h^k(x_h^k)} \left\{ \max\{ \Qu_h^k(x_h^k, a'), \til Q_h^k(x_h^k, a') \} \right\} - \ep \cdot H\nonumber\\
      & =\til V_h^k(x_h^k) - \ep \cdot H.\nonumber
    \end{align}
    Thus, by item \ref{it:tilv-vstar} of the lemma, we have that $\Ro_h^k(x_h^k, a_h^k) \geq V_h^\st(x_h^k) - 2\ep \cdot (H+1)$. Hence
    \begin{align}
\Ro_h^k(x_h^k, a_h^k) - Q_h^\st(x_h^k, a_h^k) \geq V_h^\st (x_h^k) - Q_h^\st(x_h^k, a_h^k) - 2\ep \cdot (H+1) \geq \Delta_h(x_h^k, a_h^k) - 2\ep \cdot (H+1)\nonumber.
    \end{align}
The statement of item \ref{it:rq-ub-clip} then follows from item \ref{it:rq-ub} and Lemma \ref{lem:clip}.
\end{proof}

\subsection{Additional bounds on $Q$- and $V$-value functions}
In this section we prove some additional bounds on $\Vo_h^k, \Vu_h^k, \RanV_h^k, \clipV_h^k$. 

Recall from Lemma \ref{lem:clip-lt-range} that for any $(x,h,k,a)$ for which $x\not\in \MG_h^k$, we have $\clipV_h^k(x) \leq \RanV_h^k(x)$. The next lemma (in particular, (\ref{eq:vo-vu-clipv-2})) shows that a reverse inequality holds up to a factor of $1+1/H$, \emph{if there is a non-optimal action in $A_h^k(x)$}. This fact formalizes the intuition that the purpose of defining the clipped value functions $\clipV_h^k, \clipQ_h^k$ is to avoid paying for the case when only optimal actions remain in $A_h^k(x)$ (in which case one should not suffer any regret no matter which action is taken at $x$).
\begin{lemma}
  \label{lem:clipv-delmin-lb}
  For any $(x,h,k) \in \MS \times [H] \times [K]$, if there is a non-optimal action in $A_h^k(x)$, 
  then, under the event $\MEwc$, it holds that
  \begin{align}
    \clipV_h^k(x)
    \geq & \frac{\delmins{h}(x)}{4}\label{eq:clipv-delmin-lb}\\
    \Vo_h^k(x) - \Vu_h^k(x) \leq& \RanV_h^k(x) \leq \left(1 + \frac 1H \right)\cdot \clipV_h^k(x) \label{eq:vo-vu-clipv-2}.
  \end{align}
\end{lemma}
\begin{proof}
  Fix $x,h,k$ and set $a^\st = \argmax_{a' \in A_h^k(x)} \{ \Qo_h^k(x, a') - \Qu_h^k(x, a' \}$. Then by Definition \ref{def:half-clipped}, we have that $\clipV_h^k(x) = \min \left\{ \clipV_h^{k-1}(x),\ \clipQ_h^k(x, a^\st) \right\}$. If $\clipV_h^k(x) = \clipV_h^{k-1}(x)$, then we may replace $k$ with $k-1$, noting the existence of a non-optimal action in $A_h^k(x)$ implies the existence of a non-optimal action in $A_h^{k-1}(x)$ (continuing this process may eventually lead to the case $k=0$, for which (\ref{eq:clipv-delmin-lb}) and (\ref{eq:vo-vu-clipv-2}) hold by the definition $\clipV_h^0(x) = H$). So we may assume that $\clipV_h^k(x) = \clipQ_h^k(x, a')$.

  Let $a'$ denote some sub-optimal action in $A_h^k(x)$. Then under the event $\MEwc$, we must have that $x \not \in \MG_h^k$, meaning that 
    \begin{align}
      \Qo_h^k(x, a') - \Qu_h^k(x, a') \leq & \Qo_h^k(x, a') - \Qu_h^k(x, a') \tag{Since $a'$ maximizes the confidence interval}\\
      \leq & \RanQ_h^k(x, a') \tag{By Lemma \ref{lem:range-bound}}\\
      \leq & \clipQ_h^k(x, a') + \frac{\delmin}{4} \tag{By Lemma \ref{lem:hc-range-bound}}.
    \end{align}
    Moreover, as in the proof of Lemma \ref{lem:bound-by-clipdelta}, we have, by Lemmas \ref{lem:range-bound} and \ref{lem:hc-range-bound},%
    \begin{align}
\Vo_h^k(x) - \Vu_h^k(x) \leq \RanV_h^k(x) \leq \clipV_h^k(x) + \frac{\delmin}{4H} = \clipQ_h^k(x, a') + \frac{\delmin}{4H}.\label{eq:vo-vu-clipq-ub}
    \end{align}
    Combining the above displays, we obtain that under the event $\MEwc$,
    \begin{align}
      \delmins{h}(x) \leq & \Delta_h(x,a') \nonumber\\
      \leq & (\Vo_h^k(x) - \Vu_h^k(x))  + (\Qo_h^k(x, a') - \Qu_h^k(x, a'))\nonumber\\
      \leq & 2 \cdot \clipQ_h^k(x, a') + \frac{\delmin}{2}\nonumber,
    \end{align}
    which implies that $\clipV_h^k(x) = \clipQ_h^k(x, a') \geq \frac{\delmins{h}(x)}{4}$. This verifies (\ref{eq:clipv-delmin-lb}). To verify (\ref{eq:vo-vu-clipv-2}), we use (\ref{eq:vo-vu-clipq-ub}) to get
    \begin{align}
\Vo_h^k(x) - \Vu_h^k(x) \leq \RanV_h^k(x) \leq \clipV_h^k(x) + \frac{\delmin}{4H} \leq \clipV_h^k(x) + \frac{\delmins{h}(x)}{4H} \leq \left(1 + \frac 1H \right) \cdot \clipV_h^k(x)\nonumber,
    \end{align}
    where the final inequality follows from (\ref{eq:clipv-delmin-lb}).
  \end{proof}

  The following simple lemma shows that $\til V_h^k$ is bounded above by $\Vo_h^k$, which is an immediate consequence of the definition of $\til V_h^k$ in \algname.
    \begin{lemma}
    \label{lem:tilv-vo}
For all $x,a,h,k \in \MS \times \MA \times [H] \times [K]$, it holds that $\til V_h^k(x) \leq \Vo_h^k(x)$.
\end{lemma}
\begin{proof}
The definition of $\til Q_h^k$ at step \ref{it:def-tilq-hk} ensures that for all $x,a,j,k$, we have that $\til Q_h^k(x,a) \leq \Qo_h^k(x,a)$. The conclusion of the lemma follows from the fact that $\Vo_h^k(x) = \max_{a' \in A_h^k(x)} \{ \Qo_h^k(x,a') \}$ and $\til V_h^k(x) \leq \max_{a' \in A_h^k(x)} \{ \til Q_h^k(x, a') \}$. 
\end{proof}

\subsection{Regret bounds for approximate distillation}
\label{sec:regret-distillation}
The below lemma, the main result of this section, shows that in the case that the provided predictions $\til Q$ are an $\ep$-approximate distillation of the true value function $Q^\st$, then we may bound the regret of \algname by a quantity that in general will be smaller than generic worst-case regret bounds. In particular, the set of states and actions $\MS \times \MA \times [H]$ is replaced with the fooling set $\fool{\ep'/2}{\ep'}$, which will be significantly smaller if the predictions $\til Q_h$ are very accurate.
\begin{lemma}
  \label{lem:bound-by-foolset}
  Suppose the event $\MEwc \cap \MEpred$ holds, and set $\ep' := 4\ep \cdot (H+1)$. 
  If $\til Q$ is an $\ep$-approximate distillation on $Q^\st$ and either $\til Q$ lacks $\ep'$-fooling optimal actions (Definition \ref{def:ofa}). Then the following regret bounds hold:
  \begin{align}
    & \sum_{k=1}^K (V_1^\st - V_1^{\pi^k})(x_1^k) \nonumber\\
    \leq& %
          O((\ep H + \ep') \cdot TH ) + O \left( H \cdot \sum_{h=1}^H \sum_{k=1}^K \sigma_h^k \clipdelta_h^k \right)\nonumber\\
    & + O\left( \min \left\{ \sqrt{H^6 K \iota \cdot |\fool{\ep(H+1)}{\ep'}|}, \sum_{(x,a,h) \in \fool{\ep(H+1)}{\ep'}} \frac{H^4 \iota}{[\Delta_h(x,a) - 2\ep \cdot (H+1)]_+}\right\} \right) \label{eq:nogap-adaptive}.
  \end{align}
\end{lemma}
\begin{proof}
  Recall from \algname that the values $\tau_h^k \in \{0,1\}$ are defined as follows: $\tau_h^k = 0$ if $x_h^k \in \MG_h^k$ or $\RanV_h^k(x_h^k) \leq \ggapfn{h}{\wh \Delta^k}$, and $\tau_h^k = 1$ otherwise.
  Also recall that we defined values $\sigma_h^k \in \{0,1\}$ for all $(h,k) \in [H] \times [K]$ as follows: $\sigma_h^k = 0$ if $x_h^k \in \MG_h^k$ or $\clipV_h^k(x_h^k) \leq \frac{1}{1+\frac 1H}\cdot \ggapfn{h}{\wh \Delta^k}$, and $\sigma_h^k = 1$ otherwise. 

For any $h \in [H]$, let $\MWsig_h \subset [K]$ denote the set of episodes $k$ for which $\sigma_h^k = 1$. Similarly, let $\MWtau_h \subset [K]$ denote the set of episodes $k$ for which $\tau_h^k = 1$. %

Set $\ep' = 4\ep \cdot (H+1)$, and for each $(x,h) \in \MS \times [H]$, let $\Aopt_{h,\ep'}(x)$ denote the set of actions $a' \in \MA$ so that $V_h^\st(x) - Q_h^\st(x, a') \leq \ep'$. 
  Next, for any $k \in [K]$, we have that, under the event $\MEwc$,
  \begin{align}
    &  \sum_{h=1}^H V_h^\st(x_h^k) - Q_h^\st(x_h^k, a_h^k) \nonumber\\
    \leq & \ep' H + \sum_{h=1}^H \One[a_h^k \not \in \Aopt_{h,\ep'}(x_h^k)] \cdot (V_h^\st(x_h^k) - Q_h^\st(x_h^k,a_h^k)) \nonumber\\
    \leq & (\ep \cdot (H+1) + \ep') H +  \sum_{h=1}^H (1-\sigma_h^k) \cdot \One[a_h^k \not \in \Aopt_{h,\ep'}(x_h^k)] \cdot (\til V_h^k(x_h^k) - Q_h^\st(x_h^k,a_h^k))  + 4 \sigma_h^k \cdot \clipdelta_h^k\label{eq:tilv-ub}\\
    \leq & (2\ep \cdot (H+1) + 2\ep') H +  \sum_{h=1}^H (1-\tau_h^k) \cdot \One[a_h^k \not \in \Aopt_{h,\ep'}(x_h^k)] \cdot (\Ro_h^k(x_h^k, a_h^k) - Q_h^\st(x_h^k,a_h^k)) + 4\sigma_h^k \cdot \clipdelta_h^k \label{eq:2tilv-ub},
  \end{align}
  where (\ref{eq:tilv-ub}) follows from item \ref{it:tilv-vstar} of Lemma \ref{lem:ep-good-preds} and Lemma \ref{lem:bound-by-clipdelta}, and (\ref{eq:2tilv-ub}) follows from Claims \ref{clm:tau-tilv} and \ref{clm:sig-tilv} below.

  \begin{claim}
    \label{clm:tau-tilv}
For any $(h,k) \in [H] \times [K]$ so that $\tau_h^k = 0$, it holds that $\til V_h^k(x_h^k) \leq \Ro_h^k(x_h^k, a_h^k) + \ep \cdot H$. 
  \end{claim}
\begin{proof}
  Since $\tau_h^k = 0$, we have that $a_h^k = \argmax_{a' \in A_h^k(x_h^k)} \{ \max\{ \til Q_h^k(x_h^k, a'), \Qu_h^k(x_h^k, a') \} \}$. Then by definition of $\til V_h^k$,  %
  \begin{align}
\til V_h^k(x_h^k) =  \max\{ \til Q_h^k(x_h^k, a_h^k), \Qu_h^k(x_h^k, a_h^k\}\leq \Ro_h^k(x_h^k, a_h^k) + \ep \cdot H\label{eq:tilv-r-ep}
  \end{align}
  where the inequality above uses the fact that $\Qu_h^k(x_h^k, a_h^k) \leq Q_h^\st(x_h^k, a_h^k) \leq \Ro_h^k(x_h^k, a_h^k) + \ep \cdot (H+1-h)$ (item \ref{it:rq-bound} of Lemma \ref{lem:ep-good-preds}) and that $\til Q_h^k(x_h^k, a_h^k) \leq \Ro_h^k(x_h^k, a_h^k)$ by the definition of $\til Q_h^k$ (step \ref{it:def-tilq-hk} of the algorithm).
\end{proof}

\begin{claim}
  \label{clm:sig-tilv}
  For any $(h,k) \in [H] \times [K]$ so that $\tau_h^k = 1$, at least one of the following statements holds true under the event $\MEwc$:
  \begin{itemize}
  \item $\til V_h^k(x_h^k) \leq Q_h^\st(x_h^k, a_h^k) + \ep'$. %
  \item $\RanV_h^k(x_h^k) \leq \left(1 + \frac 1H\right) \cdot \clipV_h^k(x_h^k)$ and $\sigma_h^k = 1$.
  \end{itemize}
\end{claim}
\begin{proof}
  Suppose that the second statement does not hold true. Then either $\RanV_h^k(x_h^k) > (1 + 1/H) \cdot \clipV_h^k(x_h^k)$ or $\sigma_h^k = 0$. First suppose that $\sigma_h^k = 0$. Since $\tau_h^k = 1$, we have $x_h^k \not \in \MG_h^k$, meaning that $\clipV_h^k(x_h^k) \leq \frac{1}{1+1/H}\cdot \ggapfn{h}{\wh \Delta^k}$. But $\tau_h^k = 1$ also implies that $\RanV_h^k(x_h^k) > \ggapfn{h}{\wh \Delta^k}$, which implies that $\RanV_h^k(x_h^k) > (1+1/H) \cdot \clipV_h^k(x_h^k)$.

  Thus we may assume from here on that $\RanV_h^k(x_h^k) > (1+1/H) \cdot \clipV_h^k(x_h^k)$. By Lemma \ref{lem:clipv-delmin-lb}, under the event $\MEwc$, $A_h^k(x_h^k)$ must consist of only optimal actions. By definition of $\til V_h^k$, there is some $a \in A_h^k(x_h^k)$ so that $\til V_h^k(x_h^k) = \max \{ \til Q_h^k(x_h^k, a), \Qu_h^k(x_h^k, a) \}$. We know that $a$ must be an optimal action, i.e., $\Delta_h(x_h^k, a) = 0$. %
  Since the input predictions $\til Q$ lack $\ep'$-fooling optimal actions (Definition \ref{def:ofa}),\footnote{We remark that this is the only place in the proof where we use that the predictions $\til Q$ lack $\ep'$-fooling optimal actions.} it holds that $\til Q_h(x_h^k, a) \leq V_h^\st(x_h^k) + \ep'$. %
  Therefore,
  \begin{align}
\til V_h^k(x_h^k) = \til Q_h^k(x_h^k, a) \leq \til Q_h(x_h^k, a) \leq V_h^\st(x_h^k) + \ep'.\nonumber
  \end{align}
  Moreover, %
  since $a_h^k \in A_h^k(x_h^k)$ (and therefore is an optimal action), we have that $V_h^\st(x_h^k) = Q_h^\st(x_h^k, a_h^k)$, meaning that $\til V_h^k(x_h^k) - Q_h^\st(x_h^k, a_h^k) \leq \ep'$, as desired.
\end{proof}
We next need the following claim:
  \begin{claim}
    \label{clm:aopt-fe}
    For any $(k,h)$ satisfying $\tau_h^k = 0$, if either
    \begin{enumerate}
    \item $a_h^k \not \in \Aopt_{h,\ep'}(x_h^k)$; or
    \item $(\til V_h^k - V_h^\st)(x_h^k) > \ep'$, 
    \end{enumerate}
    then we have that, under the event $\MEwc$, $(x_h^k, a_h^k, h) \in \fool{\ep \cdot (H+1)}{\ep'}$.
  \end{claim}
  \begin{proof}[Proof of Claim \ref{clm:aopt-fe}]
    For the entirety of the proof of the claim we assume that $\MEwc$ holds. 
    We first suppose that $a_h^k \not \in \Aopt_{h,\ep'}(x_h^k)$. Notice that $\Delta_h(x_h^k, a_h^k) > \ep'$ since $a_h^k \not\in \Aopt_{h,\ep'}(x_h^k)$. By item \ref{it:tilv-vstar} of Lemma \ref{lem:ep-good-preds} and the choice of $a_h^k$ when $\tau_h^k = 0$,
    \begin{align}
\max \{ \til Q_h^k(x_h^k, a_h^k), \Qu_h^k(x_h^k, a_h^k) \}  = \max_{a' \in A_h^k(x_h^k)} \left\{ \max\{ \til Q_h^k(x_h^k, a'), \Qu_h^k(x_h^k, a') \} \right\} = \til V_h^k(x_h^k) \geq V_h^\st(x_h^k) - \ep \cdot (H+1)\label{eq:tilq-qu-tilv}
    \end{align}
    If $\Qu_h^k(x_h^k, a_h^k) \geq V_h^\st(x_h^k) - \ep \cdot (H+1)$, then it holds that $Q_h^\st(x_h^k, a_h^k) \geq V_h^\st(x_h^k) - \ep \cdot (H+1) > V_h^\st(x_h^k) - \ep'$ (since $\ep' > \ep \cdot (H+1)$), which contradicts $a_h^k \not\in \Aopt_{h,\ep'}(x_h^k)$. Hence $\til Q_h^k(x_h^k, a_h^k) \geq V_h^\st(x_h^k) - \ep \cdot (H+1)$, meaning that $\til Q_h^k(x_h^k, a_h^k) - Q_h^\st(x_h^k, a_h^k) \geq \Delta_h(x_h^k, a_h^k) - \ep \cdot (H+1)$. Since $\til Q_h(x_h^k, a_h^k) = \til Q_h^1(x_h^k, a_h^k) \geq \til Q_h^k(x_h^k, a_h^k)$, we get that $(x_h^k, a_h^k, h) \in \fool{\ep \cdot (H+1)}{\ep'}$.

    Next suppose that $(\til V_h^k - V_h^\st)(x_h^k) > \ep'$. Then, again using the choice of $a_h^k$ when $\tau_h^k = 0$,
    \begin{align}
\max\{ \til Q_h^k(x_h^k, a_h^k), \Qu_h^k(x_h^k ,a_h^k) \} \geq \til V_h^k(x_h^k) > V_h^\st(x_h^k) + \ep'.\nonumber
    \end{align}
    Since $\Qu_h^k(x_h^k, a_h^k) \leq Q_h^\st(x_h^k, a_h^k) \leq V_h^\st(x_h^k)$ under the event $\MEwc$, we must have $\til Q_h^k(x_h^k, a_h^k) > V_h^\st(x_h^k) + \ep'$, which implies that $\til Q_h^1(x_h^k, a_h^k) > V_h^\st(x_h^k) + \ep'$, meaning that $(x_h^k, a_h^k, h) \in \fool{\ep \cdot (H+1)}{\ep'}$. 
{        }
    \end{proof}

    Next for any $h \in [H]$, we compute %
    \begin{align}
      & \sum_{k \not\in \MW_h^\tau}  \One[a_h^k \not\in \Aopt_{h,\ep'}(x_h^k)] \cdot (\Ro_h^k(x_h^k, a_h^k) - Q_h^\st(x_h^k, a_h^k) ) \nonumber\\
      \leq & \sum_{(x,a,h) \in \fools{\ep \cdot (H+1)}{\ep'}{h} } \sum_{i=1}^{ N_h^{K+1}(x,a)}  ( \Ro_h^{k_h^i(x,a)}(x,a) - Q_h^\st(x,a) )\label{eq:use-claim-subopt}\\
      \leq & H \cdot |\fools{\ep(H+1)}{\ep'}{h}| + \sum_{(x,a,h) \in \fools{\ep(H+1)}{\ep}{h}} \left(\sum_{i=1}^{N_h^{K+1}(x,a)} \left( 1 + \frac 1H \right) \sum_{t=1}^i \alpha_i^t \cdot (\til V_{h+1}^{k_h^t(x,a)} - V_{h+1}^\st)(x_{h+1}^{k^t_h(x,a)}) \right. \nonumber\\
           &   \left. + \clip{\beta_i}{\frac{[\Delta_h(x,a) - 2\ep \cdot (H+1)]_+}{4H}}\right) \label{eq:use-rq-ub-clip} \\
      \leq &  H \cdot |\fools{\ep(H+1)}{\ep'}{h}| + \sum_{(x,a,h) \in \fools{\ep(H+1)}{\ep'}{h}} \left( \left( 1 + \frac 1H \right) \sum_{t=1}^{ N_h^{K+1}(x,a)}  (\til V_{h+1}^{ k_h^t(x,a)} - V_{h+1}^\st)(x_{h+1}^{k^t_h(x,a)}) \cdot \sum_{i=t}^\infty \alpha_i^t \right.\nonumber\\
      & \left.+ \sum_{i=1}^{N_h^{K+1}(x,a)} \clip{\beta_i}{\frac{[\Delta_h(x,a) - 2\ep \cdot (H+1)]_+}{4H} } \right)\nonumber\\
      \leq &  H \cdot | \fools{\ep(H+1)}{\ep'}{h}| + \sum_{(x,a,h) \in \fools{\ep(H+1)}{\ep'}{h}} \min \left\{8C_0 \sqrt{H^3 \iota \cdot N_h^{K+1}(x,a)},  \frac{64C_0^2 H^4\iota }{[\Delta_h(x,a) - 2\ep \cdot (H+1)]_+}  \right\} \nonumber\\
      & (1 + 1/H)^2 \cdot \sum_{k \in [K]} (\til V_{h+1}^k - V_{h+1}^\st)(x_{h+1}^k) , %
        \label{eq:tilv-vstar-ub}
    \end{align}
    where (\ref{eq:use-claim-subopt}) uses Claim \ref{clm:aopt-fe}, (\ref{eq:use-rq-ub-clip}) uses item \ref{it:rq-ub-clip} of Lemma \ref{lem:ep-good-preds} and the fact that $(x_h^{k_h^i(x,a)}, a_h^{k_h^i(x,a)}) = (x,a)$, and the final inequality (\ref{eq:tilv-vstar-ub}) uses %
    item \ref{it:alpha-sum-sub} of Lemma \ref{lem:alpha} and Lemma \ref{lem:clip-sum}.
    Moreover, we have that
    \begin{align}
      &  \sum_{k \in [K]} (\til V_{h+1}^k - V_{h+1}^\st)(x_{h+1}^k) \nonumber\\
      \leq & K \cdot \ep' + \sum_{k \in [K]} \One[\sigma_{h+1}^k = 1 \mbox{ and } \RanV_{h+1}^k(x_{h+1}^k) \leq (1+ 1/H) \clipdelta_{h+1}^k] \cdot (\Vo_{h+1}^k - \Vu_{h+1}^k)(x_{h+1}^k) \nonumber\\
      & + \sum_{k \not \in \MWtau_{h+1}} (\til V_{h+1}^k - V_{h+1}^\st)(x_{h+1}^k)\tag{Using Claim \ref{clm:sig-tilv} and Lemma \ref{lem:tilv-vo}}\\
      \leq & (2\ep' + \ep \cdot H) \cdot K +  \sum_{k=1}^K 2 \sigma_{h+1}^k \cdot \clipdelta_{h+1}^k + \sum_{k \not \in \MW_{h+1}^\tau} \One[(\til V_{h+1}^k - V_{h+1}^\st)(x_{h+1}^k) > \ep'] \cdot \left( (\Ro_{h+1}^k- Q_{h+1}^\st)(x_{h+1}^k, a_{h+1}^k)\right) \tag{Using Claim \ref{clm:tau-tilv} and Lemma \ref{lem:range-bound}}\\
      \leq &(2\ep' + \ep H) \cdot K  + 2 \sum_{k=1}^K \sigma_{h+1}^k \clipdelta_{h+1}^k + \sum_{(x,a,h+1) \in \fools{\ep(H+1)}{\ep'}{h+1}}  \sum_{i=1}^{N_{h+1}^{K+1}(x,a)} (\Ro_{h+1}^{k_{h+1}^i(x,a)} (x,a) - Q_{h+1}^\st(x,a))\label{eq:rq-hp1},
    \end{align}
    where (\ref{eq:rq-hp1}) follows from Claim \ref{clm:aopt-fe} (in particular, if $\tau_{h+1}^k = 0$ and $(\til V_{h+1}^k - V_{h+1}^\st)(x_{h+1}^k) > \ep'$, then $(x_{h+1}^k, a_{h+1}^k, h+1) \in \fools{\ep(H+1)}{\ep'}{h+1}$).

    Combining (\ref{eq:tilv-vstar-ub}) and (\ref{eq:rq-hp1}), and  iterating for $h,h+1, \ldots, H$, we see that
    \begin{align}
      & \sum_{(x,a,h) \in \fools{\ep(H+1)}{\ep'}{h}} \sum_{i=1}^{N_h^{K+1}(x,a)} ( \Ro_h^{k_h^i(x,a)} - Q_h^\st(x_h^k, a_h^k)) \nonumber\\
      \leq & e^2 K H \cdot (2\ep' + \ep H) + e^2 H \cdot | \fool{\ep(H+1)}{\ep'}| + 2e^2 \sum_{h'=h+1}^H \sum_{k=1}^K \sigma_{h'}^k \clipdelta_{h'}^k\nonumber\\
      &+ e^2 \sum_{h'=h}^H \sum_{(x,a,h') \in \fools{\ep(H+1)}{\ep'}{h'}} \min \left\{8C_0 \sqrt{H^3 \iota \cdot  N_{h'}^{K+1}(x,a)},  \frac{64C_0^2 H^4\iota }{[\Delta_{h'}(x,a) - 2\ep \cdot (H+1)]_+}  \right\} \nonumber\\
      \leq &  e^2 K H \cdot (2\ep' + \ep H) + e^2 H \cdot | \fool{\ep(H+1)}{\ep'}|+ 2e^2 \sum_{h'=h+1}^H \sum_{k=1}^K \sigma_{h'}^k \clipdelta_{h'}^k \nonumber\\
      & + e^2 \sum_{h'=h}^H \min \left\{ 8C_0 \sqrt{H^3K \iota \cdot | \fools{\ep(H+1)}{\ep'}{h'}|}, \sum_{(x,a,h') \in \fools{\ep(H+1)}{\ep'}{h'}} \frac{64 C_0^2 H^4 \iota}{[\Delta_{h'}(x,a) - 2\ep \cdot (H+1)]_+}  \right\} \nonumber\\
      \leq &  e^2KH \cdot (2\ep' + \ep H) + e^2H \cdot | \fool{\ep(H+1)}{\ep'}| + 2e^2 \sum_{h'=h+1}^H \sum_{k=1}^K \sigma_{h'}^k \clipdelta_{h'}^k\nonumber\\
      & + \min \left\{ 8C_0e^2  \sqrt{H^4 K \iota \cdot |\fool{\ep(H+1)}{\ep'}|},\sum_{(x,a,h) \in \fool{\ep(H+1)}{\ep'}} \frac{64 e^2 C_0^2 H^4 \iota}{[\Delta_{h}(x,a) - 2\ep \cdot (H+1)]_+}   \right\}\nonumber.
    \end{align}
    Combining this with (\ref{eq:2tilv-ub}) and (\ref{eq:use-claim-subopt}) gives that
    \begin{align}
      & \sum_{k\in [K]} (V_1^\st - V_1^{\pi^k})(x_1) \nonumber\\
      \leq & HK (2\ep (H+1) + 2\ep') + HK \cdot \Pr[\overline{\MEwc \cap \MEpred}] + e^2KH^2 \cdot (2\ep' + \ep H) + e^2 H^2 \cdot | \fool{\ep(H+1)}{\ep'}| \nonumber\\
      & + \min \left\{ 8C_0e^2  \sqrt{H^6 K \iota \cdot |\fool{\ep(H+1)}{\ep'}|},\sum_{(x,a,h) \in \fool{\ep(H+1)}{\ep'}} \frac{64 e^2 C_0^2 H^5 \iota}{[\Delta_h(x,a) - 2\ep \cdot (H+1)]_+}   \right\} \nonumber\\ %
      & + 2e^2 H \sum_{h=1}^H \sum_{k=1}^K \sigma_h^k \clipdelta_h^k \nonumber\\
      \leq & O((\ep H + \ep') \cdot TH ) + O\left( \min \left\{ \sqrt{H^6 K \iota \cdot |\fool{\ep(H+1)}{\ep'}|}, \sum_{(x,a,h) \in \fool{\ep(H+1)}{\ep'}} \frac{H^4 \iota}{[\Delta_h(x,a) - 2\ep \cdot (H+1)]_+}\right\} \right)\nonumber\\
      & + 2e^2 H \sum_{h=1}^H \sum_{k=1}^K \sigma_h^k \clipdelta_h^k,\nonumber
    \end{align}
    where in the final inequality we use that $H^2 \cdot |\fool{\ep(H+1)}{\ep'}| \leq \sqrt{H^6 K \iota \cdot |\fool{\ep(H+1)}{\ep'}|}$ as long as $K \geq |\fool{\ep(H+1)}{\ep'}|$. (For $K < |\fool{\ep(H+1)}{\ep'}|$ the trivial regret bound of $KH$ is bounded above by $\sqrt{H^6 K \iota \cdot |\fool{\ep(H+1)}{\ep'}|}$.) Moreover, we also use that $H^2 \cdot |\fool{\ep(H+1)}{\ep'}| \leq \sum_{(x,a,h) \in \fool{\ep(H+1)}{\ep'}} \frac{H^4 \iota}{[\Delta_h(x,a) - 2\ep \cdot (H+1)]_+}$ in the final line. This verifies the statement (\ref{eq:nogap-adaptive}) of the theorem.

  \end{proof}

  Note that Lemma \ref{lem:bound-by-foolset} does not quite establish the guarantee of the improved regret bounds of Theorems \ref{thm:main-deltaincr} or \ref{thm:main-deltaconst} when the predictions $\til Q$ are an approximate distillation of $Q^\st$. In particular, we have not yet shown how to bound the term $\sum_{h=1}^H \sum_{k=1}^K \sigma_h^k \clipdelta_h^k$. %
  We do so in Lemmas \ref{lem:sigma-delta-const} and \ref{lem:sigma-delta-incr} below; the first treats the case where \algname uses \deltaconst, and the second treats the case where \algname uses \deltaincr.
  \begin{lemma}
  \label{lem:sigma-delta-const}
  For any prediction function $\til Q$, the algorithm \algname (with \deltaconst) has the following guarantee under the event $\MEwc$:
  \begin{align}
\sum_{h=1}^H \sum_{k=1}^K \sigma_h^k \cdot\clipdelta_h^k \leq O \left( \min \left\{ \sqrt{\wh \lambda \cdot H^8 SA T \iota}, H^7 \iota \cdot \left(  \sum_{(x,a,h) \in \MS \times \MA \times [H] : a \not \in \Aopt_{h,0}(x)} \frac{1}{\Delta_h(x,a)} + \frac{|\Amul|}{\delmin} \right) \right\} \right)\nonumber.
  \end{align}
\end{lemma}
\begin{proof}
  Suppose that $\MEwc$ holds. Recall that $\clipdelta_h^k$ is only defined for $h,k$ so that $x_h^k \not \in \MG_h^k$; but if $x_h^k \in \MG_h^k$, then $\sigma_h^k = 0$, meaning that the sum $\sum_{h=1}^H \sum_{k=1}^K \sigma_h^k \cdot \clipdelta_h^k$ is well-defined.

  Further, recall that $\wh \lambda$ is chosen so that $\frac{1}{\wh \lambda} \cdot \CC_{M,T,\wh \lambda} = \CR$. We claim that $\wh \lambda \geq SAH^3/K$; to see this, note that $\frac{1}{\lambda} \cdot \CC_{M,T,\lambda}$ is a decreasing function of $\lambda$, and that for the choice $\lambda_0 = SAH^3/K = SAH^4/T$,
  \begin{align}
\frac{1}{\lambda_0} \CC_{M,T,\lambda_0} \geq \min \left\{ TH^2, \frac{H^7T}{SAH^4} \cdot \frac{SAH}{2H} \right\} \geq T/2 \geq  \CR\nonumber.
  \end{align}
  Finally note that $\wh \Delta^K = \frac{\CR}{KH} = \frac{\CC_{M,T,\wh \lambda}}{\wh \lambda \cdot KH}$. 
  
  For each $h \in \MH$, set $\MY_h := \{ k : \sigma_h^k = 1 \}$. Lemma \ref{lem:bound-sigma} %
  gives that $|\MY_h| \leq \max \{ SAH^3, \wh \lambda \cdot K \} \leq \wh \lambda \cdot K$, where we use that $\wh \lambda$ is chosen so that $\wh \lambda K \geq SAH^3$.  By item \ref{it:w-sum-all} of Lemma \ref{lem:w-bound-gap},
    \begin{align}
      & \sum_{h=1}^H \sum_{k=1}^K \sigma_h^k \clipdelta_h^k\nonumber\\
      = & \sum_{h=1}^H \sum_{k \in \MY_h} \clipdelta_h^k \nonumber\\
      \leq & \sum_{h=1}^H \left(|\MY_h| \cdot \ggapfn{h}{\wh \Delta} + e^2 SAH^2 + \min \left\{ e^2 C_2 \sqrt{H^5 SA |\MY_h| \iota}, \sum_{(x,a,h') \in \MS \times \MA \times [H]} \frac{e^2 C_2^2 H^3 \iota}{\gapfinals{h'}{K}} \right\}\right) \nonumber\\
      \leq & e^2 SAH^3 + \min \left\{ e^2 C_2 \sqrt{H^7 SA  \wh \lambda K \iota}, \sum_{(x,a,h) \in \MS \times \MA \times [H]} \frac{e^2 C_2^2 H^4 \iota}{\gapfinals{h}{K}} \right\}\nonumber\\
       &+ O(\wh \lambda \cdot K)\cdot \min \left\{ \sqrt{\frac{H^7SA \iota}{\wh \lambda \cdot K}}, \frac{1}{\wh \lambda \cdot K} \cdot H^6 \iota \cdot  \left( \sum_{(x,a,h') \in \MS \times \MA \times [H] : a \not \in \Aopt_{h',0}(x)} \frac{1}{\Delta_h(x,a)} + \frac{|\Amul|}{\delmin} \right) \right \} \nonumber.
    \end{align}
    
    By Lemma \ref{lem:frzq-gap} (which we can use since $\MEwc$ holds) and Lemma \ref{lem:cc-lb} (together with $\wh \lambda \geq SAH^3/K$), it therefore follows that
    \begin{align}
      & \sum_{h=1}^H \sum_{k=1}^K \sigma_h^k \clipdelta_h^k \nonumber\\
      \leq & O \left( \min \left\{ \sqrt{\wh \lambda \cdot H^8 SA T \iota}, H^7 \iota \cdot \left(  \sum_{(x,a,h) \in \MS \times \MA \times [H] : a \not \in \Aopt_{h,0}(x)} \frac{1}{\Delta_h(x,a)} + \frac{|\Amul|}{\delmin} \right) \right\} \right)\nonumber.
    \end{align}
  \end{proof}

  \begin{lemma}
    \label{lem:sigma-delta-incr}
  For any prediction function $\til Q$, the algorithm \algname given some parameter $\lambda \geq SAH^3/K$ (with \deltaincr and some input parameter $\tildm \leq \delmin$) has the following guarantee under the event $\MEwc$:
  \begin{align}
\sum_{h=1}^H \sum_{k=1}^K \sigma_h^k \clipdelta_h^k
    \leq & O \left( \min \left\{ \sqrt{\lambda \cdot SAH^9 T \iota^2}, H^8 \iota^2 \cdot \left( \sum_{(x,a,h) : a \not \in \Aopt_{h,0}(x)} \frac{1}{\Delta_h(x,a)} + \frac{|\Amul|}{\tildm} \right) \right\} \right)\nonumber.
  \end{align}
\end{lemma}
\begin{proof}
For each $h \in \MH$, set $\MY_h := \{ k : \sigma_h^k = 1 \}$. Lemma \ref{lem:bound-sigma-incr} gives that $|\MY_h| \leq \max\{ SAH^3, \lambda \cdot K \}$. Recall that the input parameter $\lambda$ was assumed to satisfy $\lambda \geq SAH^3 /K$, meaning that $|\MY_h| \leq \lambda K$. By item \ref{it:w-sum-all} of Lemma \ref{lem:w-bound-gap},
    \begin{align}
      & \sum_{h=1}^H \sum_{k=1}^K \sigma_h^k \clipdelta_h^k\nonumber\\
      = & \sum_{h=1}^H \sum_{k \in \MY_h} \clipdelta_h^k \nonumber\\
      \leq & \sum_{h=1}^H \left(|\MY_h| \cdot \ggapfn{h}{\wh \Delta^K} + e^2 SAH^2 + \min \left\{ e^2 C_2 \sqrt{H^5 SA |\MY_h| \iota}, \sum_{(x,a,h') \in \MS \times \MA \times [H]} \frac{e^2 C_2^2 H^3 \iota}{\gapfinals{h'}{K}} \right\}\right) \nonumber\\
      \leq & e^2 SAH^3 + \min \left\{ e^2 C_2 \sqrt{H^7 SA   \lambda K \iota}, \sum_{(x,a,h) \in \MS \times \MA \times [H]} \frac{e^2 C_2^2 H^4 \iota}{\gapfinals{h}{K}} \right\}\nonumber\\
        &+ O(1) \cdot H \cdot \lambda K \cdot   \min \left\{\frac{H^5 \iota^2}{\lambda \cdot K} \cdot \sum_{(x,a,h)} \frac{1}{\max\left\{ \frac{\tfrzQ_h^{K}(x,a)}{2H}, \frac{\tildm}{4H^2} \right\} }, \sqrt{\frac{SAH^8\iota^2}{\lambda \cdot K}} \right\} \nonumber.
    \end{align}
    By Lemma \ref{lem:frzq-gap} (which we can apply since $\MEwc$ holds) and Lemma \ref{lem:cc-lb} (using that $\lambda \geq SAH^3/K$), it follows that
    \begin{align}
      \sum_{h=1}^H \sum_{k=1}^K \sigma_h^k \clipdelta_h^k
      \leq & O \left(\min \left\{ \sqrt{\lambda \cdot SAH^9 T \iota^2}, H^8 \iota^2 \cdot \left( \sum_{(x,a,h) : a \not \in \Aopt_{h,0}(x)} \frac{1}{\Delta_h(x,a)} + \frac{|\Amul|}{\tildm} \right) \right\} \right)\nonumber.
    \end{align}
  \end{proof}

  \section{Proofs of main theorems}
  \label{sec:main-thms-proof}
  We begin by proving Theorem \ref{thm:main-deltaconst}, restated below for convenience.
\begin{reptheorem}{thm:main-deltaconst}[Restated]
  The algorithm \algname with the \deltaconst subroutine satisfies the following two guarantees, when given as input a parameter $\CR \in [SAH^3, \frac{T}{SA}]$ and predictions $\til Q$:
  \begin{enumerate}
  \item If $\CR \geq \CC_{M,T,1}$, then for an \emph{arbitrary} choice of input predictions $\til Q$, the regret of \algname is $O(\CR)$.\label{it:worstcase-deltaconst-full}
  \item Fix any $\ep > 0$, and set $\ep' = 4\ep \cdot (H+1)$. When the input predictions $\til Q$ are an $\ep$-approximate distillation of $Q^\st$ (Definition \ref{def:distillation}) and lack $\ep'$-fooling optimal actions (Definition \ref{def:ofa}), the regret of \algname is
    \begin{align}
      O\left( H \cdot \CC_{M,T,\wh \lambda} + \ep' TH+\min \left\{ \sqrt{H^5 T \iota \cdot |\fool{\ep'/2}{\ep'}|}, \sum_{(x,a,h) \in \fool{\ep'/2}{\ep'}} \frac{H^4 \iota}{[\Delta_h(x,a) - \ep'/2]_+}\right\} \right),\label{eq:fooling-regret-full}
    \end{align}
    where $\wh \lambda \in (0,1)$ is chosen so that $\frac{1}{\wh \lambda} \cdot \CC_{M,T,\wh \lambda} = \CR$. 
    \label{it:distillation-deltaconst-full}
  \end{enumerate}
\end{reptheorem}
\begin{proof}
  We begin with the proof of item \ref{it:worstcase-deltaconst-full}. It is without loss to assume $T \geq SAH^3$; otherwise, by similar reasoning to that in Lemma \ref{lem:cc-lb}, the trivial regret bound of $T$ suffices. Now, item \ref{it:worstcase-deltaconst-full} is an immediate consequence of Lemma \ref{lem:worst-case-const}.

  We next prove item \ref{it:distillation-deltaconst-full}. The event $\MEwc \cap \MEpred$ does not hold with probability at most $2p = 2/(H^2 K)$, which adds at most $T \cdot 2p = O(1)$ to the regret bound. Thus it suffices to bound the regret conditioned on $\MEwc \cap \MEpred$. Then item \ref{it:distillation-deltaconst-full} is an immediate consequence of Lemmas \ref{lem:bound-by-foolset} and \ref{lem:sigma-delta-const}. 
\end{proof}

We next prove Theorem \ref{thm:main-deltaincr}; below we present the version of the theorem which does not require that each state has a unique optimal action. In this more general setting, the subroutine \deltaincr of \algname requires as input a parameter $\tildm$ which is guaranteed to be a lower bound on $\delmin$. The resulting regret bounds will depend on a modified version of the $\lambda$-complexity $\CC_{M,T,\lambda}$ (see (\ref{eq:lambda-complexity})) with the parameter $\tildm$ replacing $\delmin$; more precisely, we define
\begin{align}
\CC_{M,T,\lambda,\tildm} :=  \min \left\{ \sqrt{\lambda \cdot TSAH^8 \iota},\ H^7 \iota \cdot \left( \sum_{(x,a,h) \in \MS \times \MA \times [H] : a \not \in \Aopt_{h,0}(x)} \frac{1}{\Delta_h(x,a)} + \frac{|\Amul|}{\tildm} \right) \right\}\label{eq:lambda-delmin-complexity}
\end{align}
In the special case that $|\Amul| = 0$ (i.e., each state has a unique optimal action) and $\tildm = 0$, the quantity $\frac{0}{0}$ in (\ref{eq:lambda-delmin-complexity}) is to be interpreted as 0.

  \begin{reptheorem}{thm:main-deltaincr}[Full version]
    Suppose we run algorithm \algname (Algorithm \ref{alg:main}) with input parameter $\lambda \in [0,1]$, together with the \deltaincr subroutine (Algorithm \ref{alg:delta-incr}) with parameter $\tildm$ which is guaranteed to satisfy $\tildm \leq \delmin$. %
Then, when given predictions $\til Q$, the algorithms satisfy the following guarantees:
\begin{enumerate}
\item Suppose $\lambda \geq \frac{SAH^4}{T}$. Then for an arbitrary choice of input predictions $\til Q$, the regret of \algname is $$O \left( \frac{H\iota}{\lambda} \cdot \CC_{M,T,\lambda,\tildm}\right).$$\label{it:worstcase-deltaincr-full}
\item \label{it:distillation-deltaincr-full} Fix any $\ep > 0$, and set $\ep' = 4\ep \cdot (H+1)$. When the input predictions $\til Q$ are an $\ep$-approximate distillation of $Q^\st$ (Definition \ref{def:distillation}) and lack $\ep'$-fooling actions (Definition \ref{def:ofa}), the regret of \algname is
  \begin{align}
    \hspace{-1cm}
O \left( H^2 \iota \cdot \CC_{M,T,\lambda,\tildm} + \ep' TH+  \min \left\{ \sqrt{H^5 T \iota \cdot |\fool{\ep'/2}{\ep'}|}, \sum_{(x,a,h) \in \fool{\ep'/2}{\ep'}} \frac{H^4 \iota}{[\Delta_h(x,a) - \ep'/2]_+}\right\} \right)\label{eq:fooling-regret-deltaincr-full}.
  \end{align}
\end{enumerate}
\end{reptheorem}
\begin{proof}
  We begin with the proof of item \ref{it:worstcase-deltaincr-full}. As in the proof of Theorem \ref{thm:main-deltaconst}, it is without loss to assume that $T \geq SAH^3$, as otherwise the trivial regret bound of $T$ suffices. Now, item \ref{it:worstcase-deltaincr-full} is an immediate consequence of Lemma \ref{lem:worst-case-incr}.

  We next prove item \ref{it:distillation-deltaincr-full}. The event $\MEwc \cap \MEpred$ does not hold with probability at most $2p = 2/(H^2 K)$, which adds at most $T \cdot 2p = O(1)$ to the regret bound. Thus it suffices to bound the regret conditioned on $\MEwc \cap \MEpred$. Then item \ref{it:distillation-deltaincr-full} is an immediate consequence of Lemmas \ref{lem:bound-by-foolset} and \ref{lem:sigma-delta-incr}. 
  \end{proof}

  \appendix
\section{Miscellaneous lemmas}
The following simple lemma establishes some properties of the parameters $\alpha_n^i$ (defined in (\ref{eq:alphas-all})).
\begin{lemma}[Lemma 4.1, \cite{jin_is_2018}]
  \label{lem:alpha}
  The real numbers $\alpha_n^i$ satisfy the following properties:
  \begin{enumerate}
  \item For every $n \geq 1$, $\frac{1}{\sqrt{n}} \leq \sum_{i=1}^n \frac{\alpha_n^i}{\sqrt i} \leq \frac{2}{\sqrt{n}}$.\label{it:alpha-sum-sup}
  \item For every $n \geq 1$, $\max_{i \in [n]} \alpha_n^i \leq \frac{2H}{n}$ and $\sum_{i=1}^n (\alpha_n^i)^2 \leq \frac{2H}{n}$.\label{it:alpha-sum-square}
  \item For every $i \geq 1$, $\sum_{n=i}^\infty \alpha_n^i = 1 + \frac 1H$.\label{it:alpha-sum-sub}
  \item For every $n \geq 1$, it holds that $\sum_{i=1}^n \alpha_n^i = 1$. \label{it:alpha-sum-one}
  \end{enumerate}
\end{lemma}

Recall the definition of $\CC_{M,T,\lambda}$ in (\ref{eq:lambda-complexity})
\begin{lemma}
  \label{lem:cc-lb}
  For any $\lambda \geq \frac{SAH^3}{K}$, it holds that
  \begin{align}
\CC_{M,T,\lambda} \geq SAH^6/2\nonumber.
  \end{align}
\end{lemma}
\begin{proof}
  Since $\lambda \geq SAH^3/K = SAH^4/T$, it holds that $\sqrt{\lambda \cdot TSAH^8 \iota} \geq SAH^6$.

  Next, it is evident that
  \begin{align}
\sum_{(x,a,h) : a \not \in \Aopt_{h,0}(x)} \frac{1}{\Delta_h(x,a)} + \frac{|\Amul|}{\delmin} \geq \frac{SAH}{2H} \geq SA/2\nonumber,
  \end{align}
  since for each $(x,h)$, all but one of the actions $a$ in $\MA$ are either counted in the form of $\frac{1}{\Delta_h(x,a)} \geq 1/H$ or $1/\delmin \geq 1/H$.

  Putting the above statements together gives the desired bound.
\end{proof}

\section{Proof for bandit case}
\label{sec:bandits}
In this section we prove Proposition \ref{prop:bandits}, which specializes our main results to the case of multi-armed bandits. Though our bounds for multi-armed bandits are superseded by our regret bounds for online learning in MDPs, we present a separate proof for the bandit case to provide intuition about our techniques.
\begin{algorithm}[!htp]
  \caption{\bf \banditalg}\label{alg:bandit}
  \KwIn{Action space $\MA$, number of time steps $T$, predictions $\til Q: \MA \ra [0,1]$, parameter $\lambda \in [0,1]$ and $\delta > 0$.}
  \begin{enumerate}[leftmargin=14pt,rightmargin=20pt,itemsep=1pt,topsep=1.5pt]
  \item For each $a \in \MA$, initialize $\Qo^1(a) = \infty,\ \Qu^1(a) = -\infty$, $N^1(a) = 0$, and $\til Q^1(a) = \til Q(a)$.
  \item For $1 \leq t \leq T$:
    \begin{enumerate}
    \item If $t \leq \lambda \cdot T$:
      \begin{enumerate}
      \item Select action $a^t := \argmax_{a \in \MA}\{ \Qo^t(a)\}$.
      \end{enumerate}
    \item Else (i.e., if $t > \lambda \cdot T$):
      \begin{enumerate}
      \item Select action
        $
           a^t := \argmax_{a \in \MA} \left\{ \til Q^t(a) \right\}\nonumber. %
$ 
     \end{enumerate}
  \item For each action $a \in \MA$, let $N^{t+1}(a)$ denote the number of times $a$ was taken up to (and including) step $t$.
  \item For each action $a \in \MA$, let $\hat \mu^{t+1}(a)$ denote the mean of all rewards received when taking $a$ up to step $t$ (if $N^{t+1}(a) = 0$, set $\hat \mu^{t+1}(a) = 0$). 
  \item Update the $Q$-value functions as follows: for each $a \in \MA$, set
    \begin{align}
\Qo^{t+1}(a) := &  \hat \mu^{t+1}(a) + \sqrt{ \frac{2 \log 1/\delta}{N^{t+1}(a)}}\nonumber\\
      \Qu^{t+1}(a) := &  \hat \mu^{t+1}(a) - \sqrt{ \frac{2 \log 1/\delta}{N^{t+1}(a)}}\nonumber\\
      \til Q^{t+1}(a) := & \max \left\{ \Qu^{t+1}(a), \min \left\{ \Qo^{t+1}(a), \til Q(a) \right\} \right\}\nonumber.
    \end{align}
  \end{enumerate}
  \end{enumerate}
\end{algorithm}

\begin{repproposition}{prop:bandits}
  There is an algorithm (\banditalg, Algorithm \ref{alg:bandit}) which satisfies the following two guarantees, when given as input a parameter $\lambda \in \left( \frac{A}{T}, 1 \right)$ and predictions $\til Q$:
  \begin{enumerate}
  \item \label{it:bandit-distillation} Fix any $\ep > 0$. If the predictions $\til Q$ are an $\ep$-approximate distillation of $Q^\st$, then the regret is $\widetilde O(\ep T + \sqrt{|\MG| \cdot T} + \sqrt{\lambda \cdot AT})$, where
    \begin{align}
      \MG := \left\{ a \in \MA\backslash \{a^\st\} : \til Q(a) \geq  Q^\st(a^\st) - \ep \right\}.\nonumber
    \end{align}
  \item For an arbitrary choice of $\til Q$, the regret is $\widetilde O \left( \sqrt{\frac{TA}{\lambda}} \right)$.\label{it:bandit-robustness}
  \end{enumerate}
\end{repproposition}

For convenience,  for each $t \leq T$, we define
\begin{align}
\til V^t := \max_{a \in \MA} \left\{ \til Q^t(a) \right\} %
\end{align}
By construction of the algorithm \banditalg, note that $\til Q^t(a^t) = \til V^t$. %
  We define the following ``good event'' $\ME_0$:
  $$
\ME_0 = \left\{ \forall t \leq T,\ \forall a \in \MA,\ \left| Q^\st(a) - \hat \mu^t(a) \right| \leq \sqrt{\frac{2 \log 1/\delta}{N^t(a)}} \right\}.
$$
Note that under the event $\ME_0$, $Q^\st(a) \in [\Qu^t(a), \Qo^t(a)]$ for all $a \in \MA$.

\begin{lemma}
  \label{lem:e0-basic}
  Suppose the event $\ME_0$ holds. Then for any sub-optimal action $a \neq a^\st$:
  \begin{enumerate}
  \item If $a$ is taken at step $t \leq \lambda T$, then $N^t(a) \leq \frac{8 \log 1/\delta}{\Delta(a)^2}$.\label{it:e0-early}
  \item If $a$ is taken at step $t > \lambda T$ and $\til Q$ is an $\ep$-approximate distillation of $Q^\st$, and if $\Delta(a) > \ep$, then $\til Q(a) \geq Q^\st(a^\st) - \ep$. \label{it:e0-distillation}
    In such a case, we have $N^t(a) \leq \frac{8 \log 1/\delta}{(\Delta(a) - \ep)^2}$.
  \end{enumerate}
\end{lemma}
\begin{proof}
  If $a$ is taken when $t \leq \lambda T$, we must have that $\Qo^t(a) \geq \Qo^t(a^\st) \geq Q^\st(a^\st)$%
  , meaning that
  \begin{align}
Q^\st(a) \geq \Qo^t(a) - \sqrt{\frac{8 \log 1/\delta}{N^t(a)}} \geq Q^\st(a^\st) - \sqrt{\frac{8 \log 1/\delta}{N^t(a)}}\nonumber,
  \end{align}
  from which the first point follows.

  If $a$ is taken when $t > \lambda T$, then we must have that $\til Q^t(a) \geq \til Q^t(a^\st) \geq Q^\st(a^\st) - \ep$. Since $\Delta(a) > \ep$, we have that $\Qu^t(a) \leq Q^\st(a) < Q^\st(a^\st) - \ep$, meaning that $\til Q^t(a) \leq \til Q(a)$, and hence $\til Q(a) \geq Q^\st(a^\st) - \ep$. In such a case, we have $\Qo^t(a) \geq \til Q^t(a) \geq Q^\st(a^\st) - \ep$, meaning that %
  \begin{align}
Q^\st(a) \geq \Qo^t(a) - \sqrt{\frac{8 \log 1/\delta}{N^t(a)}} \geq Q^\st(a^\st) - \sqrt{\frac{8 \log 1/\delta}{N^t(a)}} - \ep\nonumber,
  \end{align}
  from which the desired inequality follows. 
\end{proof}

  \begin{lemma}
    \label{lem:arm-pull-bnds}
    Suppose the event $\ME_0$ holds. Then if $\Delta > 0$ satisfies $\frac{16 A \log 1/\delta}{\Delta^2} \leq \lambda T$, in the first $\frac{16 A \log 1/\delta}{\Delta^2}$ steps, some action $a$ with $\Delta(a) \leq \Delta$ has been taken at least $\frac{8 \log 1/\delta}{\Delta^2}$ times.
  \end{lemma}
  \begin{proof}
    By item \ref{it:e0-early} of Lemma \ref{lem:e0-basic}, under the event $\ME_0$, each action $a \in \MA$ is taken at most $\frac{8 \log 1/\delta}{\Delta(a)^2}$ time steps in time steps $t \leq \lambda T -1$. Let $A_\Delta$ denote the number of actions $a$ with $\Delta(a) \leq \Delta$. Then by the pigeonhole principle, in the first
    \begin{align}
\sum_{a \in \MA : \Delta(a) > \Delta} \frac{8 \log 1/\delta}{\Delta(a)^2} + \frac{8 A_\Delta \log 1/\delta}{\Delta^2} \leq \frac{8 A \log 1/\delta}{\Delta^2}\leq \lambda T -1\nonumber
    \end{align}
    steps, some action $a$ with $\Delta(a) \leq \Delta$ has been taken at least $\frac{8 \log 1/\delta}{\Delta^2}$ times. 
  \end{proof}

\begin{proof}[Proof of Proposition \ref{prop:bandits}]
We set $\delta = 1/(AT^2)$ in \banditalg (Algorithm \ref{alg:bandit}). 
  It is straightforward from a Chernoff bound and union bound that $\Pr(\ME_0) \geq 1- \delta \cdot TA$. %
  Therefore, the expected regret is bounded above by
  \begin{align}
    \E \left[\sum_{a \in \MA} \Delta(a) \cdot N^{T+1}(a)\right] \leq & \delta TA \cdot T + \E \left[ \One[\ME_0] \cdot \sum_{a \in \MA} \Delta(a) \cdot N^{T+1}(a) \right]\nonumber\\
    \leq & 1 + \E \left[ \sum_{a \in \MA} \Delta(a) \cdot N^{T+1}(a) | \ME_0 \right]\nonumber,
  \end{align}
  where the second inequality uses $\delta \leq 1/(AT^2)$. %

  We begin by bounding the regret in the event that $\til Q$ is an $\ep$-approximate distillation of $Q^\st$ (i.e., item \ref{it:bandit-distillation} of Proposition \ref{prop:bandits}).  By item \ref{it:e0-distillation} of Lemma \ref{lem:e0-basic}, any arm $a$ with $\Delta(a) \geq 2\ep$ that is pulled at some step $t > \lambda T$ must satisfy $\til Q(a) \geq Q^\st(a^\st) - \ep$, i.e., $a \in \MG$. Moreover, for such arms $a$, we have that $N^{T+1}(a) \leq \frac{32 \log 1/\delta}{\Delta(a)^2}$. Therefore, under the event $\ME_0$, we have
  \begin{align}
    \sum_{a \in \MA} \Delta(a) \cdot N^{T+1}(a) \leq & 2\ep T + \sum_{a \in \MG} \Delta(a) \cdot N^{T+1}(a) + \sum_{a \in \MA \backslash \MG} \Delta(a) \cdot N^{\lambda T+1}(a) \nonumber\\
    \leq & 2 \ep T + \sqrt{\lambda TA \log 1/\delta} + \sum_{a \in \MA \backslash \MG : \Delta(a) > \sqrt{A \log 1/\delta/(\lambda T)}} \frac{8 \log 1/\delta}{\Delta(a)}\nonumber\\
                                                     &+ \One[|\MG| > 0] \cdot \left( \sqrt{T |\MG| \log 1/\delta} + \sum_{a \in \MG : \Delta(a) > \sqrt{|\MG| \log 1/\delta / T}} \frac{32\log 1/\delta}{\Delta(a)}\right)\nonumber\\
    \leq & O \left( \ep T + \sqrt{\lambda TA \log 1/\delta} + \sqrt{T |\MG| \log 1/\delta} \right)\nonumber.
  \end{align}

  Next we prove item \ref{it:bandit-robustness} (the robustness claim) of the proposition. Set $\Delta_\lambda := \sqrt{\frac{16 A \log 1/\delta}{T \lambda}}$. Then by Lemma \ref{lem:arm-pull-bnds}, in the first $\lambda T$ steps of \banditalg, some action $\bar a$ with $\Delta(\bar a) \leq \Delta_\lambda$ has been taken at least $\frac{8 \log 1/\delta}{\Delta_\lambda^2} = \frac{T \lambda}{2 A}$ times. Hence
  \begin{align}
\Qu^{\lambda T +1}(\bar a) \geq Q^\st(a^\st) - \Delta(\bar a) - \sqrt{\frac{16 A \log 1/\delta}{T \lambda}} \geq Q^\st(a^\st) - 2 \Delta_\lambda \nonumber.
  \end{align}
  Therefore, for all $t > \lambda T$,
  $
\max_{a \in \MA} \{ \til Q^t(a) \} \geq \max_{a \in \MA} \{ \Qu^t(a) \} \geq Q^\st(a^\st) - 2 \Delta_\lambda.
$ %
Hence, for any action $a \in \MA$ that is taken at step $t > \lambda T$ and satisfies $\Delta(a) > 2 \Delta_\lambda$, we have that $\Qo^t(a) \geq Q^\st(a^\st) - 2 \Delta_\lambda$, meaning that
\begin{align}
Q^\st(a) \geq \Qo^t(a) - \sqrt{\frac{8 \log 1/\delta}{N^t(a)}} \geq Q^\st(a^\st) - \sqrt{\frac{8 \log 1/\delta}{N^t(a)}} - 2 \Delta_\lambda,\nonumber
\end{align}
meaning that $N^t(a) \leq \frac{8 \log 1/\delta}{(\Delta(a) - 2 \Delta_\lambda)^2}.$ Hence, under the event $\ME_0$, we have
\begin{align}
\sum_{a \in \MA} \Delta(a) \cdot N^{T+1}(a) \leq & 4 \Delta_\lambda \cdot T + \sum_{a \in \MA : \Delta(a) > 4 \Delta_\lambda} \frac{32 \log 1/\delta}{\Delta(a)}\leq O \left( \sqrt{\frac{TA \log 1/\delta}{\lambda} }\right)\nonumber.
\end{align}

\end{proof}

\bibliographystyle{alpha}
\bibliography{RL.bib,learning-augmented.bib}

\end{document}